%% file: lpdensecrf.tex
\documentclass[10pt,onecolumn,letterpaper]{article}

\usepackage{cvpr}
\usepackage{times}
\usepackage{epsfig}
\usepackage{graphicx}
\usepackage{amsmath}
\usepackage{amssymb}

\usepackage{enumerate}
\usepackage{amsthm}

\theoremstyle{definition}
\newtheorem{dfn}{Definition}[section]
\newtheorem{pro}{Proposition}[section]
\theoremstyle{plain}

\newcommand{\SKIP}[1]{}
\newcommand{\calV}{\mathcal{V}}

\newcommand{\calL}{\mathcal{L}}
\newcommand{\calO}{\mathcal{O}}

\newcommand{\bfx}{\mathbf{x}}

\newcommand{\bff}{\mathbf{f}}
\newcommand{\bfy}{\mathbf{y}}
\newcommand{\bfz}{\mathbf{z}}
\newcommand{\tE}{\tilde{E}}
\newcommand{\I}{\mathbbm{1}}
\newcommand{\calX}{\mathcal{X}}
\newcommand{\bfalpha}{\boldsymbol{\alpha}}
\newcommand{\bfbeta}{\boldsymbol{\beta}}
\newcommand{\bfgamma}{\boldsymbol{\gamma}}
\newcommand{\bfphi}{\boldsymbol{\phi}}
\newcommand{\bfone}{\boldsymbol{1}}
\newcommand{\bfzero}{\boldsymbol{0}}

\newcommand{\R}{\rm I\!R}
\newcommand{\bfs}{\mathbf{s}}
\newcommand{\hbfs}{\hat{\bfs}}
\newcommand{\hs}{\hat{s}}

\newcommand{\calP}{\mathcal{P}}
\newcommand{\barv}{\bar{v}}
\newcommand{\barV}{\bar{V}}
\newcommand{\allpixels}{\{1\ldots n\}}
\newcommand{\barN}{\bar{N}}

\newcommand{\calM}{\mathcal{M}}
\newcommand{\calC}{\mathcal{C}}

\newcommand{\tbfalpha}{\tilde{\bfalpha}}
\newcommand{\bfh}{\mathbf{h}}
\newcommand{\tbfy}{\tilde{\bfy}}
\newcommand{\ty}{\tilde{y}}
\newcommand{\bfp}{\mathbf{p}}
\newcommand{\bfbI}{\mathbf{I}}
\newcommand{\barcalV}{\bar{\calV}}
\newcommand{\calY}{\mathcal{Y}}

\def\myref#1{{\color{red}{#1}}}%
\newcommand{\LineForAll}[2]{%
    \State\algorithmicforall\ {#1}\ \algorithmicdo\ {#2}
}

\usepackage{array}
\usepackage{caption}
\usepackage{subcaption}
\usepackage{framed}

\usepackage{xcolor}
\newcommand{\mybox}[2]{#2}

\usepackage{multirow}

\usepackage{bbm}

\usepackage[noend]{algpseudocode}
\usepackage{algorithm}

\usepackage{authblk}

\graphicspath{{images/}}

\usepackage[pagebackref=true,breaklinks=true,letterpaper=true,colorlinks,bookmarks=false]{hyperref}

\cvprfinalcopy 

\def\cvprPaperID{1227} 

\begin{document}

\title{Efficient Linear Programming for Dense CRFs}

\author[1]{Thalaiyasingam Ajanthan}
\author[2]{Alban Desmaison}
\author[2]{Rudy Bunel}
\author[3]{Mathieu Salzmann}
\author[2]{Philip H.S. Torr}
\author[2,4]{M. Pawan Kumar}
\affil[1]{Australian National University and Data61, CSIRO}
\affil[2]{Department of Engineering Science, University of Oxford}
\affil[3]{Computer Vision Laboratory, EPFL}
\affil[4]{Alan Turing Institute}

\maketitle

\input{introduction.tex}

\input{preliminaries.tex}


\input{prox_lp.tex}

\input{fast_ph.tex}

\input{related_work.tex}

\input{experiments.tex}

\input{conclusion.tex}

\input{appendix.tex}

{
\bibliographystyle{ieee}
\bibliography{lpdensecrf}
}

\end{document}

%% file: introduction.tex
\begin{abstract}
The fully connected conditional random field (CRF) with Gaussian pairwise
potentials has proven popular and effective for multi-class semantic
segmentation.
While the energy of a dense CRF can be minimized accurately using a linear
programming (LP) relaxation, the state-of-the-art algorithm is too slow to be
useful in practice. To alleviate this deficiency,   
we introduce an efficient LP minimization algorithm for dense CRFs.
To this end, we develop a proximal minimization
framework, where the dual of each proximal problem is optimized via block
coordinate descent. We show that each block of variables can be efficiently
optimized.
Specifically, for one block, the problem decomposes into significantly smaller 
subproblems, each of which is defined over a single
pixel.
For the other block, the problem is optimized via conditional gradient
descent. 
This has two advantages: 1) the conditional gradient can be computed in a time
linear in the number of pixels and labels; and 2) the optimal step
size can be computed analytically.
Our experiments on standard datasets provide compelling evidence that our
approach outperforms all existing baselines including the previous LP based 
approach for dense CRFs.
\end{abstract}

\section{Introduction}
In the past few years, the dense conditional random field (CRF) with Gaussian
pairwise potentials has become popular for multi-class image-based semantic
segmentation. At the origin of this popularity lies the use of an efficient
filtering method~\cite{adams2010fast}, which 
was shown to lead to a linear time mean-field inference
strategy~\cite{koltun2011efficient}.
Recently, this filtering method was exploited to minimize the dense CRF
energy using other, typically more effective, continuous relaxation
methods~\cite{desmaison2016efficient}.
Among the relaxations considered in~\cite{desmaison2016efficient}, the linear
programming (LP) relaxation provides strong theoretical guarantees on the
quality of the
solution~\cite{kleinberg2002approximation,kumar2009analysis}.

In~\cite{desmaison2016efficient}, the LP was minimized via projected
subgradient descent. 
While relying on the filtering method,
computing the subgradient was shown to be \textit{linearithmic} in the number of
pixels, but not \textit{linear}. Moreover, even with the use of a line search
strategy, the algorithm required a large number of iterations to converge,
making it inefficient.

We introduce an iterative LP minimization algorithm for a dense
CRF with Gaussian pairwise potentials which has
\textit{linear} time complexity per iteration.
 To this end, instead of relying on a standard
subgradient technique, we propose to make use of the proximal 
method~\cite{parikh2014proximal}. The resulting proximal problem
has a smooth dual, which can be efficiently optimized using block coordinate
descent.
We show that each block of variables can be optimized efficiently.
Specifically, for one block, the problem decomposes into significantly smaller
subproblems, each of which is defined over a single pixel.
For the other block, the problem can be optimized via 
the Frank-Wolfe
algorithm~\cite{frank1956algorithm,lacoste2012block}. We show that the
conditional gradient required by this algorithm can be computed efficiently.
 In particular, we
modify the filtering method of~\cite{adams2010fast} such that the conditional
gradient can be computed in a time \textit{linear} in the number of pixels and
labels. Besides this linear complexity, our approach has two additional benefits.
First, it can be
initialized with the solution of a faster, less accurate algorithm, such as
mean-field~\cite{koltun2011efficient} or the difference of convex (DC)
relaxation of~\cite{desmaison2016efficient}, thus speeding up convergence.
Second,
the optimal step size of our iterative procedure can be obtained analytically,
thus preventing the need to rely on an expensive line search procedure.

We demonstrate the effectiveness of our algorithm on the MSRC and Pascal VOC
2010~\cite{everingham2010pascal} segmentation datasets. The experiments
evidence that our algorithm is significantly faster than the state-of-the-art LP
minimization technique of~\cite{desmaison2016efficient}. Furthermore, it yields
assignments whose energies are much lower than those obtained by other competing
methods~\cite{desmaison2016efficient,koltun2011efficient}. Altogether, our
framework constitutes the first efficient and effective minimization 
algorithm for dense CRFs with Gaussian pairwise potentials.

%% file: preliminaries.tex
\section{Preliminaries}
Before introducing our method, let us first provide some background on the dense
CRF model and its LP relaxation.

\paragraph{Dense CRF energy function.}
A dense CRF is defined on a set of $n$ random variables $\calX=\{X_1,\ldots,
X_n\}$, where each random variable $X_a$ takes a label $x_a \in \mathcal{L}$,
with $|\calL| = m$. For a given labelling $\bfx$, the energy associated with a
pairwise dense CRF can be expressed as
\begin{equation}
E(\bfx) = \sum_{a=1}^n \phi_a(x_a) + \sum_{a=1}^n\sum_{\substack{b=1\\b\ne a}}^n
\psi_{ab}(x_a, x_b)\ ,
\label{eqn:e}
\end{equation}
where $\phi_a$ and $\psi_{ab}$ denote the \textit{unary potentials} and
\textit{pairwise potentials}, respectively. The unary potentials define the data
cost and the pairwise potentials the smoothness cost.

\paragraph{Gaussian pairwise potentials.}
Similarly to~\cite{desmaison2016efficient,koltun2011efficient}, we consider 
Gaussian pairwise potentials, which have the following form:
\vspace{-0.2cm}
\begin{align}
\label{eqn:pe}
\psi_{ab}(x_a,x_b) &= \mu(x_a,x_b) \sum_c w^{(c)}\,k\left(\bff_a^{(c)},
\bff_b^{(c)}\right)\ ,\\[-0.1cm]\nonumber
 k(\bff_a, \bff_b) &= \exp\left(\frac{-\|\bff_a -
\bff_b\|^2}{2} \right)\ .
\end{align}
Here, $\mu(x_a,x_b)$ is referred to as the \textit{label compatibility}
function and the mixture of Gaussian kernels as the
\textit{pixel compatibility} function. The weights $w^{(c)}$ define the mixture
coefficients, and $\bff_a^{(c)}\in\R^{d^{(c)}}$ encodes features associated to
the random variable $X_a$, where $d^{(c)}$ is the feature dimension.
For semantic segmentation, each pixel in an image corresponds to a random
variable. In practice, as in~\cite{desmaison2016efficient,koltun2011efficient}, we then use the
position and
RGB values of a pixel as features, and assume the label compatibility
function to be the Potts model, that is, $\mu(x_a,x_b)=\I[x_a\ne x_b]$. These
potentials proved to be useful in obtaining fine grained labellings in
segmentation tasks~\cite{koltun2011efficient}.
 
\paragraph{Integer programming formulation.}
An alternative way of representing a labelling is by defining indicator
variables $y_{a:i}\in \{0,1\}$, where $y_{a:i} = 1$ if and only if $x_a = i$.
Using this notation, the energy minimization problem can be written as the
following Integer Program (IP):
\begin{alignat}{3}
\label{eqn:ip}
&\underset{\bfy}{\operatorname{\min}}\quad  E(\bfy) &&=
\sum_a\sum_i \phi_{a:i}\,y_{a:i} + \sum_{a,b\ne a} \sum_{i,j}
\psi_{ab:ij}\,y_{a:i}\,y_{b:j}\ ,\\\nonumber 
&\text{s.t.}\quad  \sum_i y_{a:i} &&= 1\quad\quad\ \ \  \forall\,a \in
\allpixels\ ,\\\nonumber 
& \hskip0.08\linewidth y_{a:i} &&\in \{0,1\}\quad \forall\,a \in
\allpixels,\quad \forall\,i\in\calL\ .
\end{alignat}
Here, we use the shorthand $\phi_{a:i}=\phi_{a}(i)$ and
$\psi_{ab:ij}=\psi_{ab}(i,j)$. The first set of constraints ensure that each
random variable is assigned exactly one label. Note that the value of objective
function is equal to the energy of the
labelling encoded by $\bfy$. 

\paragraph{Linear programming relaxation.}
By relaxing the binary constraints of the indicator variables in~(\ref{eqn:ip})
and using the fact that the label compatibility function is the Potts
model, the linear programming relaxation~\cite{kleinberg2002approximation}
of~(\ref{eqn:ip}) is defined as 
\begin{alignat}{3}
\label{eqn:lp}
&\underset{\bfy}{\operatorname{\min}}\quad \tE(\bfy) &&= \sum_a\sum_i
\phi_{a:i}\,y_{a:i} + \sum_{a,b\ne a} \sum_i K_{ab}\frac{|y_{a:i}-y_{b:i}|}{2}\
,\\\nonumber
&\text{s.t.}\quad\bfy\in\calM &&= \left\{ \begin{array}{l|l}
\multirow{2}{*}{$\bfy$} & \sum_i y_{a:i} = 1,\, a \in \allpixels\\ 
 & y_{a:i} \ge 0,\, a \in \allpixels,\,  i\in\calL \end{array}
 \right\}\ ,
\end{alignat}
where $K_{ab} = \sum_c w^{(c)}\,k\left(\bff_a^{(c)},\bff_b^{(c)}\right)$. 
For integer labellings, the LP objective $\tE(\bfy)$ has the same value as
the IP objective $E(\bfy)$. 
It is also worth noting that this LP relaxation is known to provide the best
theoretical bounds~\cite{kleinberg2002approximation}.
Using standard solvers to minimize this LP would require the introduction of 
$\calO(n^2)$ variables, making it intractable. Therefore the
non-smooth objective of Eq.~(\ref{eqn:lp}) has to be optimized directly. This
was handled using projected subgradient descent
in~\cite{desmaison2016efficient}, which also turns out to be inefficient in
practice. 
In this paper, we introduce an efficient
algorithm to tackle this problem while maintaining \textit{linear} scaling in
both space and time complexity.

%% file: prox_lp.tex
\section{Proximal minimization for LP relaxation}
Our goal is to design an efficient minimization strategy for the LP relaxation
in~\eqref{eqn:lp}. To this end, we propose to use the proximal minimization
algorithm~\cite{parikh2014proximal}. This guarantees monotonic decrease in the
objective value, enabling us to leverage faster, less accurate methods for
initialization. Furthermore, the additional quadratic regularization
term makes the dual problem smooth, enabling the use of more sophisticated
optimization methods.
In the remainder of this paper, we detail
this approach and show that each iteration has linear time complexity. In
practice, our algorithm converges in a small number of iterations, thereby
making the overall approach computationally efficient. 

The proximal minimization algorithm~\cite{parikh2014proximal} is an iterative
method that, given the current estimate of the solution $\bfy^k$, solves the
problem
\begin{alignat}{3}
\label{eqn:proxlp0}
&\underset{\bfy}{\operatorname{\min}}\quad &&\tE(\bfy)
+ \frac{1}{2\lambda}\left\|\bfy - \bfy^k\right\|^2\ ,\\\nonumber
&\text{s.t.}\quad &&\bfy \in\calM\ ,
\end{alignat}
where $\lambda$ sets the strength of the proximal term.

Note that~\eqref{eqn:proxlp0} consists of piecewise linear terms and a
quadratic regularization term.
Specifically, the piecewise linear term comes from
the pairwise term $|y_{a:i} - y_{b:i}|$ in~\eqref{eqn:lp} that can
be reformulated as $\max\{y_{a:i} - y_{b:i}, y_{b:i} - y_{a:i}\}$.
 The proximal term $\|\bfy - \bfy^k\|^2$
provides the quadratic regularization.
In this section, we introduce a new
algorithm that is tailored to this problem. In 
particular, we optimally solve the Lagrange dual of~\eqref{eqn:proxlp0} in a
block-wise fashion.

\begin{algorithm*}[t]
\caption{Proximal minimization of LP}
\label{alg:proxlp}
\begin{algorithmic}
\Require Initial solution $\bfy^0\in\calM$ and the dual objective $g$

\For{$k \gets 0\ldots K$}

\State $A\bfalpha^0 \gets \bfzero,\quad \bfbeta^0 \gets \bfzero,\quad
\bfgamma^0 \gets \bfzero$
\Comment{Feasible initialization}

\For{$t\gets 0\ldots T$}

\State $\left(\bfbeta^t, \bfgamma^t\right) \gets
\underset{\bfbeta,\bfgamma}
{\operatorname{argmin}}\,g\left(\bfalpha^t,\bfbeta, \bfgamma\right)$
\Comment{Sec.~\ref{sec:gamma}}

\State $\tbfy^t \gets \lambda\left(
A\bfalpha^t+B\bfbeta^t + \bfgamma^t-\bfphi\right)+\bfy^k$
\Comment{Current primal solution, may be infeasible}


%

\State $A\bfs^t \gets $ conditional gradient of $g$, computed using $\tbfy^t$
\Comment{Sec.~\ref{sec:cgrad}}

\State $\delta \gets $ optimal step size given $\left(\bfs^t, \bfalpha^t,
\tbfy^t\right)$
\Comment{Sec.~\ref{sec:step}}

\State $A\bfalpha^{t+1} \gets (1-\delta)A\bfalpha^t + \delta A\bfs^t$
\Comment{Frank-Wolfe update on $\bfalpha$}

\EndFor

\State $\bfy^{k+1} \gets P_{\calM}\left(\tbfy^t\right)$
\Comment{Project the primal solution to the feasible set $\calM$}

\EndFor

\end{algorithmic}
\end{algorithm*}

\subsection{Dual formulation}
Let us first write the proximal problem~\eqref{eqn:proxlp0} in the standard
form by introducing auxiliary variables $z_{ab:i}$.
\begin{subequations}
\begin{alignat}{3}
\label{eqn:p}
&\underset{\bfy,\bfz}{\operatorname{min}}\quad \sum_a&&\sum_i
\phi_{a:i}\,y_{a:i} + \sum_{a,b\ne a} \sum_i \frac{K_{ab}}{2}z_{ab:i} +
\frac{1}{2\lambda}\|\bfy - \bfy^k\|^2 \ ,\\
\label{eqn:proxlpc1}
&\text{s.t.}\quad  z_{ab:i} &&\ge y_{a:i} - y_{b:i}\quad \forall\,a\ne b\quad
\forall\,i \in \calL\ ,\\
\label{eqn:proxlpc2}
&\hskip0.044\linewidth z_{ab:i} &&\ge y_{b:i} - y_{a:i}\quad \forall\,a\ne
b\quad \forall\,i \in \calL\ ,\\
\label{eqn:proxlpc3}
&\hskip0.015\linewidth\sum_i y_{a:i} &&= 1\quad \forall\,a \in \allpixels\
,\\
\label{eqn:proxlpc4}
&\hskip0.05\linewidth y_{a:i} &&\ge 0\quad \forall\,a \in \allpixels\quad
\forall\,i\in\calL\ .
\end{alignat}
\end{subequations}
We introduce three blocks of dual variables. Namely,
$\bfalpha = \{\alpha^1_{ab:i}, \alpha^2_{ab:i}\mid a\ne b, i \in \calL\}$ for
the constraints in Eqs.~(\ref{eqn:proxlpc1}) and~(\ref{eqn:proxlpc2}), $\bfbeta =
\{\beta_{a}\mid a \in \allpixels\}$ for Eq.~(\ref{eqn:proxlpc3}) and $\bfgamma =
\{\gamma_{a:i}\mid a \in \allpixels, i \in \calL\}$ for Eq.~(\ref{eqn:proxlpc4}),
respectively. The vector $\bfalpha$ has $p=2n(n-1)m$ elements.
Here, we introduce two matrices that will be
useful to write the dual problem compactly. 

\begin{dfn}
Let $A \in \R^{nm\times p}$ and $B\in\R^{nm\times n}$ be two matrices such that
\begin{align}
\left(A\bfalpha\right)_{a:i} &= -\sum_{b\ne a}\left(\alpha^1_{ab:i} -
\alpha^2_{ab:i} + \alpha^2_{ba:i} - \alpha^1_{ba:i}\right)\ ,\\\nonumber
\left(B\bfbeta\right)_{a:i} &= \beta_a\ .
\end{align}
\end{dfn}

We can now state our first proposition.

\begin{pro} 
\label{pro:dual}
Given matrices $A \in\R^{nm\times p}$ and $B\in\R^{nm\times n}$ and dual
variables $(\bfalpha,\bfbeta,\bfgamma)$.
\begin{enumerate}
  \item The Lagrange dual of~\eqref{eqn:p} takes the following form:
	\begin{alignat}{3}
	\label{eqn:dual}
	&\underset{\bfalpha,\bfbeta,\bfgamma}{\operatorname{min}}\ g(\bfalpha,
	\bfbeta, \bfgamma) &&= \frac{\lambda}{2}\|A\bfalpha +
	B\bfbeta+\bfgamma-\bfphi\|^2 + \left\langle A\bfalpha +
	B\bfbeta+\bfgamma-\bfphi, \bfy^k \right\rangle - \langle \bfone, \bfbeta
	\rangle\ ,\\\nonumber &\text{s.t.}\hskip0.08\linewidth \gamma_{a:i} &&\ge
	0\quad\forall\,a\in\allpixels\quad \forall\,i \in \calL\ , \\\nonumber
	&\hskip0.08\linewidth\bfalpha \in \calC &&= \left\{ \begin{array}{l|l}
	\multirow{2}{*}{$\bfalpha$} & \alpha^1_{ab:i} + \alpha^2_{ab:i} =
	\frac{K_{ab}}{2},\, \forall\,a\ne b,\, \forall\,i \in \calL\\
 	& \alpha^1_{ab:i}, \alpha^2_{ab:i} \ge 0,\,\forall\,a\ne
 	b,\, \forall\,i \in \calL \end{array} \right\}\ .
	\end{alignat}
  \item The primal variables $\bfy$ satisfy
	\begin{equation}
	\label{eqn:cp}
	\bfy = \lambda \left(A\bfalpha+B\bfbeta+\bfgamma-\bfphi\right)+\bfy^k\ .
	\end{equation}
\end{enumerate}
\end{pro}
\begin{proof}
In Appendix~\ref{app:dual}.
\end{proof}


\subsection{Algorithm}
The dual problem~\eqref{eqn:dual}, in its standard form, can only
be tackled using projected gradient descent. However, by separating the
variables based on the type of the feasible domains, we propose an efficient
block coordinate descent approach. Each of these blocks are amenable to more
sophisticated optimization, resulting in a computationally efficient algorithm. 
As the dual problem is strictly convex and smooth, the optimal solution is still
guaranteed. For $\bfbeta$ and $\bfgamma$, the problem decomposes
over the pixels, as shown in~\ref{sec:gamma}, therefore making it efficient. 
The minimization with respect to $\bfalpha$ is over a compact domain,
which can be efficiently tackled using the Frank-Wolfe
algorithm~\cite{frank1956algorithm,lacoste2012block}.
 Our complete algorithm is summarized in
Algorithm~\ref{alg:proxlp}.
In the following sections, we discuss each step in more detail.

\subsubsection{Optimizing over $\bfbeta$ and $\bfgamma$}\label{sec:gamma}
We first turn to the problem of optimizing over $\bfbeta$
and $\bfgamma$ while $\bfalpha^t$ is fixed. 
Since the dual variable $\bfbeta$ is
unconstrained, the minimum value of the dual objective $g$ is attained when
$\nabla_{\bfbeta}g(\bfalpha^t, \bfbeta,\bfgamma)=0$. 

\begin{pro}
If $\nabla_{\bfbeta}g(\bfalpha^t, \bfbeta,\bfgamma)=0$, then 
$\bfbeta$ satisfy
\begin{equation}
\bfbeta = B^T\left(A\bfalpha^t + \bfgamma - \bfphi\right)/m\ .
\label{eqn:beta}
\end{equation}
\end{pro}
\begin{proof}
More details on the simplification is given in
Appendix~\ref{app:gamma}.
\end{proof}

%
Note that, now, $\bfbeta$ is a function of $\bfgamma$.
We therefore substitute $\bfbeta$ in~\eqref{eqn:dual} and minimize over
$\bfgamma$. Interestingly, the resulting problem can be optimized independently
for each pixel, with each subproblem being an $m$ dimensional quadratic program (QP)
 with nonnegativity constraints, where $m$ is the number of labels.
 
\begin{pro}
The optimization over $\bfgamma$ decomposes over pixels where for a pixel $a$,
this QP has the form
\begin{alignat}{3}
\label{eqn:qpgammas}
&\underset{\bfgamma_a}{\operatorname{\min}}\quad 
&&\frac{1}{2}\bfgamma^T_aQ\bfgamma_a + \left\langle \bfgamma_a, 
Q\left((A\bfalpha^t)_a-\bfphi_a\right) + \bfy^k_a \right\rangle \
,\\\nonumber 
&\text{s.t.}\quad &&\bfgamma_a\ge \bfzero\ .
\end{alignat}
Here, $\bfgamma_a$ denotes the vector $\{\gamma_{a:i}\mid
i\in\calL\}$ and $Q= \lambda\left(I-\bfone/m\right)\in\R^{m\times m}$, with $I$
the identity matrix and $\bfone$ the matrix of all ones.
\end{pro}
\begin{proof}
In Appendix~\ref{app:gamma}.
\end{proof}
 

We use the algorithm
of~\cite{xiao2014multiplicative} to efficiently optimize every such QP. In
our case, due to the structure of the matrix $Q$, the time complexity of an
iteration is linear in the number of labels.
Hence, the overall time complexity of optimizing over $\bfgamma$ is
$\calO(nm)$. Once the optimal $\bfgamma$ is computed for a given $\bfalpha^t$, the
corresponding optimal $\bfbeta$ is given by Eq.~(\ref{eqn:beta}).

\subsubsection{Optimizing over $\bfalpha$}\label{sec:cgrad}
We now turn to the problem of optimizing over $\bfalpha$ given $\bfbeta^t$ and
$\bfgamma^t$.
To this end, we use the Frank-Wolfe algorithm~\cite{frank1956algorithm},
which has the advantage of being projection free.
Furthermore, for our specific problem, we show that the required conditional
gradient can be computed efficiently and the optimal step size can
be obtained analytically.

\paragraph{Conditional gradient computation.}
The conditional gradient with respect to $\bfalpha$ is obtained by solving 
the following linearization problem
\begin{equation}
\bfs = \underset{\hbfs\in\calC} {\operatorname{argmin}}\,
 \left\langle \hbfs,\nabla_{\bfalpha} g(\bfalpha^t, \bfbeta^t, \bfgamma^t)
 \right\rangle\ .
\end{equation}
Here, $\nabla_{\bfalpha} g(\bfalpha^t, \bfbeta^t, \bfgamma^t)$ denotes the
gradient of the dual objective function with respect to $\bfalpha$ evaluated at
$(\bfalpha^t, \bfbeta^t, \bfgamma^t)$. 

\begin{pro}
The conditional gradient $\bfs$ satisfy
\begin{equation}
\label{eqn:cgrad}
\left(A\bfs\right)_{a:i} = -\sum_b \left(K_{ab} \I[\ty^t_{a:i} \ge
\ty^t_{b:i}] - K_{ab} \I[\ty^t_{a:i} \le \ty^t_{b:i}]\right)\ ,
\end{equation}
where $\tbfy^t = \lambda\left(A\bfalpha^t+B\bfbeta^t+\bfgamma^t-\bfphi\right) +
\bfy^k$ using Eq.~\eqref{eqn:cp}.
\end{pro}
\begin{proof}
In Appendix~\ref{app:cgrad}.
\end{proof}


Note that Eq.~\eqref{eqn:cgrad} has the same form as the LP
subgradient (Eq.~(\myref{20}) in~\cite{desmaison2016efficient}). This is not a
surprising result. In fact, it has been shown that, for certain problems, there
exists a duality relationship between subgradients and conditional
gradients~\cite{bach2015duality}.
To compute this subgradient, the state-of-the-art algorithm proposed
in~\cite{desmaison2016efficient} has a time complexity linearithmic in the number of
pixels.
Unfortunately, since this constitutes a critical step of both our algorithm and that
of~\cite{desmaison2016efficient}, such a linearithmic cost greatly affects their 
efficiency. In Section~\ref{sec:ph}, however, we show that this complexity can
be reduced to linear, thus effectively leading to a speedup of an order of
magnitude in practice.

\paragraph{Optimal step size.}\label{sec:step}
One of the main difficulties of using an iterative algorithm, whether
subgradient or conditional gradient descent, is that its performance depends
critically on the choice of the step size. Here, we can analytically compute the
optimal step size that results in the maximum decrease in
the objective for the given descent direction.

\begin{pro}
The optimal step size $\delta$ satisfy
\begin{equation}
\delta = P_{[0,1]}\left(\frac{\langle A\bfalpha^t - A\bfs^t,
\tbfy^t\rangle}{\lambda\|A\bfalpha^t - A\bfs^t\|^2}\right)\ .
\end{equation}
Here, $P_{[0,1]}$ denotes the projection to the interval $[0,1]$, that is,
clipping the value to lie in $[0,1]$.
\end{pro}
\begin{proof}
In Appendix~\ref{app:step}.
\end{proof}
 

\paragraph{Memory efficiency.}
For a dense CRF, the dual variable $\bfalpha$ requires
$\calO(n^2m)$ storage, which becomes infeasible since $n$ is the number
of pixels in an image. Note, however, that $\bfalpha$ always appears in the 
product $\tbfalpha = A\bfalpha$ in Algorithm~\ref{alg:proxlp}.
Therefore, we only store the variable $\tbfalpha$, which reduces the storage
complexity to $\calO(nm)$. 

\subsubsection{Summary}
To summarize, our method has four desirable qualities of an efficient
iterative algorithm. 
First, it can benefit from an initial
solution obtained by a faster but less accurate algorithm,
such as mean-field or DC relaxation. Second, with our choice of a quadratic proximal
term, the dual of the proximal problem can be efficiently optimized in a block-wise
fashion. Specifically, the dual variables $\bfbeta$ and $\bfgamma$ are computed
efficiently by minimizing one small QP (of dimension the
number of labels) for each pixel independently.
The remaining dual variable $\bfalpha$ is optimized using the
Frank-Wolfe algorithm, where the conditional gradient is computed in linear
time, and the optimal step size is obtained analytically. Overall, the time
complexity of one iteration of our algorithm is $\calO(nm)$.
To the best of our knowledge, this constitutes the first LP minimization
algorithm for dense CRFs that has linear time iterations.
We denote this standard algorithm as PROX-LP.

%% file: fast_ph.tex
\section{Fast conditional gradient computation}\label{sec:ph}
The algorithm described in the previous section assumes that the conditional
gradient (Eq.~(\ref{eqn:cgrad})) can be computed efficiently. Note that
Eq.~(\ref{eqn:cgrad}) contains two terms that are similar up to sign and 
order of the label constraint in the indicator function. To simplify the
discussion, let us focus on the first term and on a particular
label $i$, which we will not explicitly write in the remainder of this section.
The second term in Eq.~(\ref{eqn:cgrad}) and the other labels can be handled in the
same manner. With these simplifications, we need to efficiently
compute an expression of the form
\begin{equation}
\forall\,a\in\allpixels, \quad v'_a = \sum_b k(\bff_a, \bff_b)\,\I[y_a \ge
y_b]\ ,
\label{eqn:oph}
\vspace{-0.2cm}
\end{equation}
with $y_a, y_b \in [0,1]$ and $\bff_a, \bff_b\in\R^d$ for all $a,b\in\allpixels$.

The usual way of speeding up computations involving such Gaussian kernels is by
using the efficient filtering method~\cite{adams2010fast}.
This approximate method has proven accurate enough for similar 
applications~\cite{desmaison2016efficient,koltun2011efficient}. 
In our case, due to the ordering constraint \mbox{$\I[y_a\ge
y_b]$}, the symmetry is broken and the direct application of the
filtering method is impossible.
In~\cite{desmaison2016efficient}, the authors tackled this problem
using a divide-and-conquer strategy, which lead to a time complexity of $\calO(d^2n
\log(n))$. In practice, this remains a prohibitively high run time, particularly
since
gradient computations are performed many times over the course of the algorithm. 
Here, we introduce a more efficient method.

Specifically, we show that the term in Eq.~\eqref{eqn:oph} can be computed in $\calO(Hdn)$ time
(where $H$ is a small constant defined in Section~\ref{sec:phnew}), at the cost
of additional storage.
In practice, this leads to a speedup of one order of magnitude. Below, 
we first briefly review the original filtering
algorithm and then explain our modified algorithm that efficiently
handles the ordering constraints.

\subsection{Original filtering method}
In this section, we assume that the reader is familiar with the permutohedral
lattice based filtering method~\cite{adams2010fast} andonly a brief overview is
provided. We refer the interested reader to the original paper~\cite{adams2010fast}.

In~\cite{adams2010fast}, each pixel $a\in\allpixels$ is associated with
a tuple $\left(\bff_a, v_a\right)$, which we call a \textit{feature point}.
The elements of this tuple are the feature $\bff_a\in\R^d$ and the value $v_a\in
\R$. Note that, in our case, $v_a=1$ for all pixels.
At the beginning of the algorithm, the feature points are embedded in
a $d$-dimensional hyperplane tessellated by the \textit{permutohedral lattice}
(see Fig.~\ref{fig:ph}).
The vertices of this permutohedral lattice are called \textit{lattice points}, 
and each lattice point $l$ is associated with a scalar value $\barv_l$.
\begin{figure}
\begin{center}
\includegraphics[width=0.4\linewidth, trim=1.8cm 5.3cm 12.5cm 2.4cm, clip=true,
page=1]{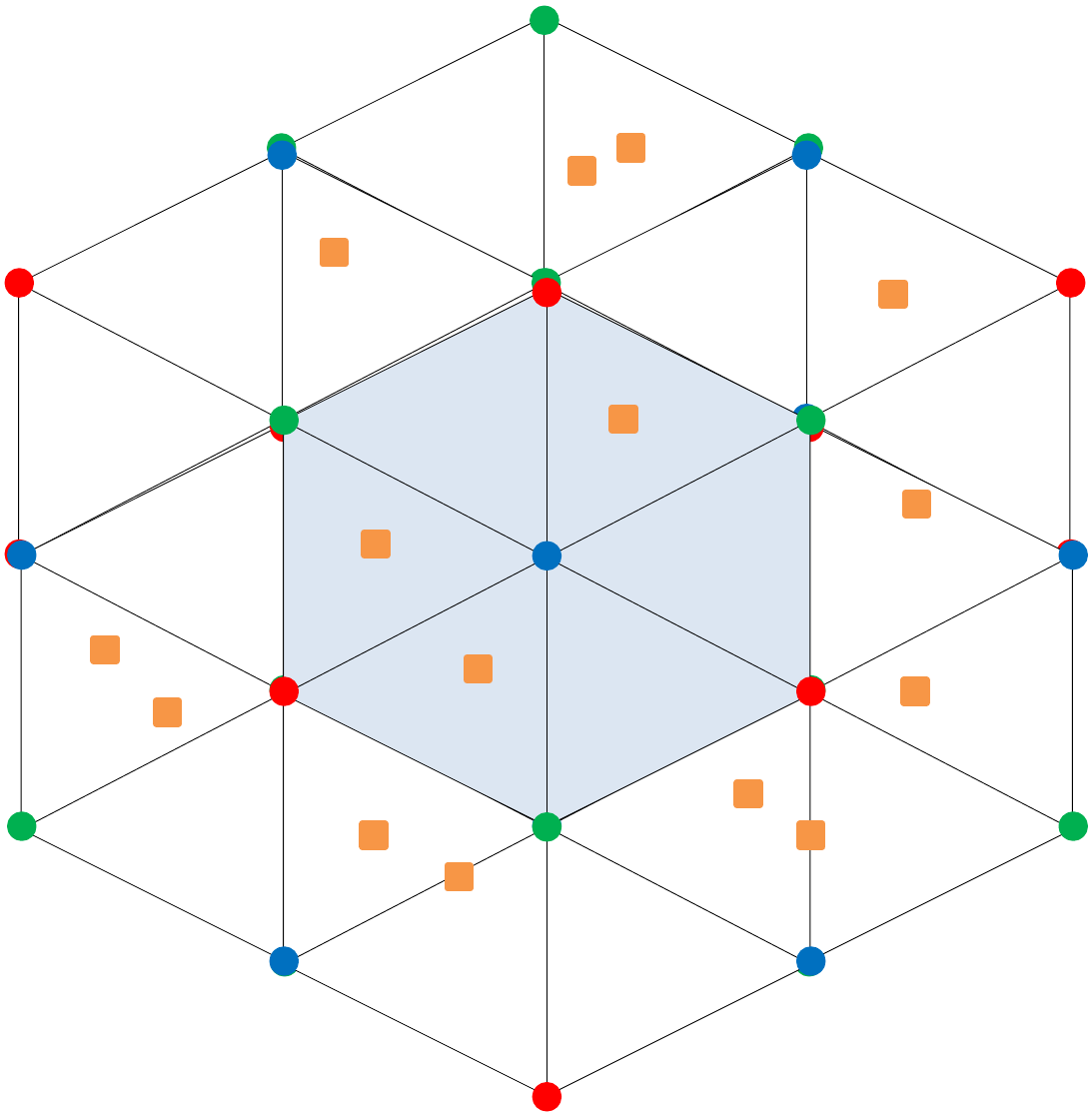}	
\vspace{-0.2cm}
\caption{\em A $2$-dimensional hyperplane tessellated by the permutohedral
lattice. The feature points are denoted with squares and the lattice
points with circles. The neighborhood of the center lattice point is
shaded and, for a feature point, the neighbouring
lattice points are the vertices of the enclosing triangle.}
\label{fig:ph}
\end{center}
\vspace{-0.5cm}
\end{figure}

Once the permutohedral lattice is constructed, the algorithm performs three main
steps:
\textit{splatting}, \textit{blurring} and \textit{slicing}.
During splatting, for each lattice point, 
the values of the neighbouring feature points are accumulated
using barycentric interpolation. Next, during blurring,
the values of the lattice points are convolved with a one dimensional truncated
Gaussian kernel along each feature dimension separately.
Finally, during slicing, the resulting values of the lattice points are
propagated back to the feature points using the same barycentric weights.
These steps are explained graphically in the top row of Fig.~\ref{fig:phalg}.
The pseudocode of the algorithm is given in Appendix~\ref{app:phold}. The time
complexity of this algorithm is
$\calO(dn)$~\cite{adams2010fast,koltun2011efficient}, and the complexity of the
permutohedral lattice creation $\calO(d^2n)$. Since the approach
in~\cite{desmaison2016efficient} creates multiple lattices at every iteration,
the overall complexity of this approach is $\calO(d^2n \log(n))$.

\begin{figure}
\def\IMHEIGHT{26ex}
\begin{center}
\begin{subfigure}{0.3\linewidth}
\begin{framed}
\includegraphics[height=\IMHEIGHT, trim=4.3cm 5.0cm 9.4cm 2.2cm, clip=true,
page=7]{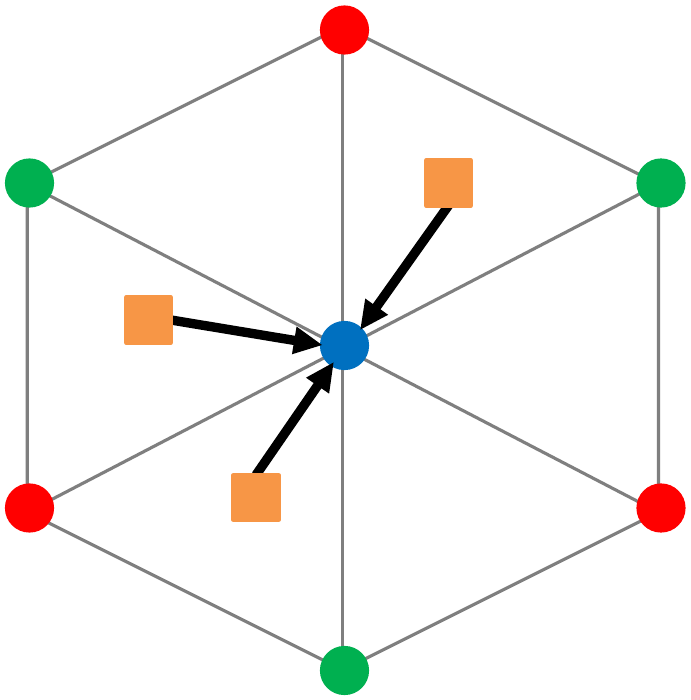}	
\end{framed}
\end{subfigure}%
\begin{subfigure}{0.4\linewidth}
\begin{framed}
\includegraphics[height=\IMHEIGHT, trim=4.3cm 5.0cm 5.9cm 2.2cm, clip=true,
 page=19]{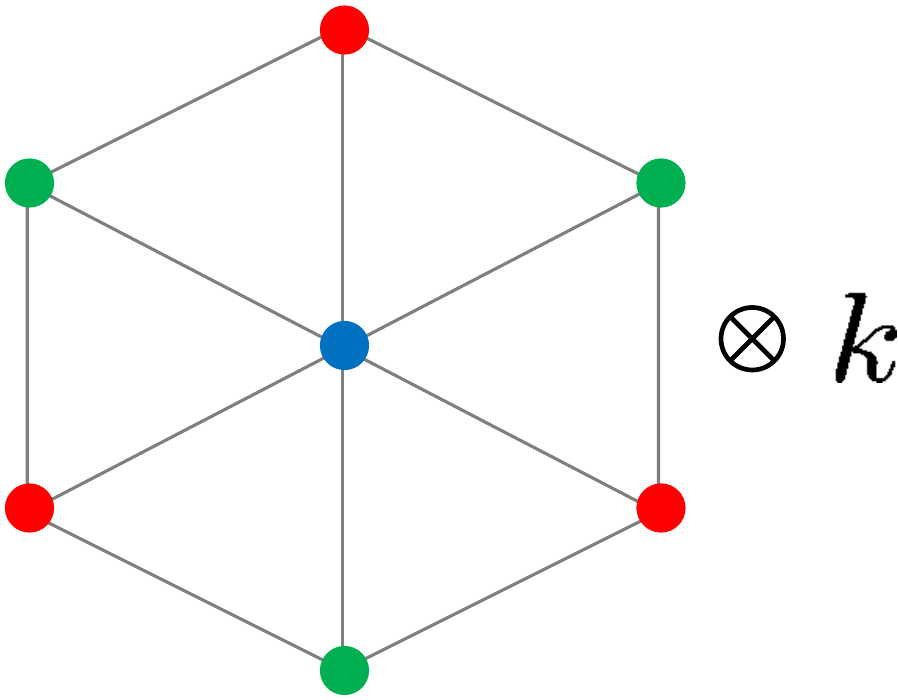}	
\end{framed}
\end{subfigure}%
\begin{subfigure}{0.3\linewidth}
\begin{framed}
\includegraphics[height=\IMHEIGHT, trim=4.3cm 5.0cm 9.4cm 2.2cm, clip=true,
page=21]{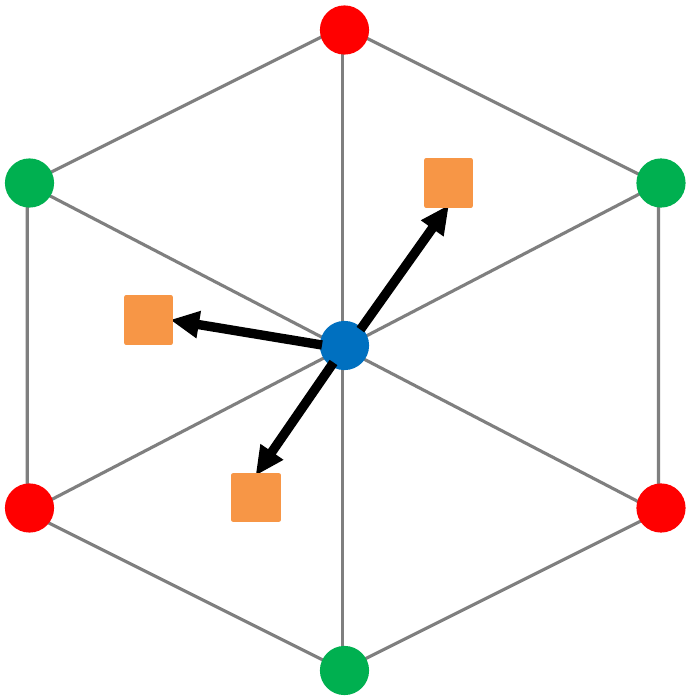}	
\end{framed}
\end{subfigure}
\begin{subfigure}{0.3\linewidth}
\begin{framed}
\includegraphics[height=\IMHEIGHT, trim=4.2cm 6.3cm 14.1cm 4.0cm, clip=true,
page=14]{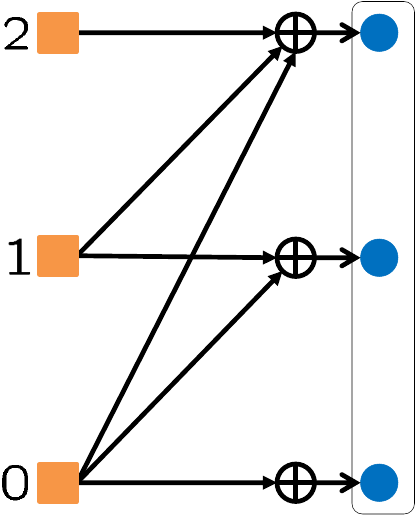}	
\end{framed}
\vspace{-0.3cm}
\caption{Splat}
\end{subfigure}%
\begin{subfigure}{0.4\linewidth}
\begin{framed}
\includegraphics[height=\IMHEIGHT, trim=4.0cm 6.3cm 10.1cm 4.0cm, clip=true,
page=23]{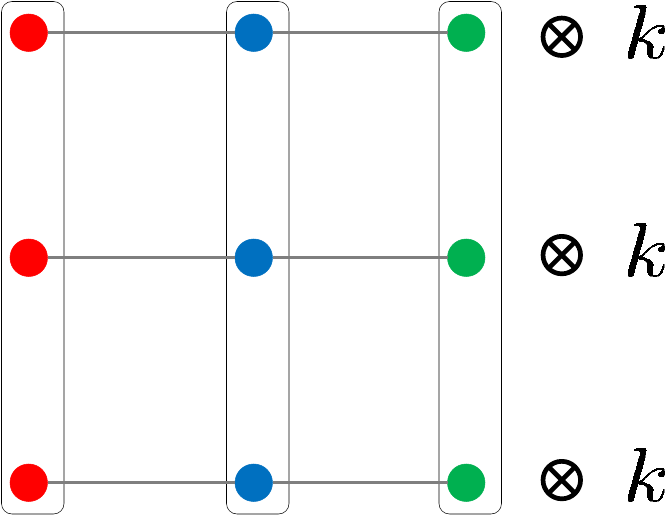} 
\end{framed}
\vspace{-0.3cm}
\caption{Blur}
\end{subfigure}%
\begin{subfigure}{0.3\linewidth}
\begin{framed}
\includegraphics[height=\IMHEIGHT, trim=4.2cm 6.3cm 14.1cm 4.0cm, clip=true,
page=22]{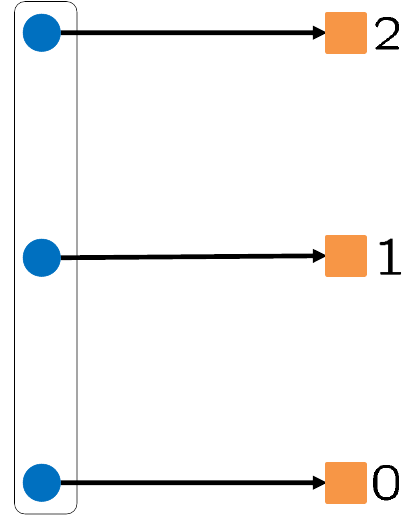}	
\end{framed}
\vspace{-0.3cm}
\caption{Slice}
\end{subfigure}

\vspace{-0.2cm}
\caption{\em {\bf Top row: }Original filtering method.
The barycentric interpolation is denoted by an arrow and $k$ here is the truncated
Gaussian kernel. 
During splatting, for each lattice point, the values of
the neighbouring feature points are accumulated using
barycentric interpolation. Next, the lattice points are blurred with $k$ along
each dimension. Finally, the values of the lattice points are given back to the
feature points using the same barycentric weights.
{\bf Bottom row: }Our modified filtering method.
Here, $H = 3$, and the figure therefore illustrates 3 lattices.
We write the bin number of each feature point next to the
point. Note that, at the splatting step, the value 
of a feature point is accumulated to its neighbouring lattice points only if it is
above or equal to the feature point level. Then, blurring is performed at each
level independently. Finally, the resulting values are recovered from the
lattice points at the feature point level.}
\label{fig:phalg}
\end{center}
\vspace{-0.5cm}
\end{figure}

Note that, in this original algorithm, there is no notion of score $y_a$
associated with each pixel. In particular, during splatting, the values $v_a$
are accumulated to the neighbouring lattice points without considering their
scores.
Therefore, this algorithm cannot be directly applied to handle our ordering
constraint $\I[y_a\ge y_b]$.

\subsection{Modified filtering method} \label{sec:phnew}
We now introduce a filtering-based algorithm that can handle ordering
constraints. To this end, we uniformly discretize the continuous 
interval $[0,1]$ into $H$ different discrete bins, or levels. Note that each
pixel, or feature point, belongs to exactly one of these bins, according to its
corresponding score. We then propose to instantiate $H$ permutohedral
lattices, one for each level $h\in\{0\ldots H-1\}$. In other words, at each
level $h$, there is a lattice point $l$, whose value we denote by $\barv_{l:h}$.
To handle the ordering constraints, we then modify the splatting step in the
following manner. A feature point belonging to bin $q$ is splat to the
permutohedral lattices corresponding to levels $q \leq h < H$. Blurring is
then performed independently in each individual permutohedral lattice. This
guarantees that a feature point will only influence the values of the feature
points that belong to the same level or higher ones. In other words, a feature
point $b$ influences the value of a feature point $a$ only if $y_a \ge y_b$. 
Finally, during the slicing step, the value of a feature point belonging to
level $q$ is recovered from the $q^{\rm th}$ permutohedral lattice. Our
algorithm is depicted graphically in the bottom row of Fig.~\ref{fig:phalg}. Its
pseudocode is provided in Appendix~\ref{app:phnew}. Note that, while discussed
for constraints of the form $\I[y_a\ge y_b]$, this algorithm can easily be
adapted to handle $\I[y_a\le y_b]$ constraints, which are required for the
second term in Eq.~\eqref{eqn:cgrad}.

Overall, our modified filtering method has a time complexity of
$\calO(Hdn)$ and a space complexity of $\calO(Hdn)$. 
Note that the complexity of the lattice creation is still $\calO(d^2n)$ and
can be reused for each of the $H$ instances. Moreover, as opposed to the method
in~\cite{desmaison2016efficient}, this operation is performed only once, during
the initialization step.
In practice, we were able to choose $H$ as small as 10, thus
achieving a substantial speedup compared to the divide-and-conquer strategy of~\cite{desmaison2016efficient}.
By discretizing the interval $[0,1]$, we add another level
of approximation to the overall algorithm. However, this approximation can be
eliminated by using a dynamic data structure, which we briefly explain in
Appendix~\ref{app:dynph}.

%% file: related_work.tex
\section{Related work}
We review the past work on three different aspects of our work in order to
highlight our contributions.

\paragraph{Dense CRF.}
The fully-connected CRF 
 has become increasingly popular for semantic segmentation. It
is particularly effective at preventing oversmoothing, thus providing
better accuracy at the boundaries of objects. As a matter of fact, in a
complementary direction, many methods have now proposed to combine dense CRFs
with convolutional neural
networks~\cite{chen2014semantic,schwing2015fully,zheng2015conditional} to
achieve state-of-the-art performance on segmentation benchmarks.

The main challenge that had previously prevented the use of dense CRFs
is their computational cost at inference, which, naively, is $\calO(n^2)$ per
iteration.
In the case of Gaussian pairwise potential, the efficient filtering method
of~\cite{adams2010fast} proved to be key to the tractability of inference in
the dense CRF.
While an approximate method, the accuracy of the computation proved sufficient
for practical purposes.
This was first observed in~\cite{koltun2011efficient} for the
specific case of mean-field inference. More recently, several continuous
relaxations, such as QP, DC and LP, were also shown to be
applicable to minimizing the dense CRF energy by exploiting this
filtering procedure in various ways~\cite{desmaison2016efficient}.
Unfortunately, while tractable, minimizing the LP relaxation, which is known 
to provide the best approximation to the original 
labelling problem, remained too slow in practice~\cite{desmaison2016efficient}.
Our algorithm is faster both theoretically and empirically. 
Furthermore, and as evidenced by
our experiments, it yields lower energy values than any existing dense CRF
inference strategy.

\paragraph{LP relaxation.} 
There are two ways to relax the integer program
\eqref{eqn:ip} to a linear program, depending on the label compatibility
function: 1) the standard LP relaxation~\cite{chekuri2004linear}; and 2)
the LP relaxation specialized to the Potts model~\cite{kleinberg2002approximation}.
There are many notable works on
minimizing the standard LP relaxation on sparse CRFs. This includes the
algorithms that directly make use the dual of this
LP~\cite{kolmogorov2006convergent,komodakis2011mrf,wainwright2005map}
and those based on a proximal minimization framework~\cite{meshi2015smooth,ravikumar2008message}.
Unfortunately, all of the above algorithms exploit the sparsity of the problem,
and they would yield an $\calO(n^2)$ cost per iteration in the fully-connected
case.
In this work, we focus on the Potts model based LP relaxation for dense
CRFs and provide an algorithm whose iterations have time complexity $\calO(n)$.
Even though we focus on the Potts
model, as pointed out in~\cite{desmaison2016efficient}, this LP relaxation can be extended
to general label compatibility functions using a hierarchical Potts
model~\cite{kumar2009map}.

\paragraph{Frank-Wolfe.}
The optimization problem of structural support vector machines (SVM) has a form
similar to our proximal problem. The Frank-Wolfe algorithm~\cite{frank1956algorithm} was shown
to provide an effective and efficient solution to such a problem via block-coordinate
optimization~\cite{lacoste2012block}. Several works have recently focused on
improving the performance of this
algorithm~\cite{osokin2016minding,shah2015multi} and extended its application
domain~\cite{krishnan2015barrier}.
Our work draws inspiration from this structural SVM literature, and makes use of
the Frank-Wolfe algorithm to solve a subtask of our overall LP minimization
method. Efficiency, however, could only be achieved thanks to our modification
of the efficient filtering procedure to handle ordering constraints.

To the best of our knowledge, our approach constitutes the first LP
minimization algorithm for dense CRFs to have linear time iterations. Our
experiments demonstrate the importance of this result on both speed and 
labelling quality. Being fast, our algorithm can be incorporated in
any end-to-end learning framework, such as~\cite{zheng2015conditional}. We
therefore believe that it will have a significant impact on future semantic
segmentation results, and potentially in other application domains.

%% file: experiments.tex
\section{Experiments}
In this section, we will first discuss two variants that 
further speedup our algorithm and some implementation details. We then turn to the empirical results.

\subsection{Accelerated variants}\label{sec:proxlpacc}
Empirically we observed that, 
our algorithm can be accelerated by restricting the optimization
procedure to affect only relevant subsets of labels and pixels. These subsets
can be identified from an intermediate solution of PROX-LP.
In particular, we remove the label $i$ from the optimization if $y_{a:i} <
0.01$ for all pixels $a$. In other words, the score of a label
$i$ is insignificant for all the pixels. We denote this version as
PROX-LP$_\ell$. Similarly, we optimize over a pixel only if it is
\textit{uncertain} in choosing a label.
Here, a pixel $a$ is called uncertain if $\max_i y_{a:i} < 0.95$. In other words, 
no label has a score higher than $0.95$.
The intuition behind this strategy is that, after a few iterations of
PROX-LP$_\ell$, most of the pixels are labelled correctly, 
and we only need to fine tune the few remaining ones. 
In practice, we limit this restricted set to 
$10\%$ of the total number of pixels. We denote
this accelerated algorithm as PROX-LP$_\text{acc}$.
As shown in our experiments, PROX-LP$_\text{acc}$ yields
a significant speedup at virtually no loss in the quality of the results.

\begin{figure*}
\def\IMHEIGHT{30ex}
\def\IMHEIGHTZ{27ex}
\begin{center}
\begin{subfigure}{0.34\linewidth}
\includegraphics[height=\IMHEIGHT]{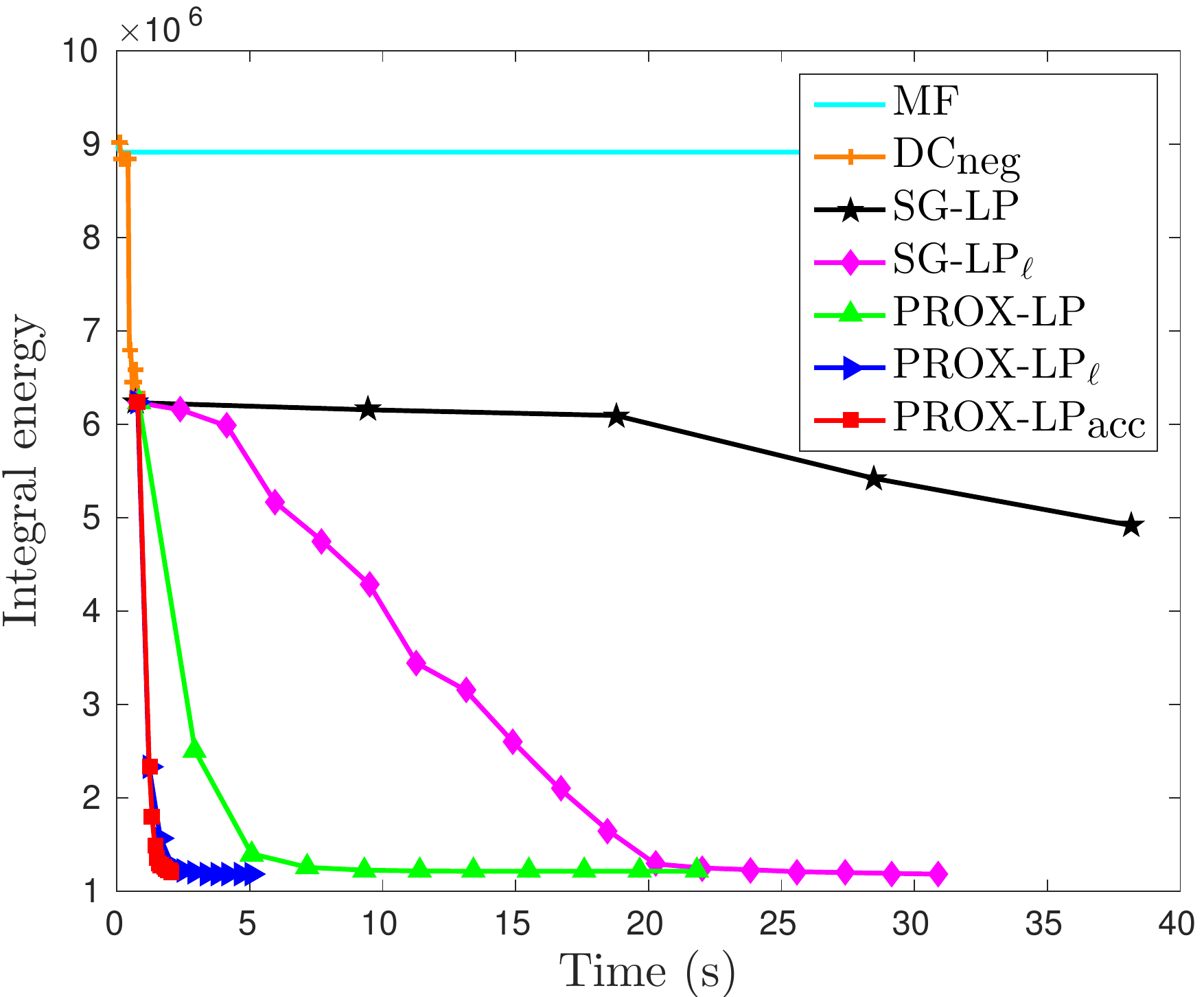}
\end{subfigure}%
\hfill
\begin{subfigure}{0.16\linewidth}
\includegraphics[height=\IMHEIGHTZ]{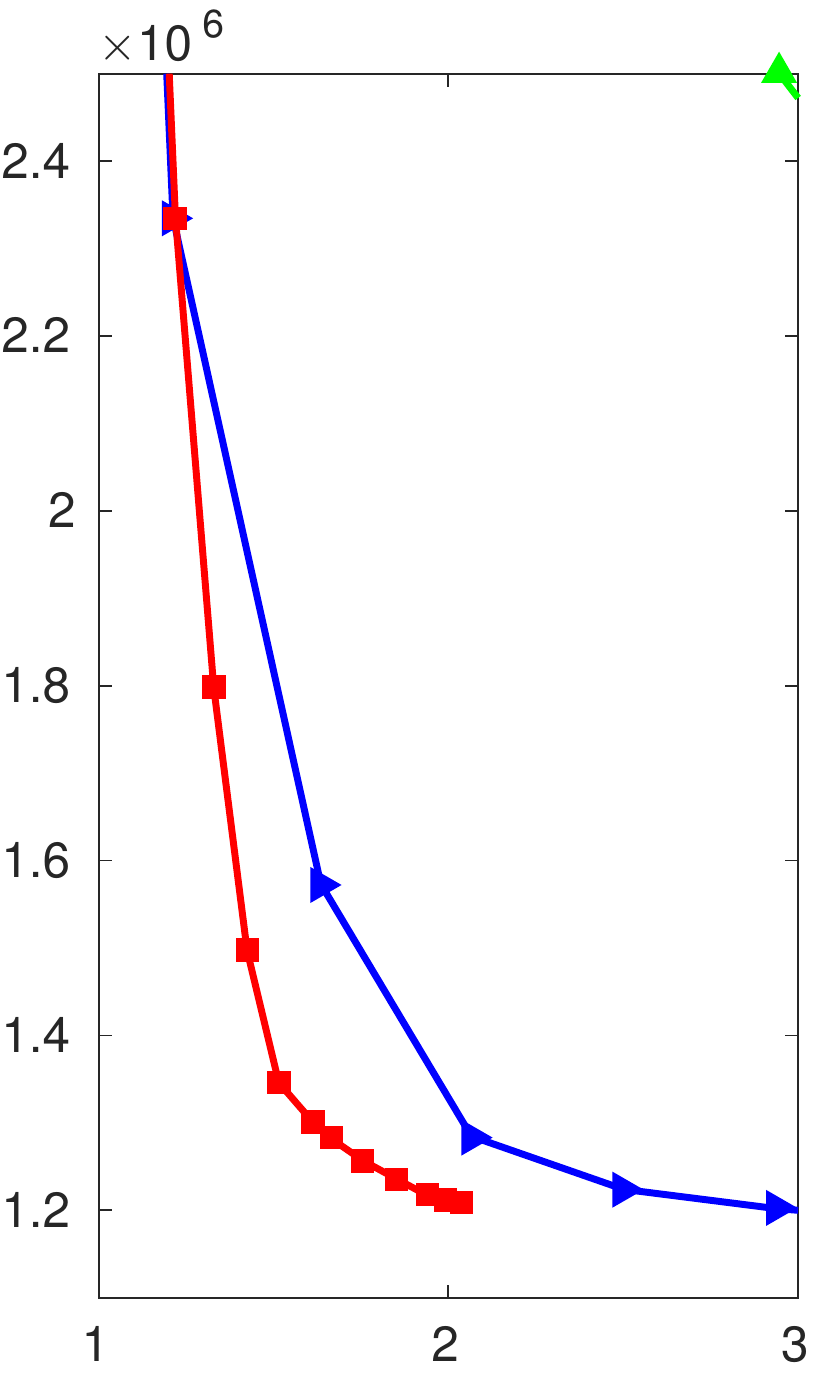}
\end{subfigure}%
\hfill
\begin{subfigure}{0.34\linewidth}
\includegraphics[height=\IMHEIGHT]{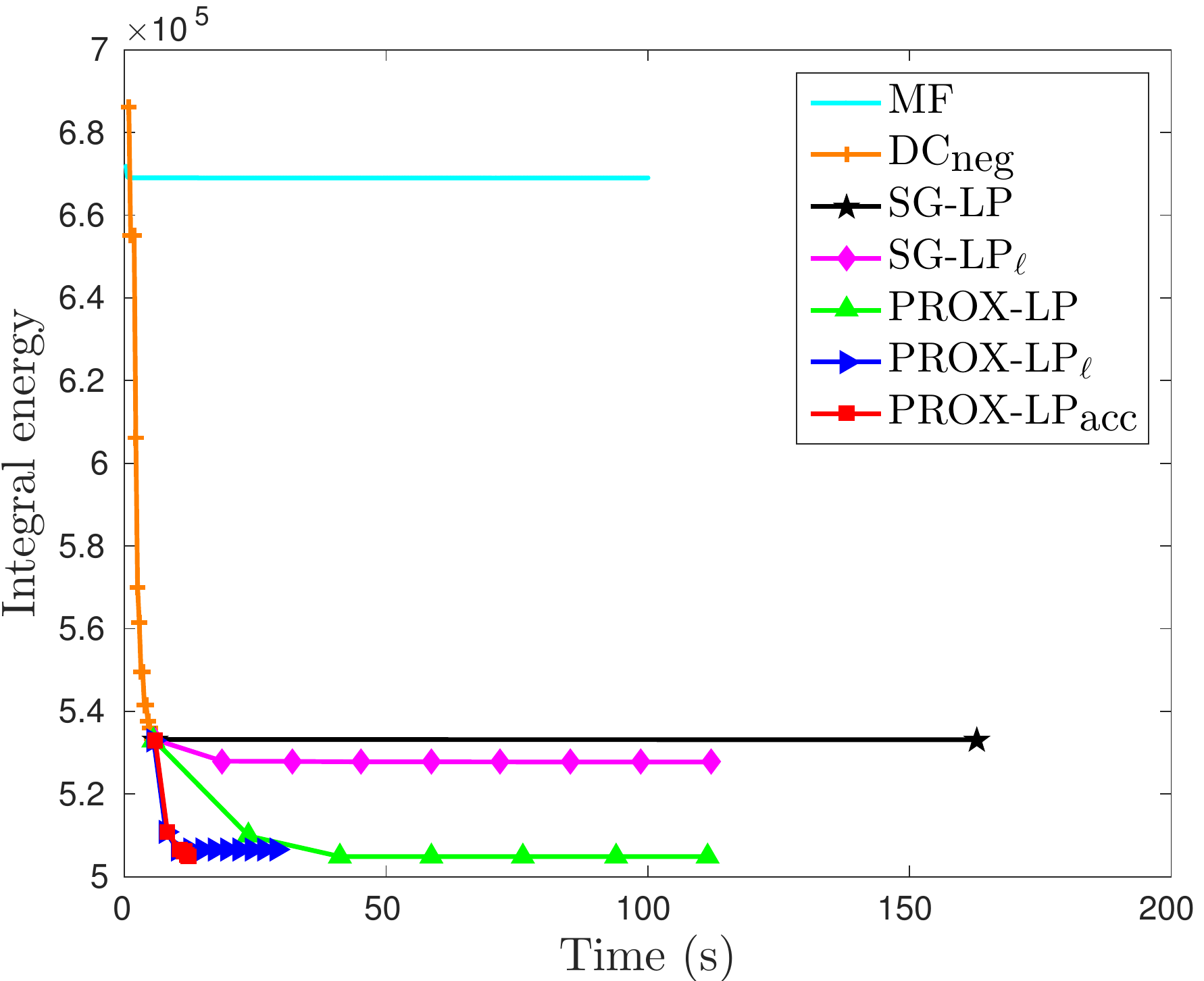}
\end{subfigure}%
\hfill
\begin{subfigure}{0.16\linewidth}
\includegraphics[height=\IMHEIGHTZ]{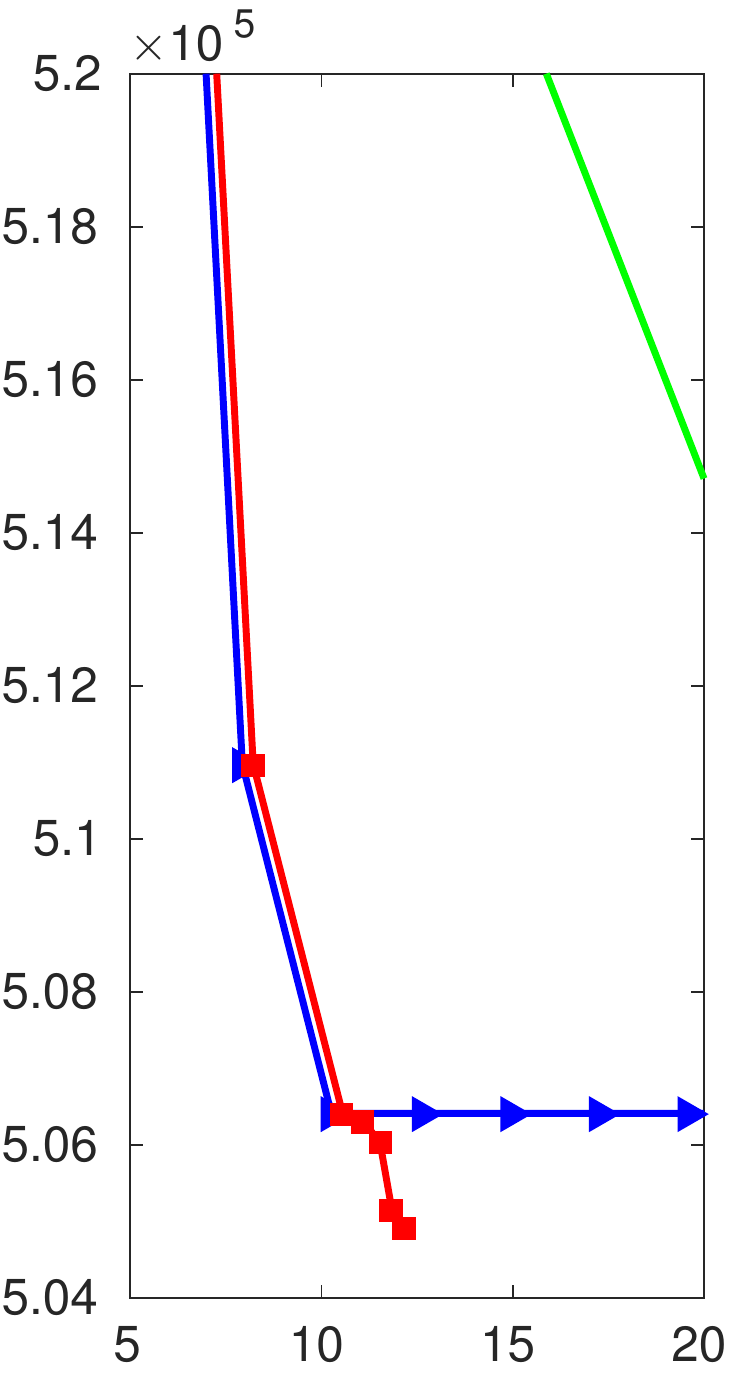}
\end{subfigure}

\vspace{-0.3cm}
\caption{\em Assignment energy as a function of time for DC$_\text{neg}$
parameters for an image in (\textbf{left}) MSRC and (\textbf{right}) Pascal.
A zoomed-in version is shown next to each plot.
Except MF, all other algorithms are initialized with DC$_\text{neg}$.
Note that PROX-LP clearly outperforms SG-LP$_\ell$ by
obtaining much lower energies in fewer iterations.
Furthermore, the accelerated versions of our algorithm obtain roughly the 
same energy as PROX-LP
but significantly faster.}
\label{fig:et}
\end{center}
\vspace{-0.5cm}
\end{figure*}

\begin{table*}[t]
\begin{center}
\begin{tabular}
{>{\raggedright\arraybackslash}m{0.15cm}|
>{\raggedright\arraybackslash}m{1.8cm}|>{\raggedleft\arraybackslash}m{0.5cm}
>{\raggedleft\arraybackslash}m{0.5cm}>{\raggedleft\arraybackslash}m{0.5cm}
>{\raggedleft\arraybackslash}m{1.0cm}>{\raggedleft\arraybackslash}m{1.0cm}
>{\raggedleft\arraybackslash}m{1.0cm}>{\raggedleft\arraybackslash}m{1.0cm}
|>{\raggedleft\arraybackslash}m{1.0cm}>{\raggedleft\arraybackslash}m{1.0cm}
>{\raggedleft\arraybackslash}m{1.0cm}>{\raggedleft\arraybackslash}m{1.0cm}}
&& MF5 & MF & DC$_\text{neg}$ & SG-LP$_\ell$ & PROX-LP  &
 PROX-LP$_\ell$ & PROX-LP$_\text{acc}$ & Ave. E ($\times10^3$) & Ave. T (s) &
 Acc. & IoU \\
\hline
\parbox[t]{2mm}{\multirow{7}{*}{\rotatebox[origin=c]{90}{MSRC}}}
&MF5&-&0&0&0&0&0&0&8078.0&\textbf{0.2}&79.33&52.30\\
&MF&96&-&0&0&0&0&0&8062.4&0.5&79.35&52.32\\
&DC$_\text{neg}$&96&96&-&0&0&0&0&3539.6&1.3&83.01&57.92\\
&SG-LP$_\ell$&96&96&90&-&3&1&1&3335.6&13.6&83.15&58.09\\\cline{2-13}
&PROX-LP&96&96&94&92&-&13&45&1274.4&23.5&83.99&\textbf{59.66}\\
&PROX-LP$_\ell$&96&96&95&94&81&-&61&\textbf{1189.8}&6.3&83.94&59.50\\
&PROX-LP$_\text{acc}$&96&96&95&94&49&31&-&1340.0&3.7&\textbf{84.16}&59.65\\

\hline
\hline
\parbox[t]{1mm}{\multirow{7}{*}{\rotatebox[origin=c]{90}{Pascal}}}
&MF5&-&13&0&0&0&0&0&1220.8&0.8&79.13&27.53\\
&MF&2&-&0&0&0&0&0&1220.8&\textbf{0.7}&79.13&27.53\\
&DC$_\text{neg}$&99&99&-&-&0&0&0&629.5&3.7&80.43&28.60\\
&SG-LP$_\ell$&99&99&95&-&5&12&12&617.1&84.4&80.49&\textbf{28.68}\\\cline{2-13}
&PROX-LP&99&99&95&84&-&32&50&507.7&106.7&80.63&28.53\\
&PROX-LP$_\ell$&99&99&86&86&64&-&43&\textbf{502.1}&22.1&\textbf{80.65}&28.29\\
&PROX-LP$_\text{acc}$&99&99&86&86&46&39&-&507.7&14.7&80.58&28.45\\

\end{tabular}
\end{center}

\vspace{-0.5cm}
\caption{\em Results on the MSRC and Pascal datasets with the parameters tuned for
DC$_\text{neg}$. We show: the percentage of images where the row method
strictly outperforms the column one on the final integral energy, the average
integral energy over the test set, the average run time, the segmentation accuracy 
and the intersection over union score. Note
that all versions of our algorithm obtain much lower energies than the baselines.
Interestingly, while our fully accelerated version does slightly worse in terms of
energy, it is the best in terms of the segmentation accuracy in MSRC.}
\label{tab:dc}
\end{table*}

\subsection{Implementation details}\label{sec:implproxlp}
In practice, we initialize our algorithm with the solution of the best
continuous relaxation algorithm, which is called DC$_\text{neg}$
in~\cite{desmaison2016efficient}.
The parameters of our algorithm, such as the proximal regularization constant
$\lambda$ and the stopping criteria, are chosen manually. 
A small value of $\lambda$ leads to easier minimization of the proximal problem,
but also yields smaller steps at each proximal iteration.
We found
$\lambda = 0.1$ to work well in all our experiments. We fixed the maximum number
of proximal steps ($K$ in Algorithm~\ref{alg:proxlp}) to 10, and each proximal step
is optimized for a maximum of 5 Frank-Wolfe iterations ($T$ in
Algorithm~\ref{alg:proxlp}). 
%
In all our experiments the number of levels $H$ is fixed to $10$. 

\begin{figure*}
\def \SUBWIDTH{0.10\linewidth}
\begin{center}
\begin{subfigure}{\SUBWIDTH}
\includegraphics[width=0.99\linewidth]{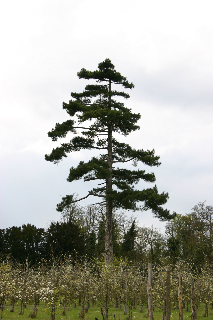}
\end{subfigure}%
\begin{subfigure}{\SUBWIDTH}
\includegraphics[width=0.99\linewidth]{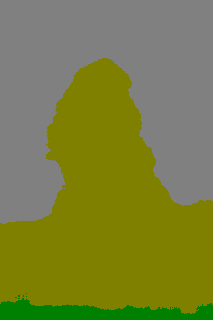}
\end{subfigure}%
\begin{subfigure}{\SUBWIDTH}
\includegraphics[width=0.99\linewidth]{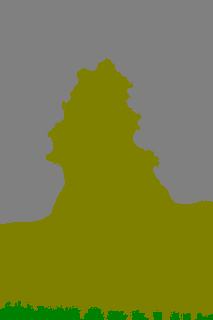}
\end{subfigure}%
\begin{subfigure}{\SUBWIDTH}
\includegraphics[width=0.99\linewidth]{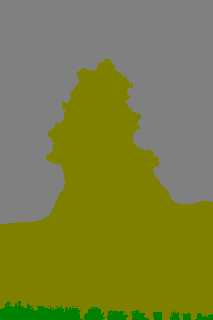}
\end{subfigure}%
\begin{subfigure}{\SUBWIDTH}
\mybox{red}{\includegraphics[width=0.99\linewidth]{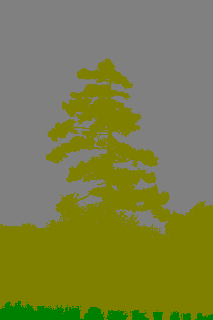}}
\end{subfigure}%
\begin{subfigure}{\SUBWIDTH}
\mybox{red}{\includegraphics[width=0.99\linewidth]{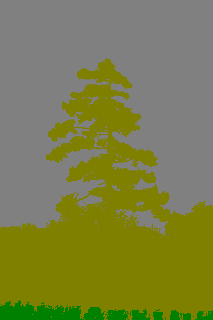}}
\end{subfigure}%
\begin{subfigure}{\SUBWIDTH}
\includegraphics[width=0.99\linewidth]{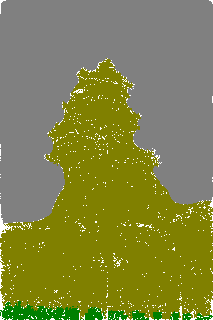}
\end{subfigure}%
\begin{subfigure}{\SUBWIDTH}
\mybox{red}{\includegraphics[width=0.99\linewidth]{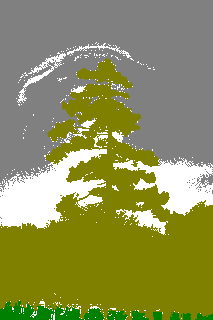}}
\end{subfigure}%
\begin{subfigure}{\SUBWIDTH}
\mybox{red}{\includegraphics[width=0.99\linewidth]{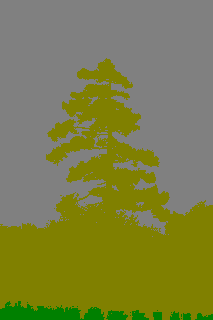}}
\end{subfigure}%
\begin{subfigure}{\SUBWIDTH}
\includegraphics[width=0.99\linewidth]{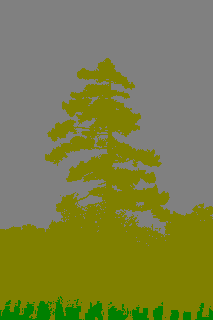}
\end{subfigure}

\begin{subfigure}{\SUBWIDTH}
\includegraphics[width=0.99\linewidth]{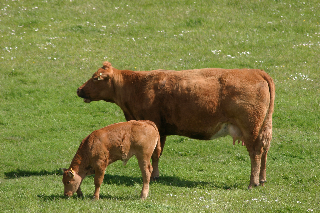}
\end{subfigure}%
\begin{subfigure}{\SUBWIDTH}
\includegraphics[width=0.99\linewidth]{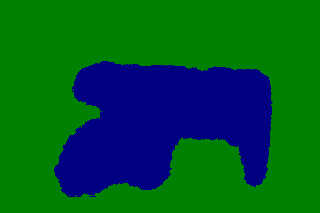}
\end{subfigure}%
\begin{subfigure}{\SUBWIDTH}
\includegraphics[width=0.99\linewidth]{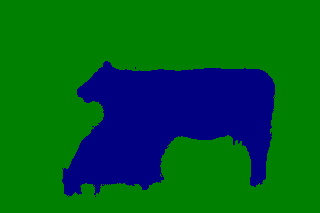}
\end{subfigure}%
\begin{subfigure}{\SUBWIDTH}
\includegraphics[width=0.99\linewidth]{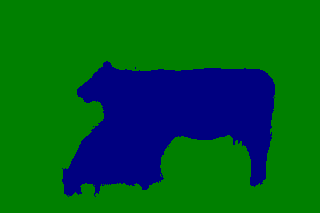}
\end{subfigure}%
\begin{subfigure}{\SUBWIDTH}
\mybox{red}{\includegraphics[width=0.99\linewidth]{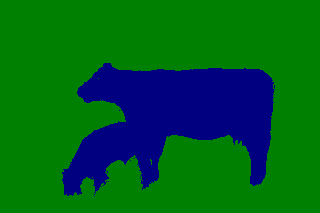}}
\end{subfigure}%
\begin{subfigure}{\SUBWIDTH}
\mybox{red}{\includegraphics[width=0.99\linewidth]{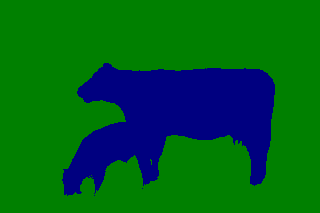}}
\end{subfigure}%
\begin{subfigure}{\SUBWIDTH}
\includegraphics[width=0.99\linewidth]{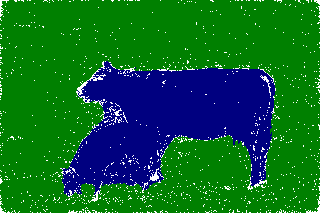}
\end{subfigure}%
\begin{subfigure}{\SUBWIDTH}
\mybox{red}{\includegraphics[width=0.99\linewidth]{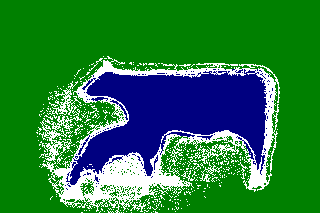}}
\end{subfigure}%
\begin{subfigure}{\SUBWIDTH}
\mybox{red}{\includegraphics[width=0.99\linewidth]{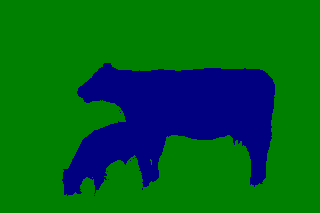}}
\end{subfigure}%
\begin{subfigure}{\SUBWIDTH}
\includegraphics[width=0.99\linewidth]{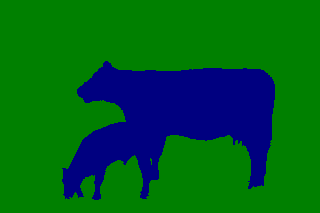}
\end{subfigure}

\begin{subfigure}{\SUBWIDTH}
\includegraphics[width=0.99\linewidth]{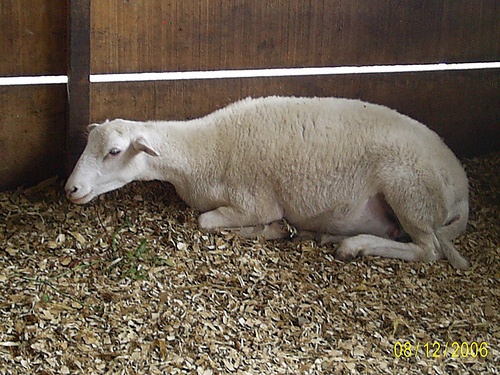}
\end{subfigure}%
\begin{subfigure}{\SUBWIDTH}
\includegraphics[width=0.99\linewidth]{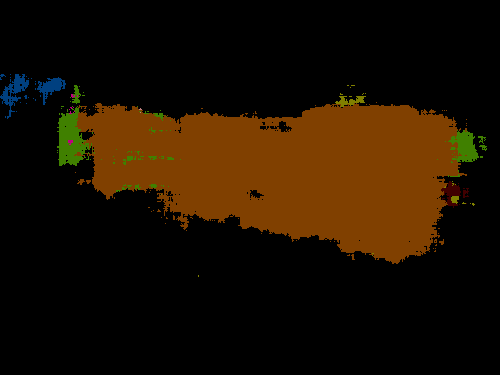}
\end{subfigure}%
\begin{subfigure}{\SUBWIDTH}
\includegraphics[width=0.99\linewidth]{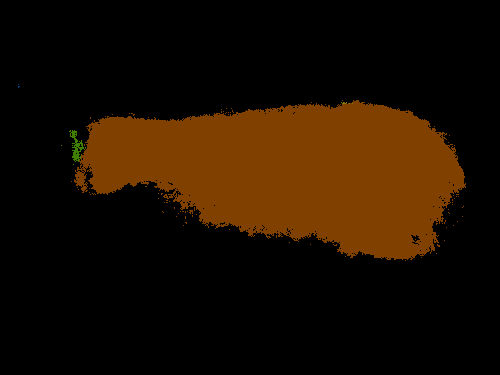}
\end{subfigure}%
\begin{subfigure}{\SUBWIDTH}
\includegraphics[width=0.99\linewidth]{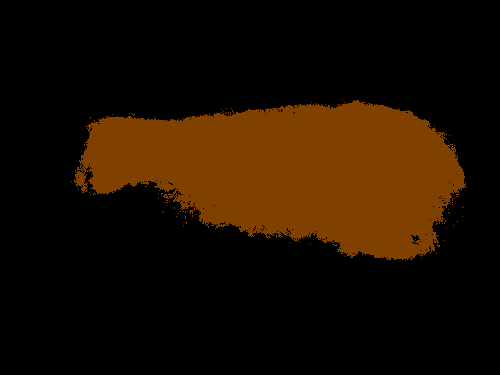}
\end{subfigure}%
\begin{subfigure}{\SUBWIDTH}
\mybox{red}{\includegraphics[width=0.99\linewidth]{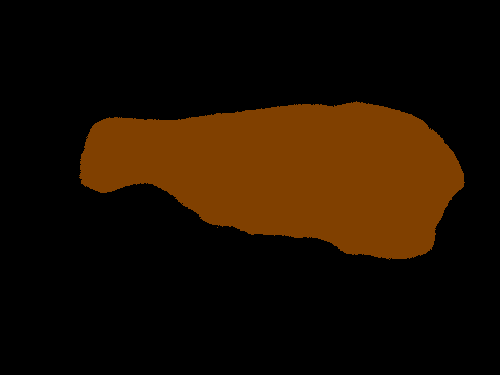}}
\end{subfigure}%
\begin{subfigure}{\SUBWIDTH}
\mybox{red}{\includegraphics[width=0.99\linewidth]{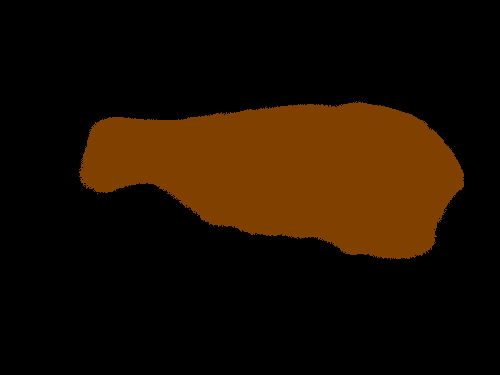}}
\end{subfigure}%
\begin{subfigure}{\SUBWIDTH}
\includegraphics[width=0.99\linewidth]{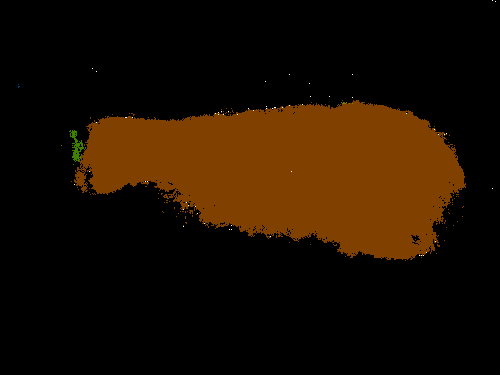}
\end{subfigure}%
\begin{subfigure}{\SUBWIDTH}
\mybox{red}{\includegraphics[width=0.99\linewidth]{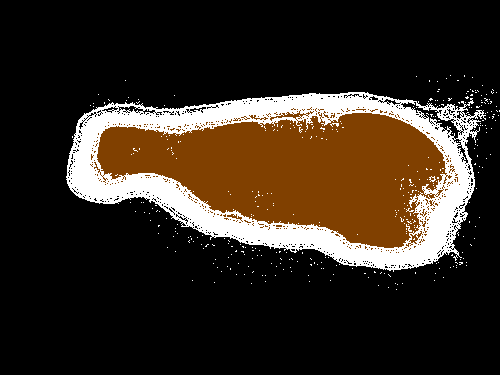}}
\end{subfigure}%
\begin{subfigure}{\SUBWIDTH}
\mybox{red}{\includegraphics[width=0.99\linewidth]{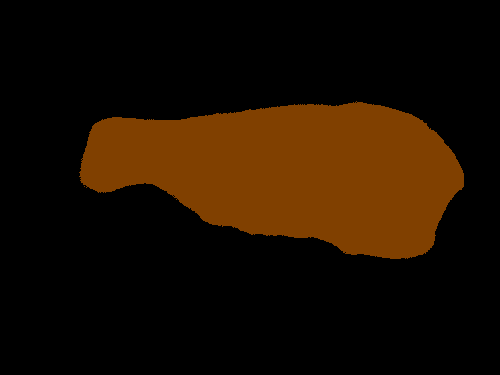}}
\end{subfigure}%
\begin{subfigure}{\SUBWIDTH}
\includegraphics[width=0.99\linewidth]{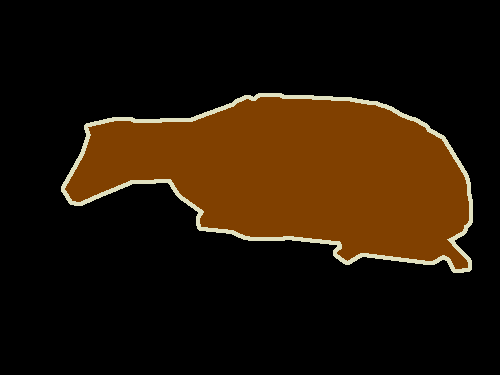}
\end{subfigure}

\begin{subfigure}{\SUBWIDTH}
\includegraphics[width=0.99\linewidth]{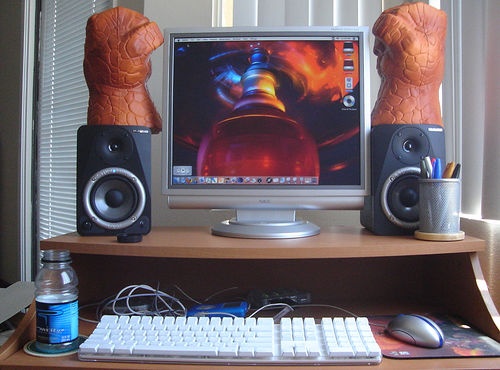}
\caption{Image}
\end{subfigure}%
\begin{subfigure}{\SUBWIDTH}
\includegraphics[width=0.99\linewidth]{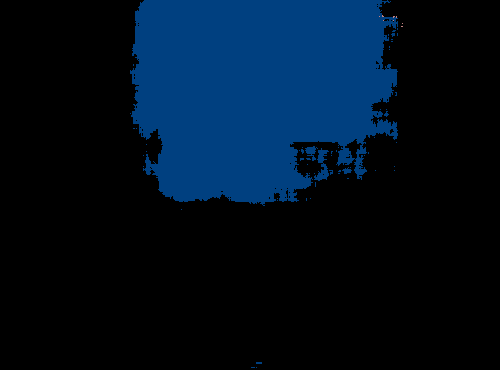}
\caption{MF}
\end{subfigure}%
\begin{subfigure}{\SUBWIDTH}
\includegraphics[width=0.99\linewidth]{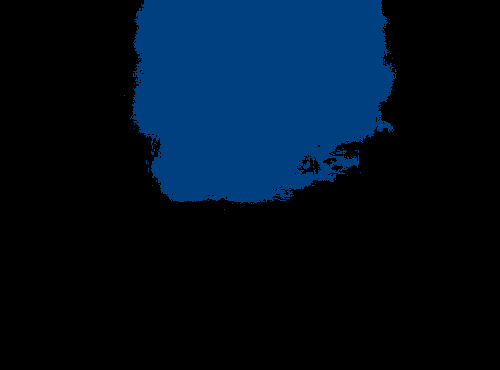}
\caption{DC$_\text{neg}$}
\end{subfigure}%
\begin{subfigure}{\SUBWIDTH}
\includegraphics[width=0.99\linewidth]{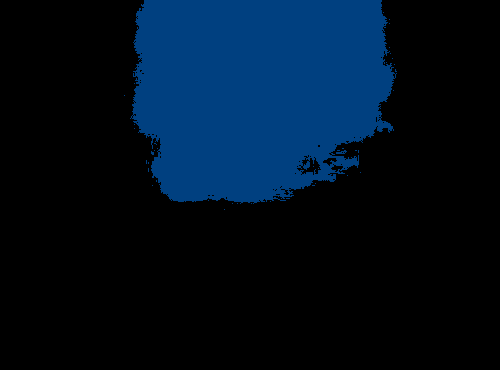}
\caption{SG-LP$_\ell$}
\end{subfigure}%
\begin{subfigure}{\SUBWIDTH}
\mybox{red}{\includegraphics[width=0.99\linewidth]{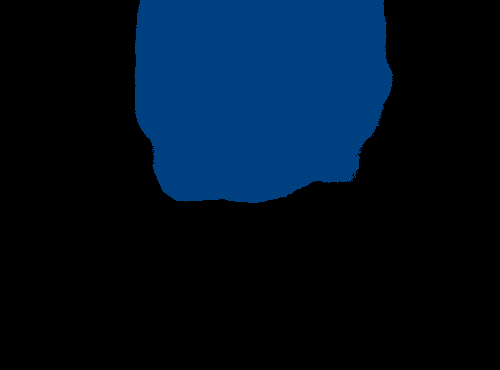}}
\caption{PROX-LP}
\end{subfigure}%
\begin{subfigure}{\SUBWIDTH}
\mybox{red}{\includegraphics[width=0.99\linewidth]{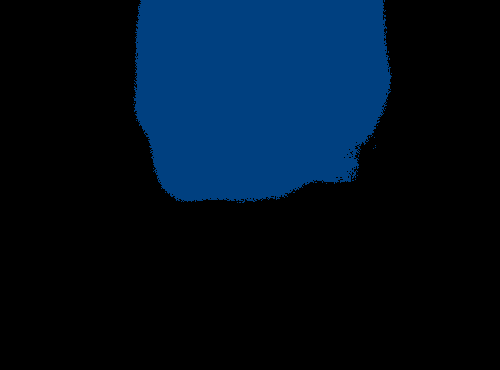}}
\caption{PROX-LP$_\ell$}
\end{subfigure}%
\begin{subfigure}{\SUBWIDTH}
\includegraphics[width=0.99\linewidth]{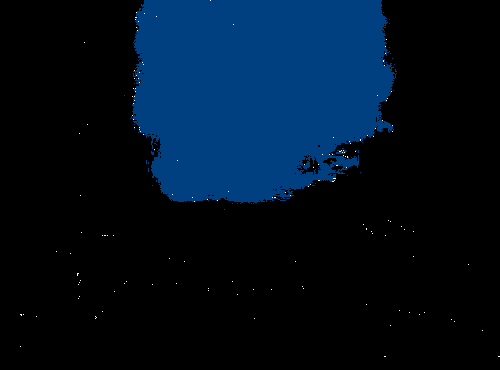}
\caption{Uncer.(DC$_\text{neg}$)}
\end{subfigure}%
\begin{subfigure}{\SUBWIDTH}
\mybox{red}{\includegraphics[width=0.99\linewidth]{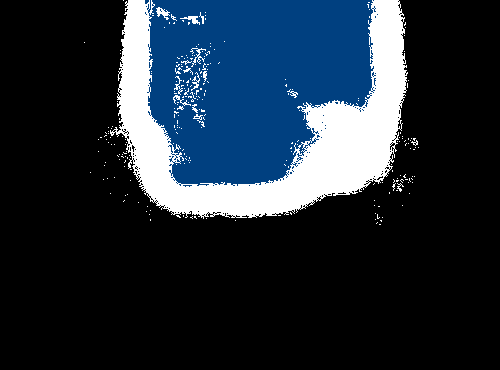}}
\caption{Uncer.(ours)}
\end{subfigure}%
\begin{subfigure}{\SUBWIDTH}
\mybox{red}{\includegraphics[width=0.99\linewidth]{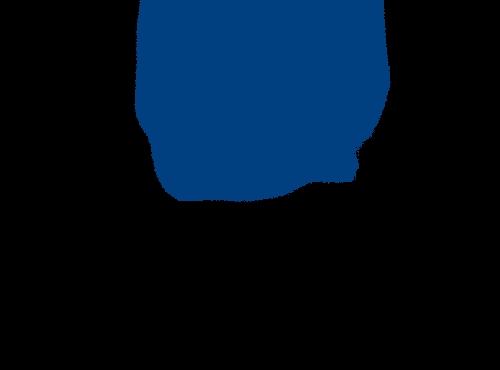}}
\caption{PROX-LP$_\text{acc}$}
\end{subfigure}%
\begin{subfigure}{\SUBWIDTH}
\includegraphics[width=0.99\linewidth]{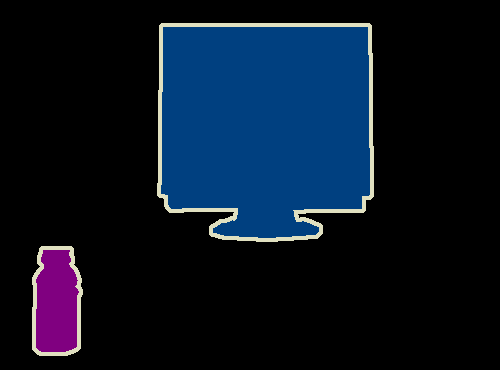}
\caption{Ground truth}
\end{subfigure}

\vspace{-0.2cm}
\caption{\em Results with DC$_\text{neg}$ parameters,
for an image in (\textbf{top}) MSRC and (\textbf{bottom}) Pascal. The
uncertain pixels identified by DC$_\text{neg}$ and PROX-LP$_\text{acc}$ are
marked in white. 
Note that, all versions of our algorithm obtain visually good segmentations.
In addition, even though DC$_\text{neg}$ is less accurate (the percentatge of
uncertain pixels for DC$_\text{neg}$ is usually less than 1\%) in predicting
uncertain pixels, our algorithm marks most of the crucial pixels (object boundaries and shadows)
 as uncertain. 
Furthermore, in the MSRC images, the improvement of PROX-LP$_\text{acc}$ over
the baselines is clearly visible and the final segmentation is virtually the
same as the accurate ground truth. (Best viewed in color)}
\label{fig:seg}
\end{center}
\vspace{-0.5cm}
\end{figure*}

\subsection{Segmentation results}
We evaluated our algorithm on the MSRC and Pascal VOC
2010~\cite{everingham2010pascal} segmentation datasets, and compare it against
mean-field inference (MF)~\cite{koltun2011efficient}, the best performing
continuous relaxation method of~\cite{desmaison2016efficient} (DC$_\text{neg}$)
and the subgradient based LP minimization method of~\cite{desmaison2016efficient}
(SG-LP). Note that, in~\cite{desmaison2016efficient}, the LP was initialized
with the DC$_\text{neg}$ solution and optimized for 5 iterations.
Furthermore, the LP optimization was performed on a subset of labels identified
by the DC$_\text{neg}$ solution in a similar manner to the one discussed in
Section~\ref{sec:proxlpacc}.
We refer to this algorithm as SG-LP$_\ell$.
For all the baselines, we employed the respective authors'
implementations that were obtained from the web or through personal
communication.
Furthermore, for all the algorithms, the integral
labelling is computed from the fractional solution using the \textit{argmax} rounding scheme.

For both datasets, we used the same splits and unary potentials as
in~\cite{koltun2011efficient}.
The pairwise potentials were defined using two kernels: a spatial kernel and a
bilateral one~\cite{koltun2011efficient}. For each method, the kernel parameters
were cross validated on validation data using
Spearmint~\cite{snoek2012practical}.
To be able to compare energy values, we then evaluated all methods with the same
parameters. In other words, for each dataset, each method was run several times
with different parameter values. The final parameter values for MF and
DC$_\text{neg}$ are given in Appendix~\ref{app:param}. Note that, on MSRC,
cross-validation was performed on the less accurate ground truth provided 
with the original dataset. Nevertheless, we evaluated all
methods on the accurate ground truth annotations provided by~\cite{koltun2011efficient}.

The results for the parameters tuned for DC$_\text{neg}$ on the MSRC and Pascal 
datasets are given in Table~\ref{tab:dc}.
Here MF5 denotes the mean-field algorithm run for 5 iterations. In
Fig.~\ref{fig:et}, we show the assignment energy as a function of time 
for an image in MSRC (the tree image in Fig.~\ref{fig:seg}) and for an
image in Pascal (the sheep image in Fig.~\ref{fig:seg}).
Furthermore, we provide some of the segmentation results in
Fig.~\ref{fig:seg}.

In summary, PROX-LP$_\ell$ obtains the lowest integral energy in both
datasets.
Furthermore, our fully accelerated version is the fastest LP minimization
algorithm and always outperforms the baselines by a great margin in terms of
energy.
From Fig.~\ref{fig:seg}, we can see that PROX-LP$_\text{acc}$ 
marks most of the crucial pixels (\eg, object boundaries) as
\textit{uncertain}, and optimizes over them efficiently and
effectively. 
Note that, on top of being fast, PROX-LP$_\text{acc}$ obtains the 
highest accuracy in MSRC for the parameters tuned for
DC$_\text{neg}$.

To ensure consistent behaviour across different energy
parameters, we ran the same experiments for the parameters tuned for MF. 
In this setting, all versions of our algorithm again yield significantly lower
energies than the baselines. The quantitative and qualitative results for this parameter
setting are given in Appendix~\ref{app:addexp}.

\subsection{Modified filtering method}
We then compare our modified filtering method, described in
Section~\ref{sec:ph}, with the divide-and-conquer strategy
of~\cite{desmaison2016efficient}. To this end, we evaluated both algorithms on
one of the Pascal VOC test images (the sheep image in Fig.~\ref{fig:seg}),
but varying the image size, the number of labels and the Gaussian kernel
standard deviation. Note that, to generate a plot for one variable, the other
variables are fixed to their respective standard values. The standard value for
the number of pixels is $187500$, for the number of labels $21$, and for the
standard deviation $1$.
For this experiment, the conditional gradients were computed from a random primal solution $\tbfy^t$.
In Fig.~\ref{fig:pasphpl}, we show the speedup of our
modified filtering approach over the one of~\cite{desmaison2016efficient} as a 
function of the number of pixels and labels. As shown in
Appendix~\ref{app:spph}, the speedup with respect to
 the kernel standard deviation is roughly constant.
The timings were averaged over 10 runs, and we observed only
negligible timing variations between the different runs.

\begin{figure}
\begin{center}
\begin{subfigure}{0.5\linewidth}
\includegraphics[width=0.95\linewidth]{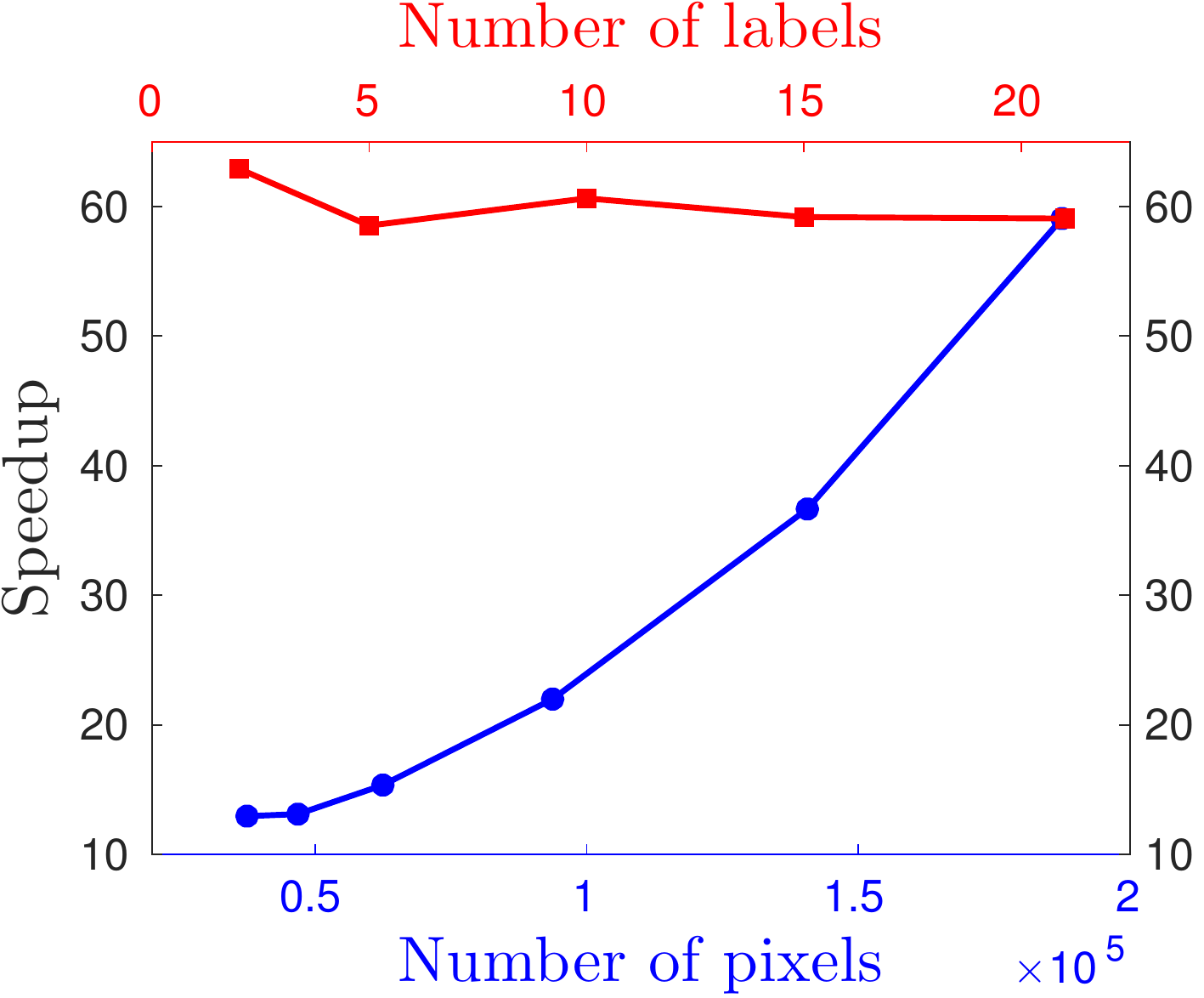}
\vspace{-0.1cm}
\caption{Spatial kernel ($d = 2$)}
\end{subfigure}%
\begin{subfigure}{0.5\linewidth}
\includegraphics[width=0.95\linewidth]{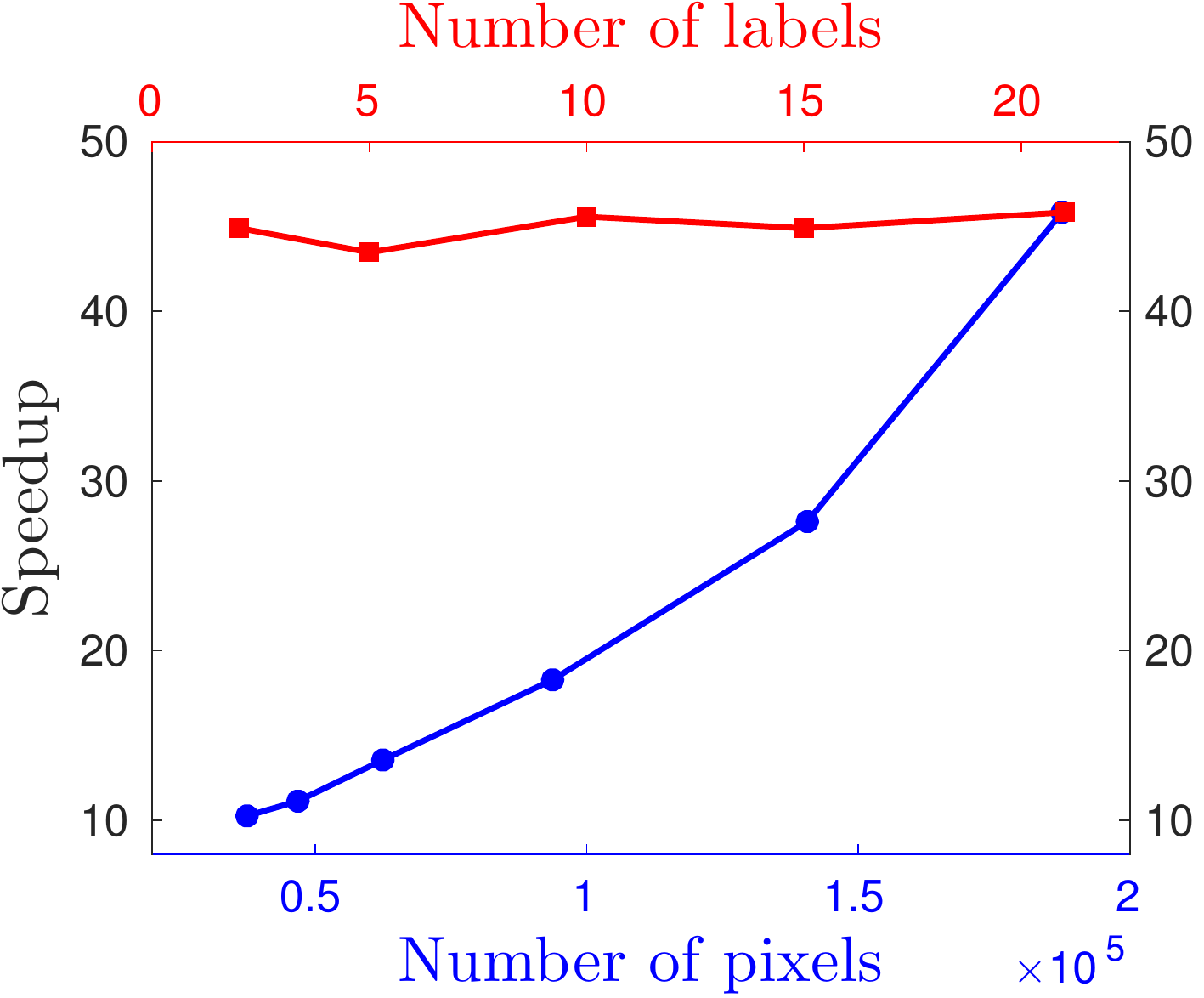}
\vspace{-0.1cm}
\caption{Bilateral kernel ($d=5$)}
\end{subfigure}

\vspace{-0.3cm}
\caption{\em Speedup of our modified filtering
algorithm over the divide-and-conquer strategy of~\cite{desmaison2016efficient} on a Pascal image. 
Note that our speedup grows with the number of pixels and is approximately constant with respect 
to the number of labels.
(Best viewed in color)}
\label{fig:pasphpl}
\end{center}
\vspace{-1.0cm}
\end{figure}

In summary, our modified filtering method is $10 - 65$ times faster than the
state-of-the-art algorithm of~\cite{desmaison2016efficient}. Furthermore,
note that all versions of our algorithm operate in
the region where the speedup is around $45-65$.

%% file: conclusion.tex
\section{Discussion}
We have introduced the first LP minimization algorithm for dense CRFs with
Gaussian pairwise potentials whose iterations are linear in the number of
pixels and labels. Thanks to the efficiency of our algorithm and to the tightness of
the LP relaxation, our approach yields much lower energy values than state-of-the-art
dense CRF inference methods.
Furthermore, our experiments demonstrated that, with the right set of energy
parameters, highly accurate segmentation results can be obtained with our
algorithm. The speed and effective energy
minimization of our algorithm make it a perfect 
candidate to be incorporated in an end-to-end learning framework, such
as~\cite{zheng2015conditional}. This, we believe, will be key to further
improving the accuracy of deep semantic segmentation architectures.

\section{Acknowledgements}
This work was supported by the EPSRC, the ERC grant ERC- 2012-AdG 321162-HELIOS,
the EPSRC/MURI grant ref EP/N019474/1, the EPSRC grant EP/M013774/1, the EPSRC
programme grant Seebibyte EP/M013774/1, the Microsoft Research PhD Scholarship,
ANU PhD Scholarship and Data61 Scholarship. Data61 (formerly NICTA) is funded by
the Australian Government as represented by the Department of Broadband,
Communications and the Digital Economy and the Australian Research Council
through the ICT Centre of Excellence program.

%% file: appendix.tex
\appendix

\section{Proximal minimization for LP relaxation - Supplementary material}
In this section, we give the detailed derivation of our proximal minimization
algorithm for the LP relaxation.

\subsection{Dual formulation}\label{app:dual}
Let us first restate the definition of the matrices $A$ and $B$, and analyze
their properties. This would be useful to derive the dual of the proximal
problem~\eqref{eqn:p}.

\begin{dfn}
Let $A \in \R^{nm\times p}$ and $B\in\R^{nm\times n}$ be two matrices such that
\begin{align}
\label{eqn:ab}
\left(A\bfalpha\right)_{a:i} &= -\sum_{b\ne a}\left(\alpha^1_{ab:i} -
\alpha^2_{ab:i} + \alpha^2_{ba:i} - \alpha^1_{ba:i}\right)\ ,\\\nonumber
\left(B\bfbeta\right)_{a:i} &= \beta_a\ .
\end{align}
\end{dfn}

\begin{pro}
  Let $\bfx \in \R^{nm}$. Then, for all $a \ne b$ and $i \in \calL$,
  \label{pro:a}
   \begin{align}
 	\left(A^T\bfx\right)_{{ab:i}^1} &= x_{b:i} - x_{a:i}\ ,\\\nonumber
 	\left(A^T\bfx\right)_{{ab:i}^2} &= x_{a:i} - x_{b:i}\ .
   \end{align}
Here, the index ${{ab:i}^1}$ denotes the element corresponding to
$\alpha^1_{ab:i}$.
\end{pro}
\begin{proof}
This can be easily proved by inspecting the matrix $A$.
\end{proof}

\begin{pro}
The matrix $B\in \R^{nm\times n}$ defined in Eq.~\eqref{eqn:ab} satisfies the
following properties:
\label{pro:b}
\begin{enumerate}
  \item Let $\bfx \in \R^{nm}$. Then, $\left(B^T\bfx\right)_a =
  \sum_{i\in\calL}x_{a:i}$ for all $a\in \allpixels$.
  \item $B^TB = mI$, where $I\in \R^{n \times n}$ is the identity matrix.
  \item $BB^T$ is a block diagonal matrix, with each block $\left(BB^T\right)_a
  = \bfone$ for all $a\in \allpixels$, where $\bfone \in \R^{m \times m}$ is the matrix of all
  ones.
\end{enumerate}
\end{pro}
\begin{proof}
Note that, from Eq.~\eqref{eqn:ab}, the matrix $B$ simply
repeats the elements $\beta_a$ for $m$ times. In particular, for
$m=3$, the matrix $B$ has the following form:
\begin{equation}
 B = 
\begin{bmatrix}
1  & 0  &\cdots &  \cdots & \cdots & 0 \\[-0.2cm]
1  & \vdots  & & & & \vdots \\[-0.2cm]
1 & 0  &\cdots &  \cdots & \cdots& \vdots \\[-0.2cm]
0 & 1 &  &  &    & \vdots \\[-0.2cm]
\vdots & 1&  &    & & \vdots\\[-0.2cm]
\vdots  & 1 &    &  &    & \vdots\\[-0.2cm]
\vdots  & 0 &  &   &   &  \vdots\\[-0.2cm]
\vdots  & \vdots  &  & &   &  0\\[-0.2cm]
\vdots  & \vdots  &  & &   &  1\\[-0.2cm]
\vdots  & \vdots  &  & &   &  1\\
0 & \cdots&  \cdots &  \cdots & 0 & 1
\end{bmatrix}\ .
\end{equation}
Therefore, multiplication by $B^T$ amounts to summing over the labels. From this, the
other properties can be proved easily. 
\end{proof}

We now derive the Lagrange dual of~\eqref{eqn:p}.

\begin{pro}
Given matrices $A \in\R^{nm\times p}$ and $B\in\R^{nm\times n}$ and dual
variables $(\bfalpha,\bfbeta,\bfgamma)$.
\begin{enumerate}
  \item The Lagrange dual of~\eqref{eqn:p} takes the following form:
	\begin{alignat}{3}
	\label{eqn:dual0}
	&\underset{\bfalpha,\bfbeta,\bfgamma}{\operatorname{min}}\ g(\bfalpha,
	\bfbeta, \bfgamma) &&= \frac{\lambda}{2}\|A\bfalpha +
	B\bfbeta+\bfgamma-\bfphi\|^2 + \left\langle A\bfalpha +
	B\bfbeta+\bfgamma-\bfphi, \bfy^k \right\rangle - \langle \bfone, \bfbeta
	\rangle\ ,\\\nonumber &\text{s.t.}\hskip0.08\linewidth \gamma_{a:i} &&\ge
	0\quad\forall\,a\in\allpixels\quad \forall\,i \in \calL\ , \\\nonumber
	&\hskip0.08\linewidth\bfalpha \in \calC &&= \left\{ \begin{array}{l|l}
	\multirow{2}{*}{$\bfalpha$} & \alpha^1_{ab:i} + \alpha^2_{ab:i} =
	\frac{K_{ab}}{2},\, \forall\,a\ne b,\, \forall\,i \in \calL\\
 	& \alpha^1_{ab:i}, \alpha^2_{ab:i} \ge 0,\,\forall\,a\ne
 	b,\, \forall\,i \in \calL \end{array} \right\}\ .
	\end{alignat}
  \item The primal variables $\bfy$ satisfy
	\begin{equation}
	\label{eqn:cp0}
	\bfy = \lambda \left(A\bfalpha+B\bfbeta+\bfgamma-\bfphi\right)+\bfy^k\ .
	\end{equation}
\end{enumerate}
\end{pro}
\begin{proof}
The Lagrangian associated with the primal
problem~\eqref{eqn:p} can be written as~\cite{boyd2009convex}:
\begin{alignat}{3}
\label{eqn:pd}
&\underset{\bfalpha,\bfbeta,\bfgamma}{\operatorname{max}}\ 
\underset{\bfy,\bfz}{\operatorname{min}}\ L(\bfalpha, \bfbeta, \bfgamma, \bfy,
\bfz) &&= \sum_a\sum_i \phi_{a:i}\,y_{a:i} + \sum_{a,b\ne a} \sum_i
\frac{K_{ab}}{2}z_{ab:i} + 
\frac{1}{2\lambda}\sum_a\sum_i\left(y_{a:i} - y^k_{a:i}\right)^2 \\\nonumber 
&&&+ \sum_{a,b\ne a}
\sum_i \alpha^1_{ab:i}\left(y_{a:i} - y_{b:i} - z_{ab:i}\right)
 + \sum_{a,b\ne a} \sum_i \alpha^2_{ab:i}\left(y_{b:i} - y_{a:i} - z_{ab:i}\right)
 \\\nonumber &&&+ \sum_a
\beta_a\left(1 - \sum_i y_{a:i}\right) - \sum_a\sum_i\gamma_{a:i}\,y_{a:i}\
,\\\nonumber
&\text{s.t.}\hskip0.1\linewidth \alpha^1_{ab:i}, \alpha^2_{ab:i} &&\ge 0\quad
\forall\,a\ne b\quad \forall\,i\in \calL\ ,\\\nonumber
&\hskip0.18\linewidth\gamma_{a:i} &&\ge 0\quad
\forall\,a\in\allpixels\quad\forall\,i\in \calL\ .
\end{alignat}
Note that the dual problem is obtained by minimizing the Lagrangian over the
primal variables $(\bfy, \bfz)$.
With respect to $\bfz$, the Lagrangian is linear and when
$\nabla_{\bfz}L(\bfalpha, \bfbeta, \bfgamma, \bfy,\bfz) \ne 0$, the minimization
in $\bfz$ yields $-\infty$. This situation is not useful as the dual function is
unbounded. Therefore we restrict ourselves to the case where
$\nabla_{\bfz}L(\bfalpha, \bfbeta, \bfgamma, \bfy,\bfz) = 0$.
By differentiating with respect to $\bfz$ and setting the derivatives to zero,
we obtain
\begin{equation}
\label{eqn:dz}
\alpha^1_{ab:i} + \alpha^2_{ab:i} = \frac{K_{ab}}{2} \quad
\forall\,a\ne b\quad \forall\,i\in \calL\ .
\end{equation}
The minimum of the Lagrangian with respect to $\bfy$ is attained when $\nabla_{\bfy}L(\bfalpha, \bfbeta, \bfgamma, \bfy,\bfz) = 0$.
Before differentiating with respect to $\bfy$, let us rewrite the
Lagrangian using Eq.~\eqref{eqn:dz} and reorder the terms:
\begin{align}
L(\bfalpha, \bfbeta, \bfgamma, \bfy, \bfz) &= \sum_a\sum_i (\phi_{a:i}
-\beta_a - \gamma_{a:i})\,y_{a:i}  + 
\frac{1}{2\lambda}\sum_a\sum_i\left(y_{a:i} - y^k_{a:i}\right)^2 
+ \sum_{a,b\ne a}\sum_i \left(\alpha^1_{ab:i} - \alpha^2_{ab:i}\right)\,y_{a:i}
\\\nonumber &+ \sum_{a,b\ne a}\sum_i \left(\alpha^2_{ba:i} -
\alpha^1_{ba:i}\right)\,y_{a:i}
+ \sum_a \beta_a\ .
\end{align}
Now, by differentiating with respect to $\bfy$ and setting the derivatives to
zero, we get
\begin{align}
\frac{1}{\lambda}\left(y_{a:i}-y^k_{a:i}\right) &= -\sum_{b\ne a}\left(\alpha^1_{ab:i}-
\alpha^2_{ab:i} + \alpha^2_{ba:i} - \alpha^1_{ba:i}\right) + \beta_a +
\gamma_{a:i} - \phi_{a:i}\quad \forall\,a\in \allpixels\quad \forall\,i\in\calL\
.
\end{align}
Using Eq.~\eqref{eqn:ab}, the above equation can be written in vector form as
\begin{equation}
\label{eqn:dy}
 \frac{1}{\lambda}\left(\bfy-\bfy^k\right) = A\bfalpha +
B\bfbeta+\bfgamma-\bfphi\ .
\end{equation}
This proves Eq.~\eqref{eqn:cp0}.
Now, using Eqs.~\eqref{eqn:dz} and~\eqref{eqn:dy}, the dual
problem can be written as
\begin{alignat}{3}
&\underset{\bfalpha,\bfbeta,\bfgamma}{\operatorname{min}}\ g(\bfalpha,
\bfbeta, \bfgamma) &&= \frac{\lambda}{2}\|A\bfalpha +
B\bfbeta+\bfgamma-\bfphi\|^2 + \left\langle A\bfalpha + B\bfbeta+\bfgamma-\bfphi, \bfy^k \right\rangle - \langle \bfone, \bfbeta
\rangle\ ,\\\nonumber
&\text{s.t.}\hskip0.08\linewidth \gamma_{a:i} &&\ge
0\quad\forall\,a\in\allpixels\quad \forall\,i \in \calL\ , \\\nonumber
&\hskip0.08\linewidth\bfalpha \in \calC &&= \left\{ \begin{array}{l|l}
\multirow{2}{*}{$\bfalpha$} & \alpha^1_{ab:i} + \alpha^2_{ab:i} = \frac{K_{ab}}{2},\,
\forall\,a\ne b,\, \forall\,i \in \calL\\
 & \alpha^1_{ab:i}, \alpha^2_{ab:i} \ge 0,\,\forall\,a\ne
 b,\, \forall\,i \in \calL \end{array} \right\}\ .
\end{alignat}
Here, $\bfone$ denotes the vector of all ones of appropriate dimension.
Note that we converted our problem to a minimization one by changing the sign of all the terms.
This proves Eq.~\eqref{eqn:dual0}.
\end{proof}

\subsection{Optimizing over $\bfbeta$ and $\bfgamma$}\label{app:gamma}
In this section, for a fixed value of $\bfalpha^t$, we optimize over $\bfbeta$
and $\bfgamma$. 
\begin{pro}
If $\nabla_{\bfbeta}g(\bfalpha^t, \bfbeta,\bfgamma)=0$, then 
$\bfbeta$ satisfy
\begin{equation}
\bfbeta = B^T\left(A\bfalpha^t + \bfgamma - \bfphi\right)/m\ .
\label{eqn:beta0}
\end{equation}
\end{pro}
\begin{proof}
By differentiating the dual objective $g$ with respect to $\bfbeta$ and setting
the derivatives to zero, we obtain the above equation. Note that, from Proposition~\ref{pro:b},
$B^T\bfy^k = \bfone$ since $\bfy^k \in \calM$ (defined in Eq.~\eqref{eqn:lp}),
and $B^TB = mI$. Both these identities are used to simplify the above equation.
\end{proof}

Let us now define a matrix $D$ and analyze its properties.
This will be useful to simplify the optimization over $\bfgamma$.

\begin{dfn}
Let $D\in\R^{nm\times nm}$ be a matrix that satisfy 
\begin{equation}
D=I-\frac{BB^T}{m}\ ,
\end{equation}
where $B$ is defined in Eq.~\eqref{eqn:ab}.
\end{dfn}

\begin{pro}
\label{pro:d}
The matrix $D$ satisfies the following properties:
\begin{enumerate}
  \item $D$ is block diagonal, with each block matrix $D_a = I - \bfone/m$, where
  $I\in\R^{m\times m}$ is the identity matrix and $\bfone\in\R^{m\times m}$ is
  the matrix of all ones.
  \item $D^TD = D\ .$
\end{enumerate}
\end{pro}
\begin{proof}
From Proposition~\ref{pro:b}, the matrix $BB^T$ is block
diagonal with each block $\left(BB^T\right)_a = \bfone$.
Therefore $D$ is block diagonal with each block
matrix $D_a = I - \bfone/m$. Note that the block matrices $D_a$ are identical. 
The second property can be proved using simple matrix algebra.
\end{proof}

Now we turn to the optimization over $\bfgamma$. 

\begin{pro}
The optimization over $\bfgamma$ decomposes over pixels where for a pixel $a$,
this QP has the form
\begin{alignat}{3}
\label{eqn:qpgammas0}
&\underset{\bfgamma_a}{\operatorname{\min}}\quad 
&&\frac{1}{2}\bfgamma^T_aQ\bfgamma_a + \left\langle \bfgamma_a, 
Q\left((A\bfalpha^t)_a-\bfphi_a\right) + \bfy^k_a \right\rangle \
,\\\nonumber 
&\text{s.t.}\quad &&\bfgamma_a\ge \bfzero\ .
\end{alignat}
Here, $\bfgamma_a$ denotes the vector $\{\gamma_{a:i}\mid
i\in\calL\}$ and $Q= \lambda\left(I-\bfone/m\right)\in\R^{m\times m}$, with $I$
the identity matrix and $\bfone$ the matrix of all ones.
\end{pro}
\begin{proof}
By substituting $\bfbeta$ in the dual
problem~(\ref{eqn:dual0}) with Eq.~\eqref{eqn:beta0}, the optimization problem
over $\bfgamma$ takes the following form:
\begin{alignat}{3}
&\underset{\bfgamma}{\operatorname{min}}\  g(\bfalpha^t, \bfgamma) &&=
\frac{\lambda}{2}\|D(A\bfalpha^t +\bfgamma-\bfphi)\|^2 + \left\langle D(A\bfalpha^t + \bfgamma-\bfphi), \bfy^k
\right\rangle + \frac{1}{m} \left\langle \bfone, A\bfalpha^t + \bfgamma-\bfphi
\right\rangle\ ,\\\nonumber
&\text{s.t.}\hskip0.07\linewidth \bfgamma &&\ge \bfzero \ ,
\end{alignat}
where $D=I-\frac{BB^T}{m}$.

Note that, since $\bfy^k \in \calM$, from
Proposition~\ref{pro:b}, $B^T\bfy^k = \bfone$. Using this fact, the identity
$D^TD = D$, and by removing the constant terms, the optimization problem over $\bfgamma$ can be simplified:
\begin{alignat}{3}
\label{eqn:gamma1}
&\underset{\bfgamma}{\operatorname{min}}\ g(\bfalpha^t, \bfgamma) &&=
\frac{\lambda}{2}\bfgamma^TD\bfgamma + \langle \bfgamma, \lambda D(A\bfalpha^t-\bfphi) + \bfy^k \rangle \
,\\\nonumber 
&\text{s.t.}\hskip0.07\linewidth \bfgamma &&\ge \bfzero\ .
\end{alignat}
Furthermore, since $D$ is block diagonal from Proposition~\ref{pro:d}, we obtain
\begin{equation}
\label{eqn:bgamma}
\underset{\bfgamma\ge \bfzero}{\operatorname{min}}\  g(\bfalpha^t, \bfgamma)
= \sum_a\underset{\bfgamma_a\ge \bfzero}{\operatorname{min}} \ 
\frac{\lambda}{2}\bfgamma^T_aD_a\bfgamma_a + \langle \bfgamma_a, \lambda
D_a\left((A\bfalpha^t)_a-\bfphi_a\right) + \bfy^k_a \rangle \ ,
\end{equation}
where the notation $\bfgamma_a$ denotes the vector $\{\gamma_{a:i}\mid
i\in\calL\}$. By substituting $Q = \lambda\, D_a$, the QP associated with each
pixel $a$ can be written as
\begin{alignat}{3}
&\underset{\bfgamma_a\ge \bfzero}{\operatorname{\min}}\quad 
&&\frac{1}{2}\bfgamma^T_aQ\bfgamma_a + \left\langle \bfgamma_a, 
Q\left((A\bfalpha^t)_a-\bfphi_a\right) + \bfy^k_a \right\rangle \ .
\end{alignat}
\end{proof}

Each of these $m$ dimensional quadratic programs (QP) are
optimized using the iterative algorithm of~\cite{xiao2014multiplicative}. Before
we give the update equation, let us first write our problem in the form used
in~\cite{xiao2014multiplicative}. For a given $a\in\allpixels$, this yields
\begin{equation}
\underset{\bfgamma_a\ge \bfzero}{\operatorname{min}}\
\frac{1}{2}\bfgamma_a^TQ\bfgamma_a - \left\langle \bfgamma_a, \bfh_a
\right\rangle\ ,
\end{equation}
where 
\begin{align}
Q &= \lambda \left(I - \frac{\bfone}{m}\right)\ ,\\\nonumber
\bfh_a &= -Q\left((A\bfalpha^t)_a-\bfphi_a\right) - \bfy^k_a\ .
\end{align}
Hence, at
each iteration, the element-wise update equation has the following form:
\begin{equation}
\gamma_{a:i} =
\gamma_{a:i}\left[\frac{2\,(Q^-\bfgamma_a)_i + h^+_{a:i} +
c}{(|Q|\bfgamma_a)_i + h^-_{a:i} + c}\right]\ ,
\end{equation}
where $Q^- = \max(-Q, 0)$, $|Q| = \text{abs}(Q)$, $h^+_{a:i} = \max(h_{a:i}, 0)$
and $h^-_{a:i} = \max(-h_{a:i}, 0)$ and $0 < c \ll 1$. These $\max$ and
$\text{abs}$ operations are element-wise. We refer the interested reader
to~\cite{xiao2014multiplicative} for more detail on this update rule.

Note that, even though the matrix $Q$ has $m^2$ elements, the multiplication by
$Q$ can be performed in $\calO(m)$. In
particular, the multiplication by $Q$ can be decoupled into a multiplication by the
identity matrix and a matrix of all ones, both of which can be performed in
linear time. Similar observations can be made for the matrices $Q^-$ and $|Q|$.
Hence, the time complexity of the above update is $\calO(m)$.
Once the optimal $\bfgamma$ for a given $\bfalpha^t$ is computed, the
corresponding optimal $\bfbeta$ is given by Eq.~(\ref{eqn:beta0}).

\subsection{Conditional gradient computation}\label{app:cgrad}
\begin{pro}
The conditional gradient $\bfs$ satisfy
\begin{equation}
\label{eqn:cgrad0}
\left(A\bfs\right)_{a:i} = -\sum_b \left(K_{ab} \I[\ty^t_{a:i} \ge
\ty^t_{b:i}] - K_{ab} \I[\ty^t_{a:i} \le \ty^t_{b:i}]\right)\ ,
\end{equation}
where $\tbfy^t = \lambda\left(A\bfalpha^t+B\bfbeta^t+\bfgamma^t-\bfphi\right) +
\bfy^k$ using Eq.~\eqref{eqn:cp0}.
\end{pro}
\begin{proof}
The conditional gradient with respect to $\bfalpha$ is obtained by solving the
following linearization problem:
\begin{equation}
\bfs = \underset{\hbfs\in\calC} {\operatorname{argmin}}\,
 \left\langle \hbfs,\nabla_{\bfalpha} g(\bfalpha^t, \bfbeta^t, \bfgamma^t)
 \right\rangle\ ,
\end{equation}
where 
\begin{equation}
\nabla_{\bfalpha} g(\bfalpha^t, \bfbeta^t,\bfgamma^t) = A^T\tbfy^t\ ,
\end{equation}
with $\tbfy^t = \lambda\left(A\bfalpha^t+B\bfbeta^t+\bfgamma^t-\bfphi\right) +
\bfy^k$ using Eq.~(\ref{eqn:cp0}). 

Note that the feasible set
$\calC$ is separable, \ie, it can be written as 
$\calC = \prod_{a,\, b\ne a,\,i \in \calL}\, \calC_{ab:i}$, with\\
\mbox{$\calC_{ab:i} = \left\{(\alpha^1_{ab:i},
\alpha^2_{ab:i})\mid \alpha^1_{ab:i} + \alpha^2_{ab:i} = K_{ab}/2,
\alpha^1_{ab:i}, \alpha^2_{ab:i} \ge 0\right\}$}.
Therefore, the conditional gradient can be computed separately, corresponding to
each set $\calC_{ab:i}$. This yields
\begin{alignat}{3}
\label{eqn:s0}
&\underset{\hs^1_{ab:i}, \hs^2_{ab:i}}{\operatorname{\min}}\quad
&&\hs^1_{ab:i}\nabla_{\alpha^1_{ab:i}}g(\bfalpha^t, \bfbeta^t,\bfgamma^t) +
\hs^2_{ab:i}\nabla_{\alpha^2_{ab:i}}g(\bfalpha^t, \bfbeta^t,\bfgamma^t)\ ,\\\nonumber 
&\quad\text{s.t.}\quad &&\hs^1_{ab:i} + \hs^2_{ab:i} = K_{ab}/2\ ,\\\nonumber
&&&\hs^1_{ab:i}, \hs^2_{ab:i} \ge 0\ ,
\end{alignat}
where, using Proposition~\ref{pro:a}, the gradients can be written as:
\begin{align}
\nabla_{\alpha^1_{ab:i}}g(\bfalpha^t, \bfbeta^t,\bfgamma^t) &= \ty^t_{b:i} -
\ty^t_{a:i}\ ,\\\nonumber
\nabla_{\alpha^2_{ab:i}}g(\bfalpha^t, \bfbeta^t,\bfgamma^t) &= \ty^t_{a:i} -
\ty^t_{b:i}\ .
\end{align}
Hence, the minimum is attained at:
\begin{align}
s^1_{ab:i} &= \left\{\begin{array}{ll}K_{ab}/2 & \mbox{if $\ty^t_{a:i}\ge
\ty^t_{b:i}$}\\ 0 & \mbox{otherwise}\ ,\end{array}\right.\\\nonumber
s^2_{ab:i} &= \left\{\begin{array}{ll}K_{ab}/2 & \mbox{if $\ty^t_{a:i}\le
\ty^t_{b:i}$}\\ 0 & \mbox{otherwise}\ .\end{array}\right.
\end{align}
Now, from Eq.~\eqref{eqn:ab}, $A\bfs$ takes the following form:
\begin{align}
\left(A\bfs\right)_{a:i} &= -\sum_{b\ne a}\left(\frac{K_{ab}}{2}\I[\ty^t_{a:i}\ge
\ty^t_{b:i}] - \frac{K_{ab}}{2}\I[\ty^t_{a:i}\le \ty^t_{b:i}] +
\frac{K_{ba}}{2}\I[\ty^t_{b:i}\le \ty^t_{a:i}] -
\frac{K_{ba}}{2}\I[\ty^t_{b:i}\ge \ty^t_{a:i}]\right)\ ,\\\nonumber &= -\sum_b
\left(K_{ab} \I[\ty^t_{a:i} \ge \ty^t_{b:i}] - K_{ab} \I[\ty^t_{a:i} \le
\ty^t_{b:i}]\right)\ .
\end{align}
Here, we used the symmetry of the kernel matrix $K$ to obtain this result. Note
that the second equation is a summation over $b\in \allpixels$. This is true due
to the identity $K_{aa}\I[\ty^t_{a:i}\ge \ty^t_{a:i}] - K_{aa}\I[\ty^t_{a:i}\le
\ty^t_{a:i}] = 0$ when $b = a$.
\end{proof}

\subsection{Optimal step size}\label{app:step}
\begin{pro}
The optimal step size $\delta$ satisfy
\begin{equation}
\delta = P_{[0,1]}\left(\frac{\langle A\bfalpha^t - A\bfs^t,
\tbfy^t\rangle}{\lambda\|A\bfalpha^t - A\bfs^t\|^2}\right)\ .
\end{equation}
Here, $P_{[0,1]}$ denotes the projection to the interval $[0,1]$, that is,
clipping the value to lie in $[0,1]$.
\end{pro}
\begin{proof}
The optimal step size $\delta$ gives the maximum
decrease in the objective function $g$ given the descent direction $\bfs^t$. 
This can be formulated as the following optimization problem:
\begin{align}
\underset{\delta} {\operatorname{min}}\quad &
\frac{\lambda}{2}\left\|A\bfalpha^t + \delta\left(A\bfs^t-A\bfalpha^t\right)+
B\bfbeta^t +\bfgamma^t-\bfphi\right\|^2 + \left\langle A\bfalpha^t +
\delta\left(A\bfs^t-A\bfalpha^t\right) +B\bfbeta^t + \bfgamma^t-\bfphi), \bfy^k
\right\rangle - \left\langle \bfone, \bfbeta \right\rangle\
,\\\nonumber \text{s.t.}\quad \delta &\in[0,1] \ .
\end{align}
Note that the above function is optimized over the scalar variable $\delta$ and
the minimum is attained when the derivative is zero. Hence, setting the
derivative to zero, we have
\begin{align}
0 &= \lambda\left\langle \delta\left(A\bfs^t - A\bfalpha^t\right) + A\bfalpha^t +
B\bfbeta^t + \bfgamma^t-\bfphi, A\bfs^t - A\bfalpha^t \right\rangle + \left\langle
\bfy^k, A\bfs^t - A\bfalpha^t\right\rangle\ ,\\\nonumber 
\delta &= \frac{\langle A\bfalpha^t - A\bfs^t, \lambda\left(A\bfalpha^t +
B\bfbeta^t + \bfgamma^t-\bfphi\right) + \bfy^k\rangle}{\lambda\|A\bfalpha^t -
A\bfs^t\|^2}\ ,\\\nonumber
\delta &= \frac{\langle A\bfalpha^t - A\bfs^t,
\tbfy^t\rangle}{\lambda\|A\bfalpha^t - A\bfs^t\|^2}\ .
\end{align}
In fact, if the optimal $\delta$ is out of the interval $[0,1]$, the value is
simply truncated to be in $[0,1]$.
\end{proof}

\section{Fast conditional gradient computation - Supplementary
material}\label{app:ph} 
In this section, we give the technical details of the original filtering
algorithm and then our modified filtering algorithm. To this end, we consider
the following computation
\begin{equation}
\forall\,a\in\allpixels, \quad v'_a = \sum_b k(\bff_a, \bff_b)\,v_b\,\I[y_a \ge
y_b]\ ,
\label{eqn:oph0}
\end{equation}
with $y_a, y_b \in [0,1]$ for all $a,b\in\allpixels$. Note that the above
equation is the same as Eq.~(\myref{14}), except for the multiplication by
the scalar $v_b$. In Section~\myref{4}, the value $v_b$ was assumed to be 1, but
here we consider the general case where $v_b\in \R$.

\subsection{Original filtering algorithm}\label{app:phold}
Let us first introduce some notations below. We denote the set of lattice points
of the original permutohedral lattice with $\calP$ and
the neighbouring feature points of lattice point $l$ by $N(l)$.
This neighbourhood is shown in Fig.~\myref{1} in the main paper. Furthermore, we denote the
neighbouring lattice points of a feature point $a$ by $\barN(a)$. In addition,
the barycentric weight between the lattice point $l$ and feature point $b$ is
denoted with $w_{lb}$. Furthermore, the value at feature point $b$ is denoted
by $v_b$ and the value at lattice point $l$ is denoted by $\barv_{l}$. 
Finally, the set of feature point scores is denoted by $\calY = \{y_b\mid b \in \allpixels\}$, their set of values is denoted by $\calV = \{v_b\mid b \in \allpixels\}$ and the set of lattice point values is denoted by $\barcalV = \{\barv_l\mid l \in \calP\}$.
The
pseudocode of the algorithm is given in Algorithm~\ref{alg:phold}.
\begin{algorithm}[H]
\caption{Original filtering algorithm~\cite{adams2010fast}}
\label{alg:phold}
\begin{algorithmic}
\Require Permutohedral lattice $\calP$, set of feature point values $\calV$
\State $\calV' \gets \bfzero\quad\barcalV \gets \bfzero\quad \barcalV' \gets \bfzero$\Comment{Initialization}
\ForAll{$l \in \calP$}\Comment{Splatting}
\ForAll{$b \in N(l)$}
\State $\barv_l \gets \barv_l + w_{lb}\,v_b$
\EndFor
\EndFor

\State $\barcalV' \gets k \otimes \barV$ \Comment{Blurring}

\ForAll{$a \in \allpixels$}\Comment{Slicing}
\ForAll{$l \in \barN(a)$}
\State $v'_a \gets v'_a + w_{la}\,\barv'_l$
\EndFor
\EndFor

\end{algorithmic}
\end{algorithm}

\subsection{Modified filtering algorithm}\label{app:phnew}
As mentioned in the main paper, the interval $[0,1]$ is discretized into $H$
 bins. Note that each bin $h\in \{0\ldots H-1\}$ is associated with an
interval which is identified as:
$\left[\frac{h}{H-1},\frac{h+1}{H-1}\right)$. Note that, the last bin (with bin id $H-1$) is associated with the interval $[1,\cdot)$. Since $y_b \le 1$, this bin contains the feature points whose scores are exactly $1$.
Given the score $y_b$ of the feature point $b$, its bin/level can be identified as
\begin{equation}
h_b = \left\lfloor y_b * (H-1) \right\rfloor\ ,
\end{equation}
where $\lfloor \cdot \rfloor$ denotes the standard floor function.

Furthermore, during splatting, the values $v_b$ are accumulated to the
neighbouring lattice point only if the lattice point is above or equal to the
feature point level. We denote the value at lattice point $l$ at level $h$ by $\barv_{l:h}$.
Formally, the barycentric interpolation at lattice point $l$ at level $h$ can be written as
\begin{equation}
\label{eqn:splat}
\barv_{l:h} = \sum_{\substack{b\in N(l)\\ h_b \le h}} w_{lb}\,v_b\ .
\end{equation}
Then, blurring is performed
independently at each discrete level $h$. Finally, during slicing, the
resulting values are interpolated at the feature point level.
Our modified algorithm is given in Algorithm~\ref{alg:phnew}. In this algorithm, we denote the set of values corresponding to all the lattice points at level $h$ as $\barcalV_h = \{v_{l:h}\mid l \in \calP\}$.

\begin{algorithm}[H]
\caption{Modified filtering algorithm}
\label{alg:phnew}
\begin{algorithmic}
\Require Permutohedral lattice $\calP$, set of feature point values $\calV$, discrete levels $H$, set of scores $\calY$
\State $\calV' \gets \bfzero\quad\barcalV \gets \bfzero\quad \barcalV' \gets \bfzero$\Comment{Initialization}
\ForAll{$l \in \calP$}\Comment{Splatting}
\ForAll{$b \in N(l)$}
\State $h_b \gets \left\lfloor y_b * (H-1) \right\rfloor$
\ForAll{$h \in \{h_b\ldots H-1\}$}\Comment{Splat at the feature point level and above}
\State $\barv_{l:h} \gets \barv_{l:h} + w_{lb}\,v_b$
\EndFor
\EndFor
\EndFor

\LineForAll{$h \in \{0\ldots H-1\}$} {$\barcalV'_h \gets k \otimes
\barcalV_h$}\Comment{Blurring at each level independently}

\ForAll{$a \in \allpixels$}\Comment{Slicing}
\State $h_a \gets \left\lfloor y_a * (H-1) \right\rfloor$
\ForAll{$l \in \barN(a)$}
\State $v'_a \gets v'_a + w_{la}\,\barv'_{l:h_a}$\Comment{Slice at the feature point level}
\EndFor
\EndFor

\end{algorithmic}
\end{algorithm}
Note that the above algorithm is given for the constraint $\I[y_a\ge y_b]$
(Eq.~(\myref{14})). However, it is fairly easy to modify it for the $\I[y_a\le
y_b]$ constraint. In particular, one needs to change the interval identified by
the bin $h$ to: $\left(\frac{h-1}{H-1},\frac{h}{H-1}\right]$. Using this fact,
one can easily derive the splatting and slicing equations for the $\I[y_a\le
y_b]$ constraint.
The algorithm given above introduces an approximation to the gradient computation that depends on the number of discrete bins $H$. However, this approximation can be eliminated by using a dynamic data structure which we briefly explain in the next section.

\subsubsection{Adaptive version of the modified filtering
algorithm}\label{app:dynph}
Here, we briefly explain the adaptive version of our modified algorithm, which replaces
the fixed discretization with a dynamic data structure.
Effectively, discretization boils down to storing a vector of length $H$ at each
lattice point.
Instead of such a fixed-length vector, one can use a dynamic data structure that
grows with the number of different scores encountered at each lattice
point in the splatting and blurring steps.
In the worst case, \ie, when all the neighbouring feature points have different
scores, the maximum number of values to store at a lattice point is
\begin{equation}
H = \max_l |N^2(l)|\ ,
\end{equation} 
where $N^2(l)$ denotes the union of neighbourhoods of the lattice point $l$ and
its neighbouring lattice points (the vertices of the shaded hexagon in
Fig.~\myref{1} in the main paper).
In our experiments, we observed that $|N^2(l)|$ is usually less than 100, with
an average around 10. Empirically, however, we found this dynamic version to be
slightly slower than the static one. We conjecture that this is due to the
static version benefitting from better compiler optimization. Furthermore, both the versions obtained results with similar precesion and therefore we used the static one for all our experiments.

\section{Additional experiments}
Let us first explain the pixel compatibility function used in the experiments.
We then turn to additional experiments.

\subsection{Pixel compatibility function used in the
experiments}\label{app:param} 
As mentioned in the main paper, our algorithm is
applicable to any pixel compatibility function that is composed of a mixture of Gaussian kernels. In all
our experiments, we used two kernels, namely spatial kernel and bilateral
kernel, similar to~\cite{desmaison2016efficient,koltun2011efficient}. Our pixel compatibility function can be written as
\begin{equation}
K_{ab} = w^{(1)}\,\exp\left(-\frac{|\bfp_a-\bfp_b|^2}{\sigma_1}\right) +
w^{(2)}\,\exp\left(-\frac{|\bfp_a-\bfp_b|^2}{\sigma_{2:s}}-\frac{|\bfbI_a-\bfbI_b|^2}{\sigma_{2:c}}\right)\
,
\end{equation}
where $\bfp_a$ denotes the $(x,y)$ position of pixel $a$ measured from top
left and $\bfbI_a$ denotes the $(r,g,b)$ values of pixel $a$. Note that
there are $5$ learnable parameters: $w^{(1)}, \sigma_1, w^{(2)},\sigma_{2:s},
\sigma_{2:c}$. These parameters are cross validated for different algorithms on
each data set. The final cross validated parameters for MF and DC$_{neg}$ are
given in Table~\ref{tab:cv}. 
To perform this cross-validation, we ran Spearmint for 2 days for each algorithm on both datasets. Note that, due to this time limitation, we were able to run approximately 1000 Spearmint iterations on MSRC but only 100 iterations on Pascal.
This is due to bigger images and larger validation set on the Pascal dataset. Hence, it resulted in less accurate energy parameters. 

\begin{table*}[t]
\begin{center}
\begin{tabular}{l|l|rr|rrr}
Data set& Algorithm &  $w^{(1)}$ & $\sigma_1$& $w^{(2)}$& $\sigma_{2:s}$ &
$\sigma_{2:c}$\\
\hline

\multirow{2}{*}{MSRC}
&MF&7.467846&1.000000&4.028773&35.865959&11.209644\\
&DC$_\text{neg}$&2.247081&3.535267&1.699011&31.232626&7.949970\\\hline

\multirow{2}{*}{Pascal}
&MF&100.000000&1.000000&74.877398&50.000000&5.454272\\
&DC$_\text{neg}$&0.500000&3.071772&0.960811&49.785678&1.000000

\end{tabular}
\end{center}
\caption{\em Parameters tuned for MF and DC$_\text{neg}$ on the MSRC and Pascal
validation sets using Spearmint~\cite{snoek2012practical}.}
\label{tab:cv}
\end{table*}

\subsection{Additional segmentation results}

In this section we provide additional segmentation results.

\subsubsection{Results on parameters tuned for MF}\label{app:addexp}

The results for the parameters tuned for MF on the MSRC and Pascal datasets are given in Table~\ref{tab:mf}. In
Fig.~\ref{fig:mfet}, we show the assignment energy as a function of time 
for an image in MSRC (the tree image in Fig.~\ref{fig:segmf}) and for an
image in Pascal (the sheep image in Fig.~\ref{fig:segmf}).
Furthermore, we provide some of the segmentation results in
Fig.~\ref{fig:segmf}.

Interestingly, for the parameters tuned for MF, even though our algorithm
obtains much lower energies, MF yields the best segmentation accuracy.
In fact, one can argue that the parameters tuned for MF do not model the
segmentation problem accurately, but were tuned such that the inaccurate MF
inference yields good results. 
Note that, in the Pascal dataset, when tuned for MF, the Gaussian mixture coefficients are very high (see Table~\ref{tab:cv}). In such a setting,
DC$_\text{neg}$ ended up classifying all pixel in most images as background.
In fact, SG-LP$_\ell$ was able to improve over DC$_\text{neg}$ in only 
1\% of the images, whereas all our versions improved over
DC$_\text{neg}$ in roughly 25\% of the images. 
Furthermore, our accelerated versions could not get any
advantage over the standard version and resulted in similar run times. Note that,
in most of the images, the \textit{uncertain} pixels are in fact the entire
image, as shown in Fig.~\ref{fig:segmf}.

\begin{figure*}
\def\IMHEIGHT{30ex}
\def\IMHEIGHTZ{27ex}
\begin{center}
\begin{subfigure}{0.34\linewidth}
\includegraphics[height=\IMHEIGHT]{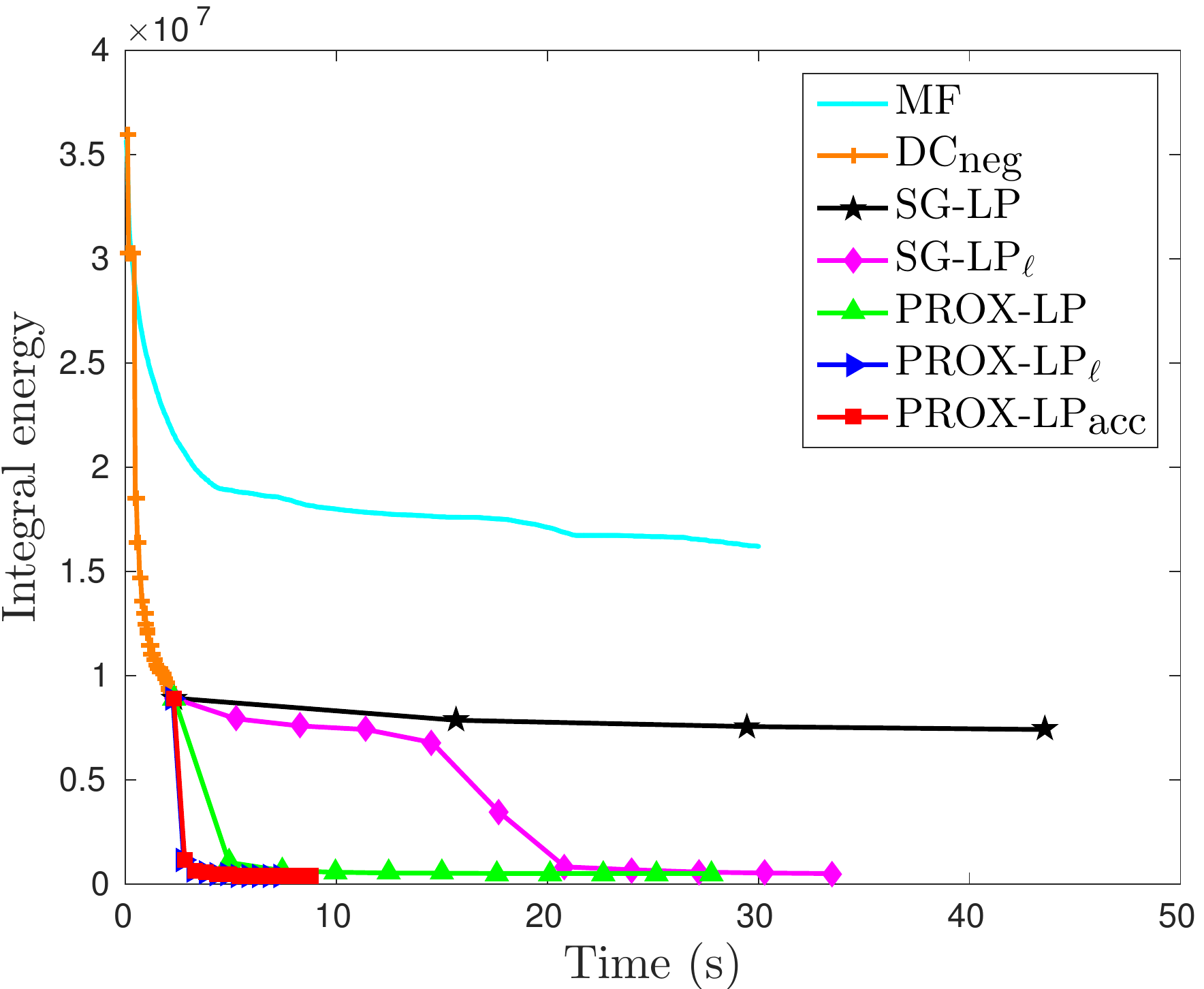}
\end{subfigure}%
\hfill
\begin{subfigure}{0.16\linewidth}
\includegraphics[height=\IMHEIGHTZ]{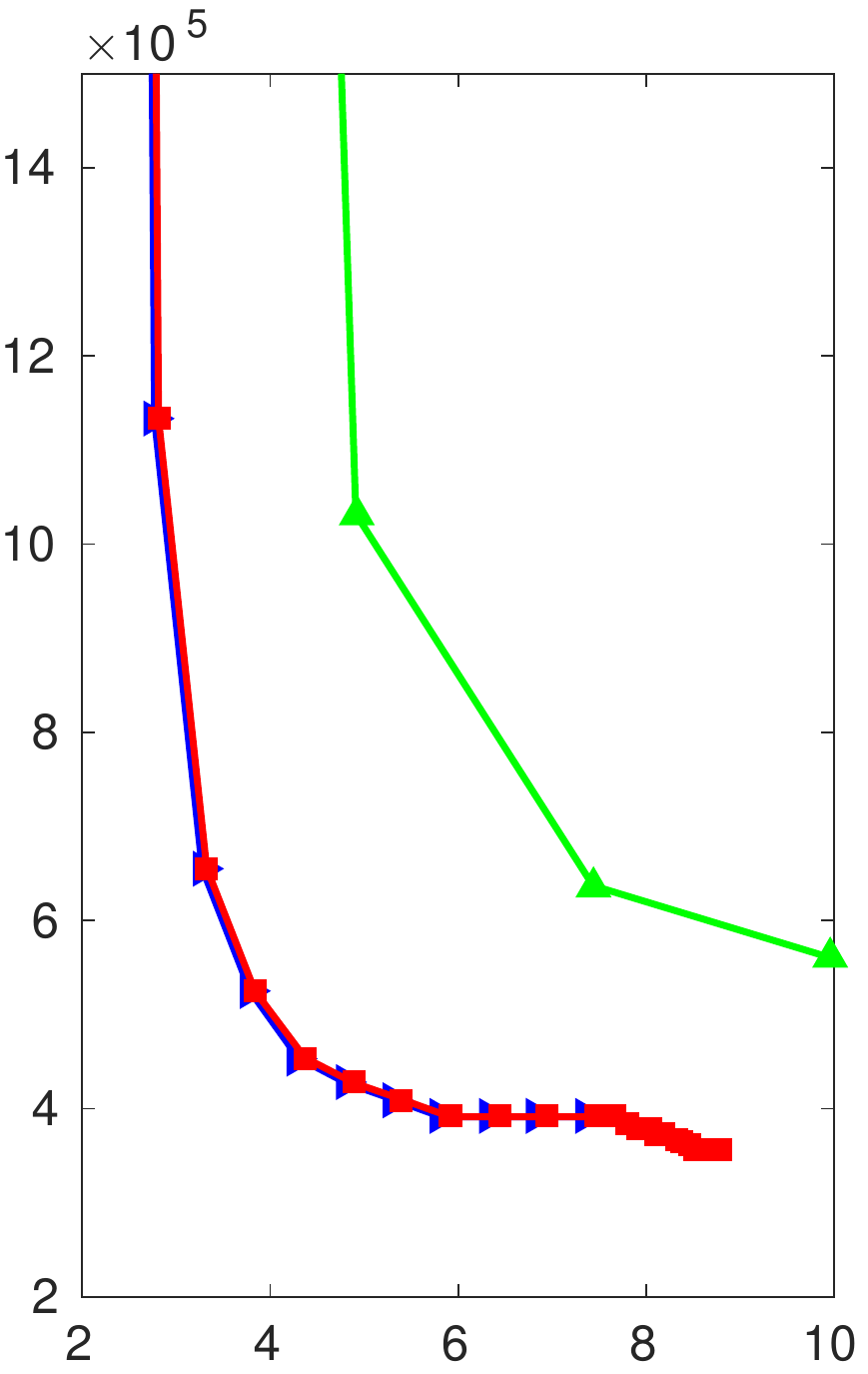}
\end{subfigure}%
\hfill
\begin{subfigure}{0.34\linewidth}
\includegraphics[height=\IMHEIGHT]{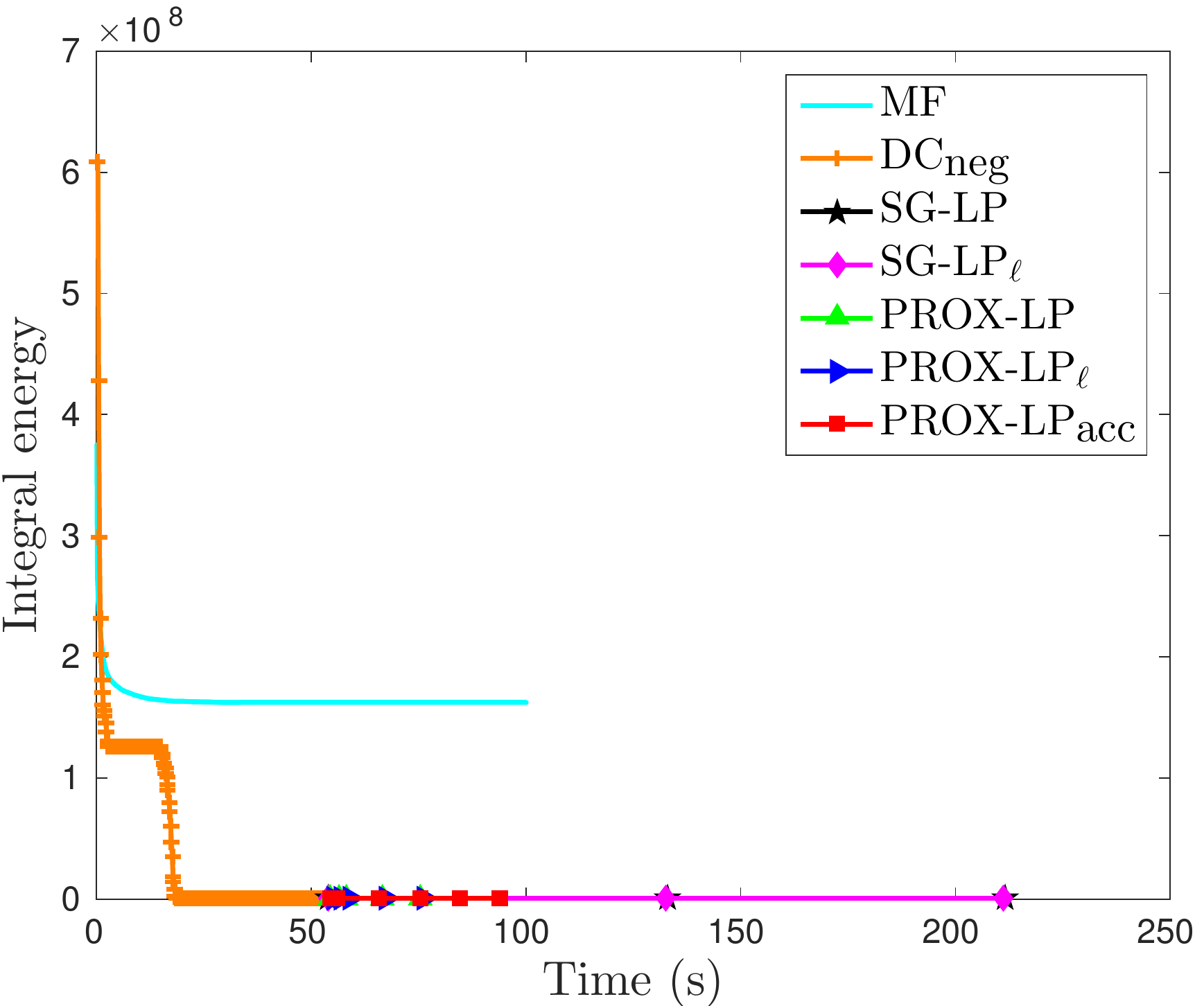}
\end{subfigure}%
\hfill
\begin{subfigure}{0.16\linewidth}
\includegraphics[height=\IMHEIGHTZ]{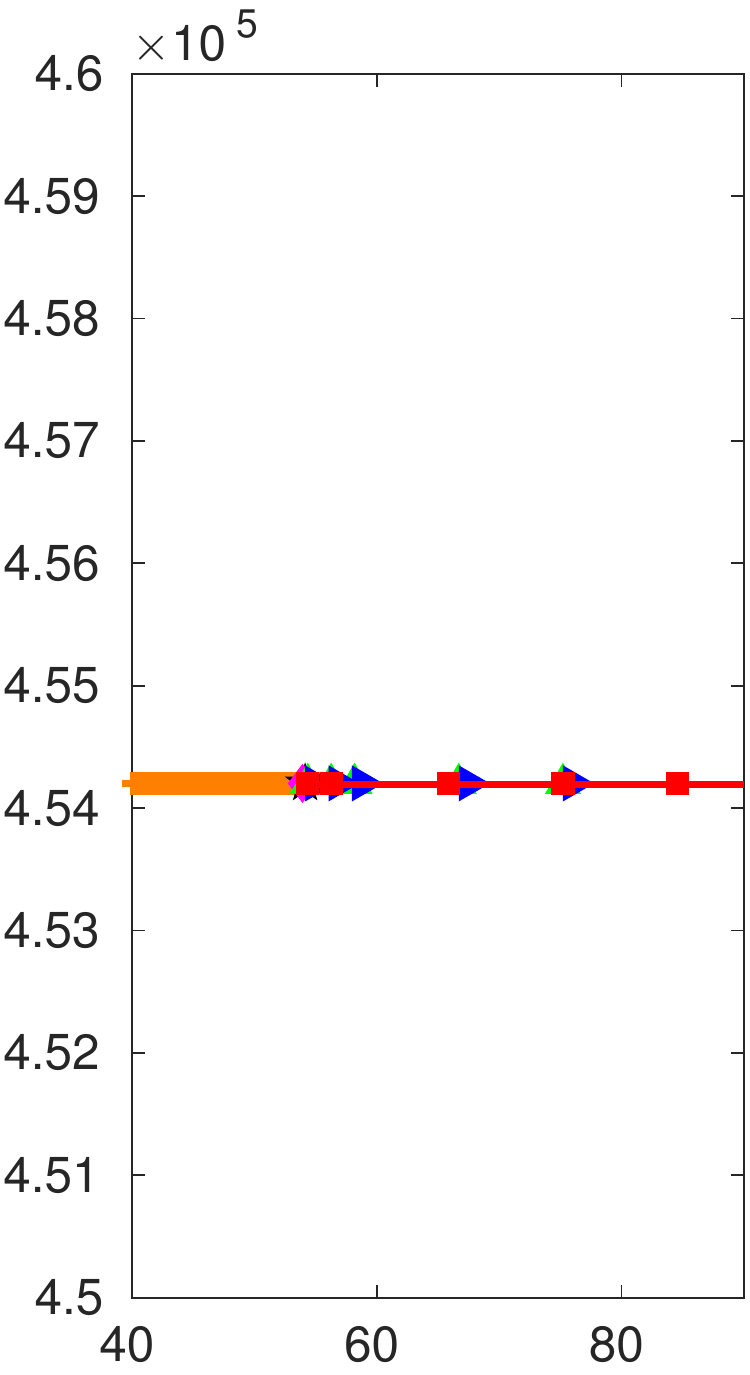}
\end{subfigure}
\caption{\em Assignment energy as a function of time for MF
parameters for an image in (\textbf{left}) MSRC and (\textbf{right}) Pascal.
A zoomed-in version is shown next to each plot.
Except for MF, all the algorithms were initialized with DC$_\text{neg}$.
For the MSRC image, PROX-LP clearly outperforms SG-LP$_\ell$ by obtaining much
lower energies in fewer iterations, and the accelerated versions of
our algorithm obtain roughly the same energy as PROX-LP but significantly
faster.
For the Pascal image, however, no LP algorithm is able to improve over DC$_\text{neg}$. Note that, in the Pascal dataset, for the MF parameters,
DC$_\text{neg}$ ended up classifying all pixel in most images as background (which yields low energy values) and no LP algorithm is able to improve over it.
}
\label{fig:mfet}
\end{center}
\end{figure*}

\begin{table*}[t]
\begin{center}
\begin{tabular}
{>{\raggedright\arraybackslash}m{0.15cm}|
>{\raggedright\arraybackslash}m{1.8cm}|>{\raggedleft\arraybackslash}m{0.5cm}
>{\raggedleft\arraybackslash}m{0.5cm}>{\raggedleft\arraybackslash}m{0.5cm}
>{\raggedleft\arraybackslash}m{1.0cm}>{\raggedleft\arraybackslash}m{1.0cm}
>{\raggedleft\arraybackslash}m{1.0cm}>{\raggedleft\arraybackslash}m{1.0cm}
|>{\raggedleft\arraybackslash}m{1.1cm}>{\raggedleft\arraybackslash}m{1.0cm}
>{\raggedleft\arraybackslash}m{1.0cm}>{\raggedleft\arraybackslash}m{1.0cm}}
&& MF5 & MF & DC$_\text{neg}$ & SG-LP$_\ell$ & PROX-LP  &
 PROX-LP$_\ell$ & PROX-LP$_\text{acc}$ & Ave. E ($\times10^4$) & Ave. T (s) &
 Acc. & IoU \\
\hline
\parbox[t]{2mm}{\multirow{7}{*}{\rotatebox[origin=c]{90}{MSRC}}}
&MF5&-&0&0&0&0&0&0&2366.6&\textbf{0.2}&81.14&54.60\\
&MF&95&-&18&15&2&1&2&1053.6&13.0&\textbf{83.86}&\textbf{59.75}\\
&DC$_\text{neg}$&95&77&-&0&0&0&0&812.7&2.8&83.50&59.67\\
&SG-LP$_\ell$&95&80&48&-&2&0&1&800.1&37.3&83.51&59.68\\\cline{2-13}
&PROX-LP&95&93&95&93&-&35&46&265.6&27.3&83.01&58.74\\
&PROX-LP$_\ell$&95&94&94&94&59&-&43&\textbf{261.2}&13.9&82.98&58.62\\
&PROX-LP$_\text{acc}$&95&93&93&93&49&46&-&295.9&7.9&83.03&58.97\\

\hline
\hline
\parbox[t]{1mm}{\multirow{7}{*}{\rotatebox[origin=c]{90}{Pascal}}}
&MF5&-&-&1&1&0&0&0&40779.8&\textbf{0.8}&80.42&28.66\\
&MF&93&-&3&3&0&0&1&20354.9&21.7&\textbf{80.95}&\textbf{28.86}\\
&DC$_\text{neg}$&93&87&-&0&0&0&0&2476.2&39.1&77.77&14.93\\
&SG-LP$_\ell$&93&87&1&-&0&0&0&2474.1&414.7&77.77&14.92\\\cline{2-13}
&PROX-LP&94&90&24&24&-&4&9&1475.6&81.0&78.04&15.79\\
&PROX-LP$_\ell$&94&90&24&24&5&-&9&\textbf{1458.9}&82.7&78.04&15.79\\
&PROX-LP$_\text{acc}$&94&89&28&27&18&18&-&1623.7&83.9&77.86&15.18\\

\end{tabular}
\end{center}

\vspace{-0.5cm}
\caption{\em Results on the MSRC and Pascal datasets with the parameters tuned for
MF. We show: the percentage of images where the row method
strictly outperforms the column one on the final integral energy, the average
integral energy over the test set, the average run time, the segmentation accuracy and the intersection over union score. Note
that all versions of our algorithm obtain much lower energies than the baselines.
However, as expected, lower energy does not correspond to better segmentation
accuracy, mainly due to the less accurate energy parameters.
Furthermore, the accelerated versions of our algorithm are similar in run time
and obtain similar energies compared to PROX-LP.}
\label{tab:mf}
\end{table*}

\begin{figure*}
\def \SUBWIDTH{0.10\linewidth}
\begin{center}
\begin{subfigure}{\SUBWIDTH}
\includegraphics[width=0.99\linewidth]{msrc/input/2_14_s.png}
\end{subfigure}%
\begin{subfigure}{\SUBWIDTH}
\includegraphics[width=0.99\linewidth]{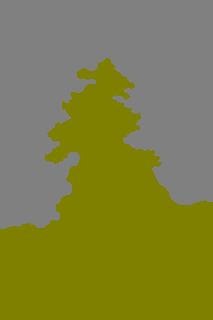}
\end{subfigure}%
\begin{subfigure}{\SUBWIDTH}
\includegraphics[width=0.99\linewidth]{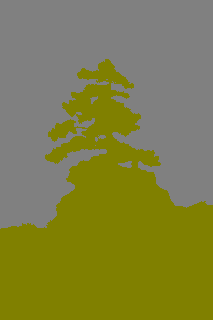}
\end{subfigure}%
\begin{subfigure}{\SUBWIDTH}
\includegraphics[width=0.99\linewidth]{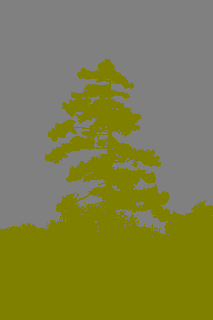}
\end{subfigure}%
\begin{subfigure}{\SUBWIDTH}
\mybox{red}{\includegraphics[width=0.99\linewidth]{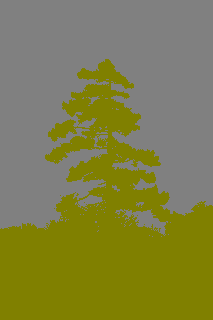}}
\end{subfigure}%
\begin{subfigure}{\SUBWIDTH}
\mybox{red}{\includegraphics[width=0.99\linewidth]{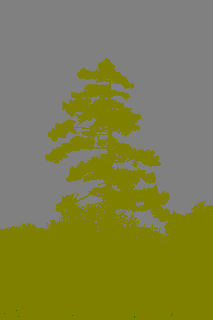}}
\end{subfigure}%
\begin{subfigure}{\SUBWIDTH}
\includegraphics[width=0.99\linewidth]{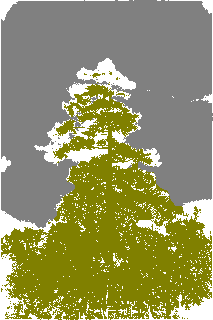}
\end{subfigure}%
\begin{subfigure}{\SUBWIDTH}
\mybox{red}{\includegraphics[width=0.99\linewidth]{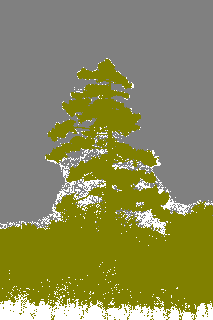}}
\end{subfigure}%
\begin{subfigure}{\SUBWIDTH}
\mybox{red}{\includegraphics[width=0.99\linewidth]{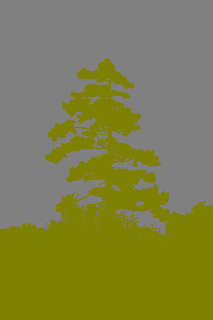}}
\end{subfigure}%
\begin{subfigure}{\SUBWIDTH}
\includegraphics[width=0.99\linewidth]{msrc/fine_annot/2_14_s_GT.png}
\end{subfigure}

\begin{subfigure}{\SUBWIDTH}
\includegraphics[width=0.99\linewidth]{msrc/input/1_9_s.png}
\end{subfigure}%
\begin{subfigure}{\SUBWIDTH}
\includegraphics[width=0.99\linewidth]{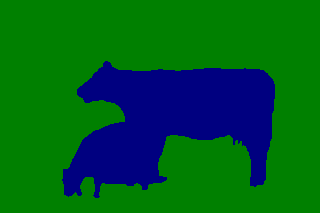}
\end{subfigure}%
\begin{subfigure}{\SUBWIDTH}
\includegraphics[width=0.99\linewidth]{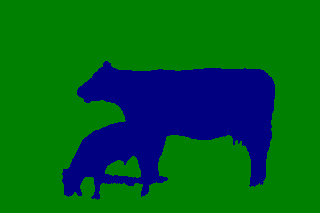}
\end{subfigure}%
\begin{subfigure}{\SUBWIDTH}
\includegraphics[width=0.99\linewidth]{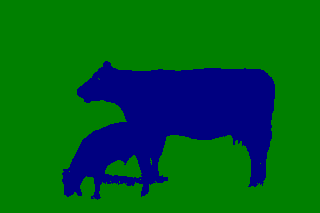}
\end{subfigure}%
\begin{subfigure}{\SUBWIDTH}
\mybox{red}{\includegraphics[width=0.99\linewidth]{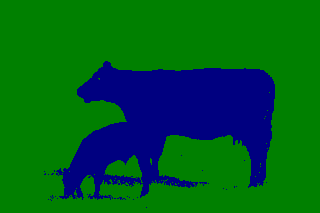}}
\end{subfigure}%
\begin{subfigure}{\SUBWIDTH}
\mybox{red}{\includegraphics[width=0.99\linewidth]{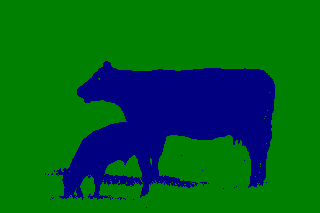}}
\end{subfigure}%
\begin{subfigure}{\SUBWIDTH}
\includegraphics[width=0.99\linewidth]{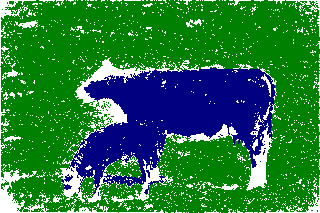}
\end{subfigure}%
\begin{subfigure}{\SUBWIDTH}
\mybox{red}{\includegraphics[width=0.99\linewidth]{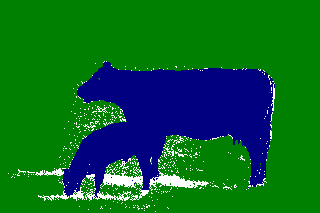}}
\end{subfigure}%
\begin{subfigure}{\SUBWIDTH}
\mybox{red}{\includegraphics[width=0.99\linewidth]{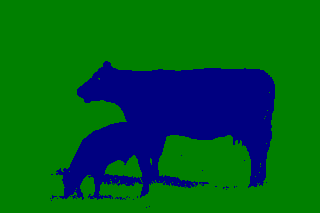}}
\end{subfigure}%
\begin{subfigure}{\SUBWIDTH}
\includegraphics[width=0.99\linewidth]{msrc/fine_annot/1_9_s_GT.png}
\end{subfigure}

\begin{subfigure}{\SUBWIDTH}
\includegraphics[width=0.99\linewidth]{pascal/input/2007_000676.jpg}
\end{subfigure}%
\begin{subfigure}{\SUBWIDTH}
\includegraphics[width=0.99\linewidth]{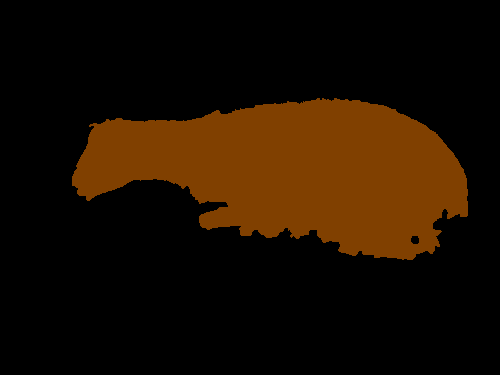}
\end{subfigure}%
\begin{subfigure}{\SUBWIDTH}
\includegraphics[width=0.99\linewidth]{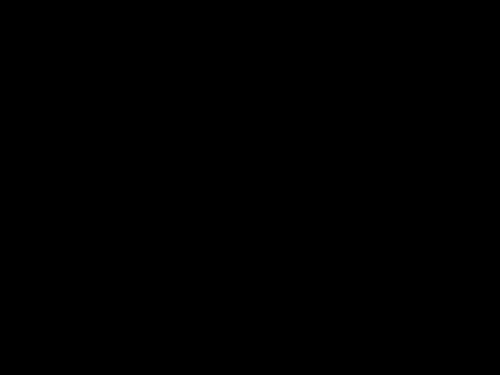}
\end{subfigure}%
\begin{subfigure}{\SUBWIDTH}
\includegraphics[width=0.99\linewidth]{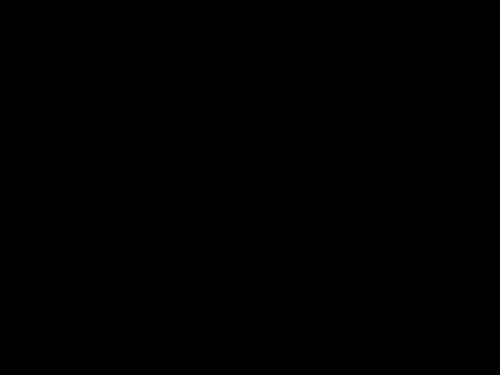}
\end{subfigure}%
\begin{subfigure}{\SUBWIDTH}
\mybox{red}{\includegraphics[width=0.99\linewidth]{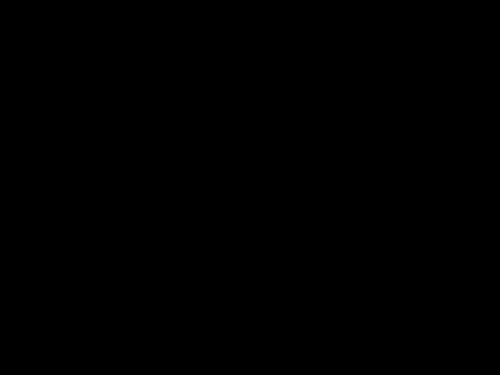}}
\end{subfigure}%
\begin{subfigure}{\SUBWIDTH}
\mybox{red}{\includegraphics[width=0.99\linewidth]{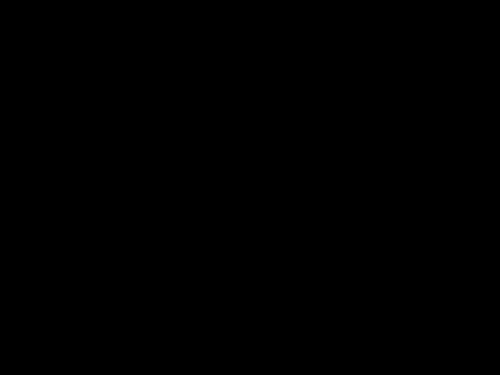}}
\end{subfigure}%
\begin{subfigure}{\SUBWIDTH}
\fbox{\includegraphics[width=0.99\linewidth]{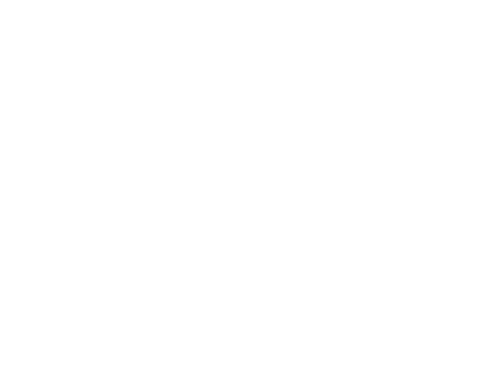}}
\end{subfigure}%
\begin{subfigure}{\SUBWIDTH}
\fbox{\includegraphics[width=0.99\linewidth]{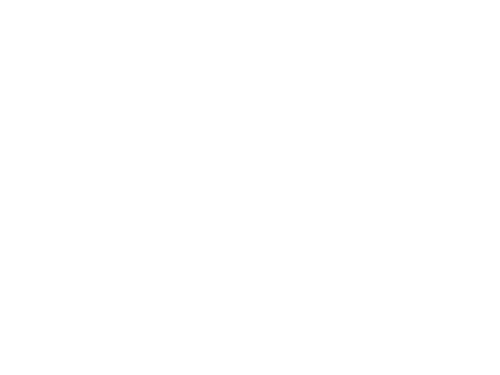}}
\end{subfigure}%
\begin{subfigure}{\SUBWIDTH}
\mybox{red}{\includegraphics[width=0.99\linewidth]{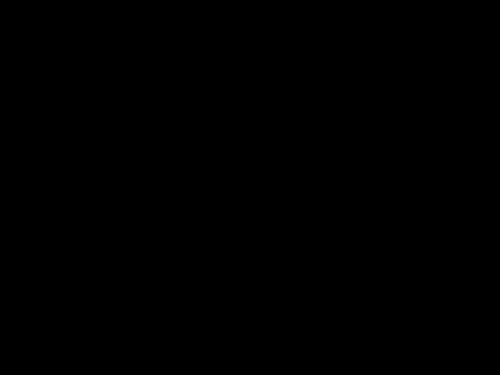}}
\end{subfigure}%
\begin{subfigure}{\SUBWIDTH}
\includegraphics[width=0.99\linewidth]{pascal/gt/2007_000676.png}
\end{subfigure}

\begin{subfigure}{\SUBWIDTH}
\includegraphics[width=0.99\linewidth]{pascal/input/2007_000559.jpg}
\caption{Image}
\end{subfigure}%
\begin{subfigure}{\SUBWIDTH}
\includegraphics[width=0.99\linewidth]{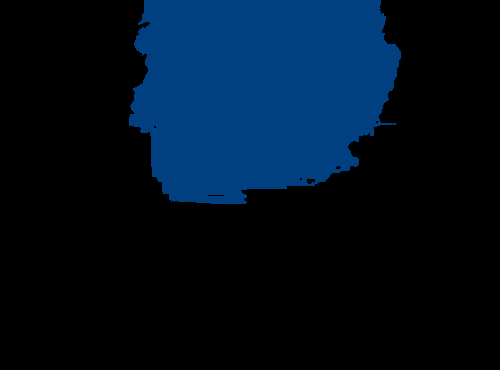}
\caption{MF}
\end{subfigure}%
\begin{subfigure}{\SUBWIDTH}
\includegraphics[width=0.99\linewidth]{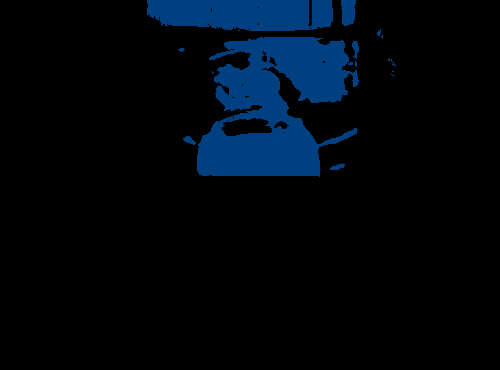}
\caption{DC$_\text{neg}$}
\end{subfigure}%
\begin{subfigure}{\SUBWIDTH}
\includegraphics[width=0.99\linewidth]{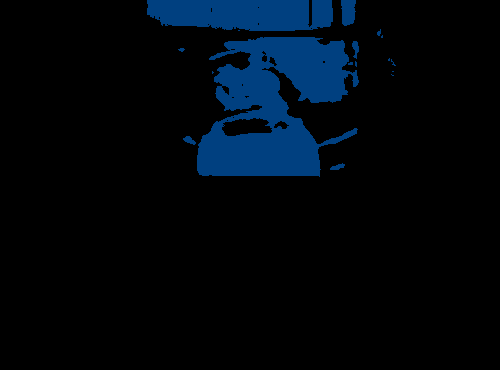}
\caption{SG-LP$_\ell$}
\end{subfigure}%
\begin{subfigure}{\SUBWIDTH}
\mybox{red}{\includegraphics[width=0.99\linewidth]{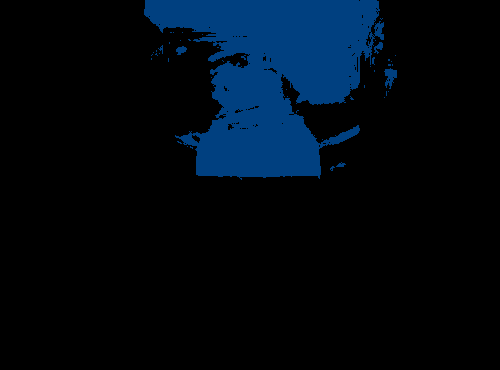}}
\caption{PROX-LP}
\end{subfigure}%
\begin{subfigure}{\SUBWIDTH}
\mybox{red}{\includegraphics[width=0.99\linewidth]{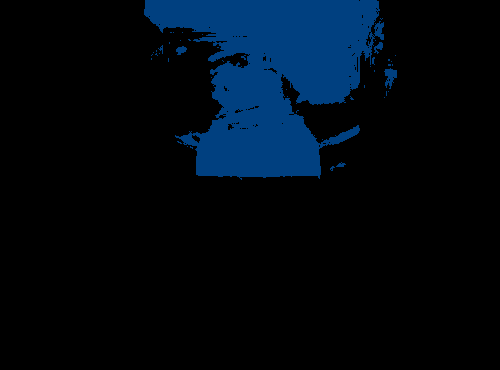}}
\caption{PROX-LP$_\ell$}
\end{subfigure}%
\begin{subfigure}{\SUBWIDTH}
\fbox{\includegraphics[width=0.99\linewidth]{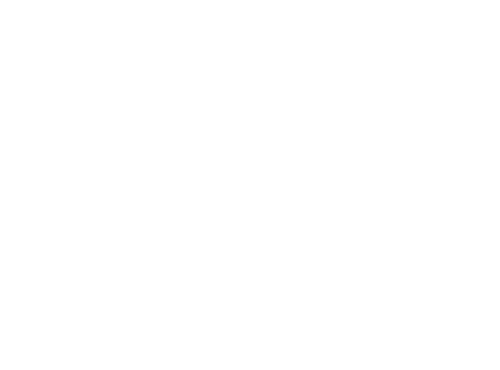}}
\caption{Uncer.(DC$_\text{neg}$)}
\end{subfigure}%
\begin{subfigure}{\SUBWIDTH}
\fbox{\includegraphics[width=0.99\linewidth]{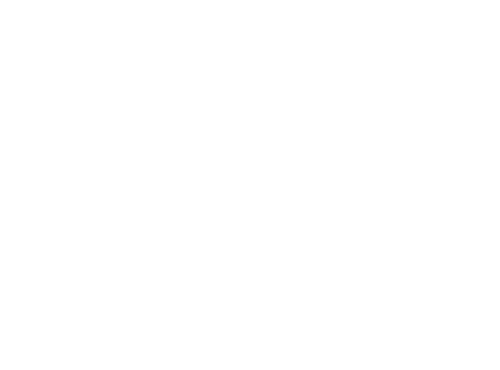}}
\caption{Uncer.(ours)}
\end{subfigure}%
\begin{subfigure}{\SUBWIDTH}
\mybox{red}{\includegraphics[width=0.99\linewidth]{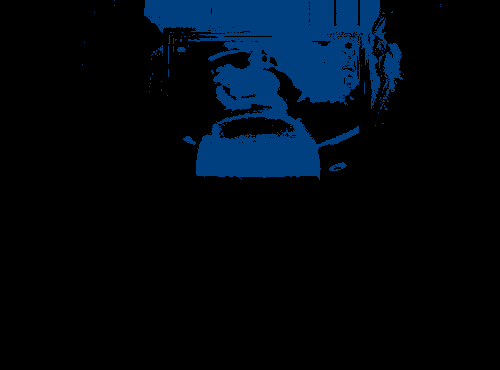}}
\caption{PROX-LP$_\text{acc}$}
\end{subfigure}%
\begin{subfigure}{\SUBWIDTH}
\includegraphics[width=0.99\linewidth]{pascal/gt/2007_000559.png}
\caption{Ground truth}
\end{subfigure}

\vspace{-0.2cm}
\caption{\em Results with MF parameters, for an image in (\textbf{top}) MSRC
and (\textbf{bottom}) Pascal. The uncertain pixels identified by DC$_\text{neg}$
and PROX-LP$_\text{acc}$ are marked in white. Note that, in MSRC all versions of
our algorithm obtain visually good segmentations similar to MF (or better).
In Pascal, the segmentation results are poor except for MF, even though they obtain much
lower energies. We argue that, in this case, the energy parameters do not model
the segmentation problem accurately.}
\label{fig:segmf}
\end{center}
\end{figure*}

\subsubsection{Summary}
We have evaluated all the algorithms using two different parameter settings.
Therefore, we summarize the best segmentation accuracy obtained by each
algorithm and the corresponding parameter setting in Table~\ref{tab:seg}. Note that, on MSRC, the best parameter setting for DC$_\text{neg}$ corresponds to the parameters tuned for MF. This is a strange result but can be explained by the fact that, as mentioned in the main paper, cross-validation was performed using the less accurate ground truth provided with the original dataset, but evaluation using the accurate ground truth annotations provided by~\cite{koltun2011efficient}.

Furthermore, in contrast to MSRC, the segmentation results of our algorithm on the Pascal dataset is not the state-of-the-art, even with the parameters tuned for DC$_\text{neg}$. This may be explained by the fact, that due to the limited cross-validation, the energy parameters obtained for the Pascal dataset is not accurate. Therefore, even though our algorithm obtained lower energies that was not reflected in the segmentation accuracy. Similar behaviour was observed in~\cite{desmaison2016efficient,wang2015efficient}.

\begin{table*}[t]
\begin{center}
\begin{tabular}{l||l|rr||l|rr}
\multirow{2}{*}{Algorithm} & \multicolumn{3}{c||}{MSRC} &
\multicolumn{3}{c}{Pascal}\\
 & Parameters & Ave. T (s) & Acc. & Parameters & Ave. T (s) & Acc. \\
\hline
MF5&MF&\textbf{0.2}&81.14&MF&\textbf{0.8}&80.42\\
MF&MF&13.0&83.86&MF&21.7&\textbf{80.95}\\
DC$_\text{neg}$&MF&2.8&83.50&DC$_\text{neg}$&3.7&80.43\\
SG-LP$_\ell$&MF&37.3&83.51&DC$_\text{neg}$&84.4&80.49\\\hline
PROX-LP&DC$_\text{neg}$&23.5&83.99&DC$_\text{neg}$&106.7&80.63\\
PROX-LP$_\ell$&DC$_\text{neg}$&6.3&83.94&DC$_\text{neg}$&22.1&80.65\\
PROX-LP$_\text{acc}$&DC$_\text{neg}$&3.7&\textbf{84.16}&DC$_\text{neg}$&14.7&80.58\\


\end{tabular}
\end{center}

\vspace{-0.5cm}
\caption{\em Best segmentation results of each algorithm with their respective
parameters, the average time on the test set and the segmentation accuracy. In MSRC,
the best segmentation accuracy is obtained by PROX-LP$_\text{acc}$ and in
Pascal it is by MF. Note that, on MSRC, the best parameter setting for DC$_\text{neg}$ corresponds to the parameters tuned for MF. This is due to the fact that cross-validation was performed on the less accurate ground truth but evaluation on the accurate ground truth annotations provided by~\cite{koltun2011efficient}. Furthermore, the low segmentation performance of our algorithm on the Pascal dataset is may be due to less accurate energy parameters resulted from limited cross-validation.}
\label{tab:seg}
\end{table*}

\subsection{Effect of the proximal regularization constant}
We plot the assignment energy as a function of time for an image in MSRC
(the same image used to generate Fig.~\ref{fig:et}) by varying the
proximal regularization constant $\lambda$. Here, we used the parameters tuned for
DC$_\text{neg}$. The plot is shown in Fig.~\ref{fig:msrclambdaet}. In summary,
for a wide range of $\lambda$, PROX-LP obtains similar energies with approximately the same run time.
\begin{figure}
\def\IMHEIGHT{30ex}
\def\IMHEIGHTZ{27ex}
\begin{center}
\begin{subfigure}{0.45\linewidth}
\includegraphics[width=0.99\linewidth]{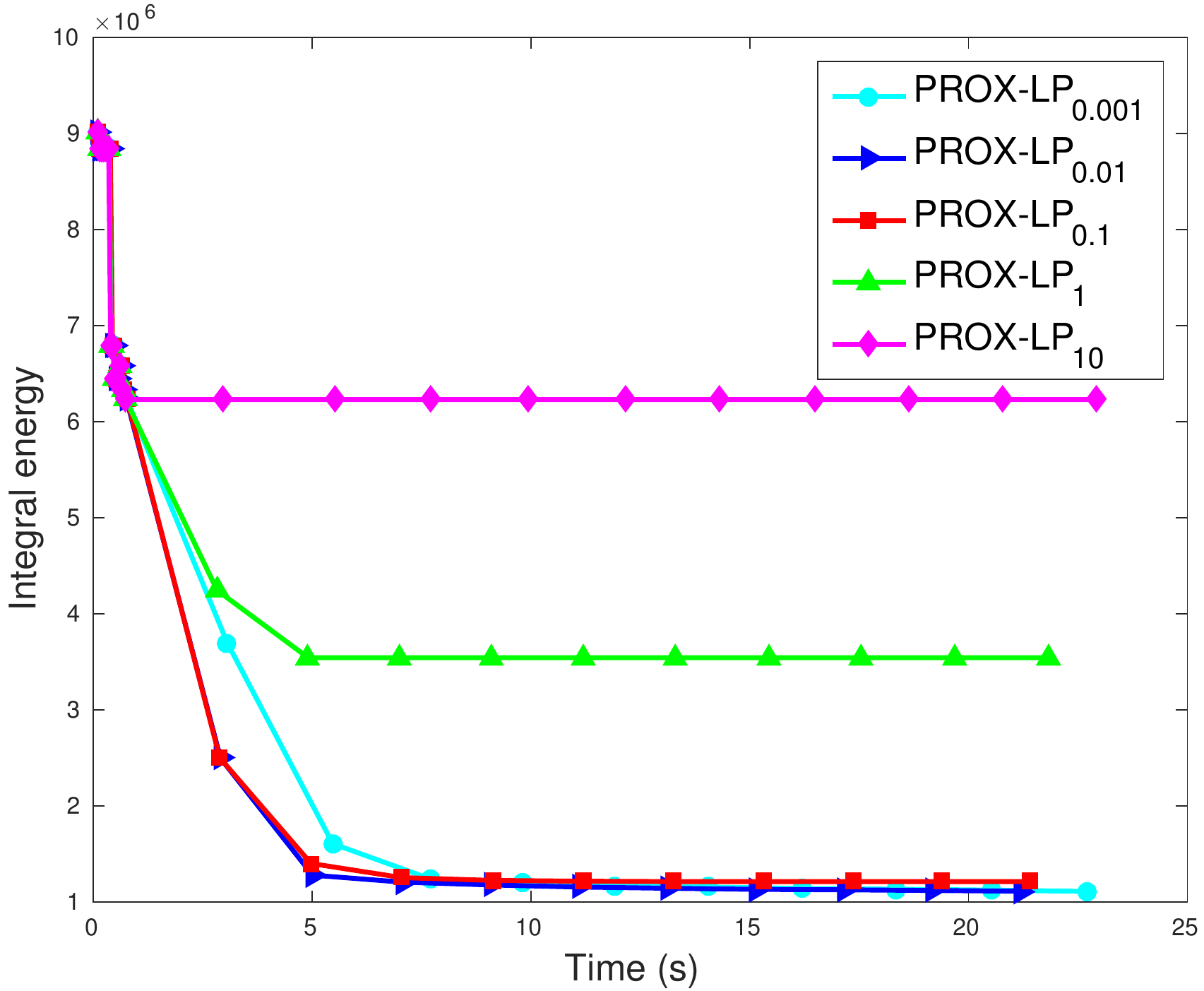}
\end{subfigure}%
\hfill
\begin{subfigure}{0.45\linewidth}
\includegraphics[width=0.99\linewidth]{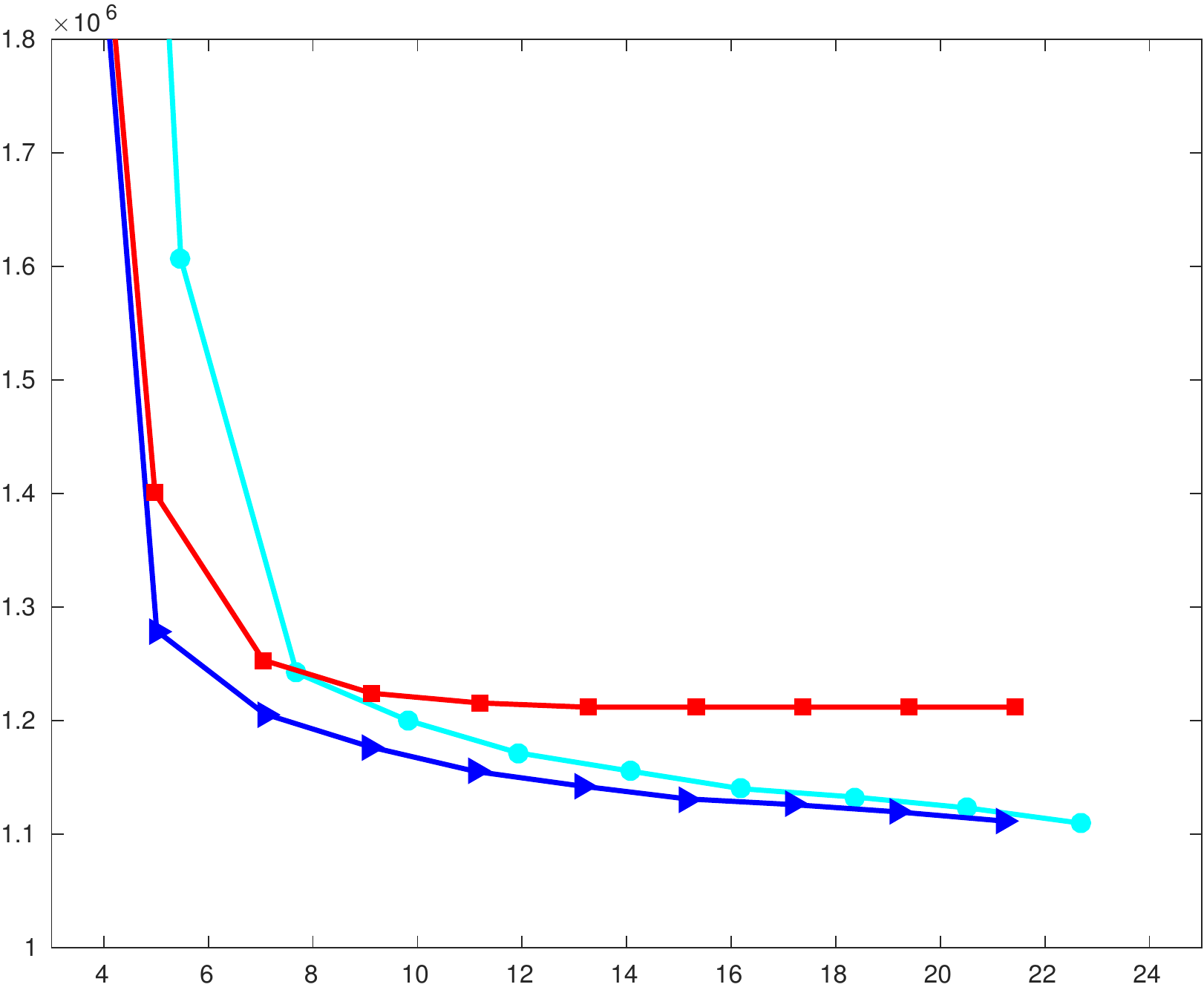}
\end{subfigure}
\caption{\em Assignment energy as a function of time for an image in MSRC, for
different values of $\lambda$. The zoomed plot is shown on the right. Note
that, for $\lambda = 0.1, 0.01, 0.001$, PROX-LP obtains similar energies in
approximately the same run time.}
\label{fig:msrclambdaet}
\end{center}
\end{figure}

\subsection{Modified filtering algorithm}\label{app:spph}

We compare our modified filtering method, described in
Section~\myref{4}, with the divide-and-conquer strategy
of~\cite{desmaison2016efficient}. To this end, we evaluated both algorithms on
one of the Pascal VOC test images (the sheep image in Fig.~\ref{fig:seg}),
but varying the image size, the number of labels and the Gaussian kernel
standard deviation. The respective plots are shown in Fig.~\ref{fig:pasph}. 
Note that, as claimed in the main paper, speedup with respect to the standard
deviation is roughly constant. 
Similar plots for an MSRC image (the tree image in Fig.~\ref{fig:seg}) are
shown in Fig.~\ref{fig:msrcph}. In this case, speedup is around $15-32$, with
around $23-32$ in the operating region of all versions of our algorithm.

\begin{figure*}
\def\SUBWIDTH{0.3\linewidth}
\begin{center}
\begin{subfigure}{\SUBWIDTH}
\includegraphics[width=0.95\linewidth]{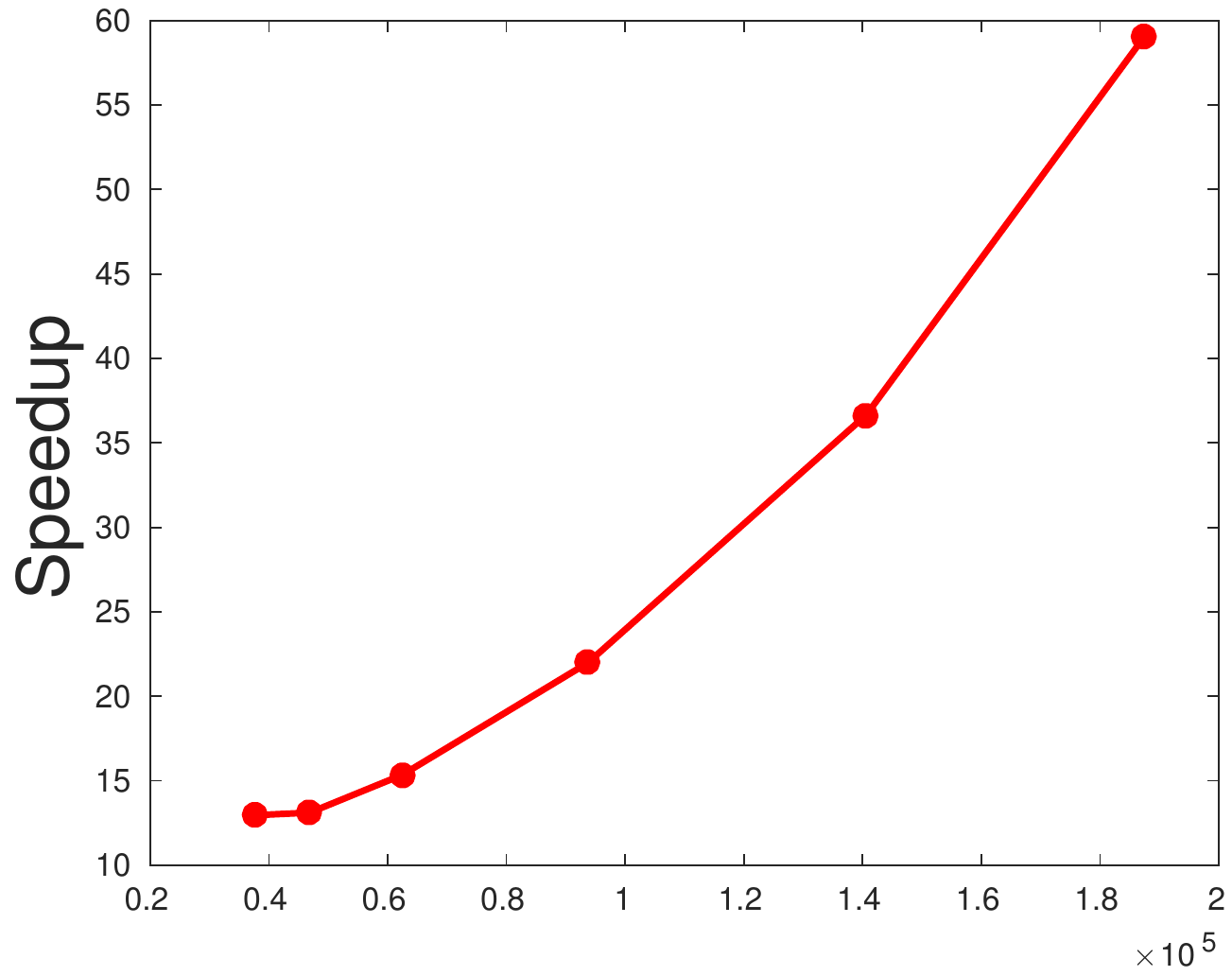}
\end{subfigure}%
\begin{subfigure}{\SUBWIDTH}
\includegraphics[width=0.95\linewidth]{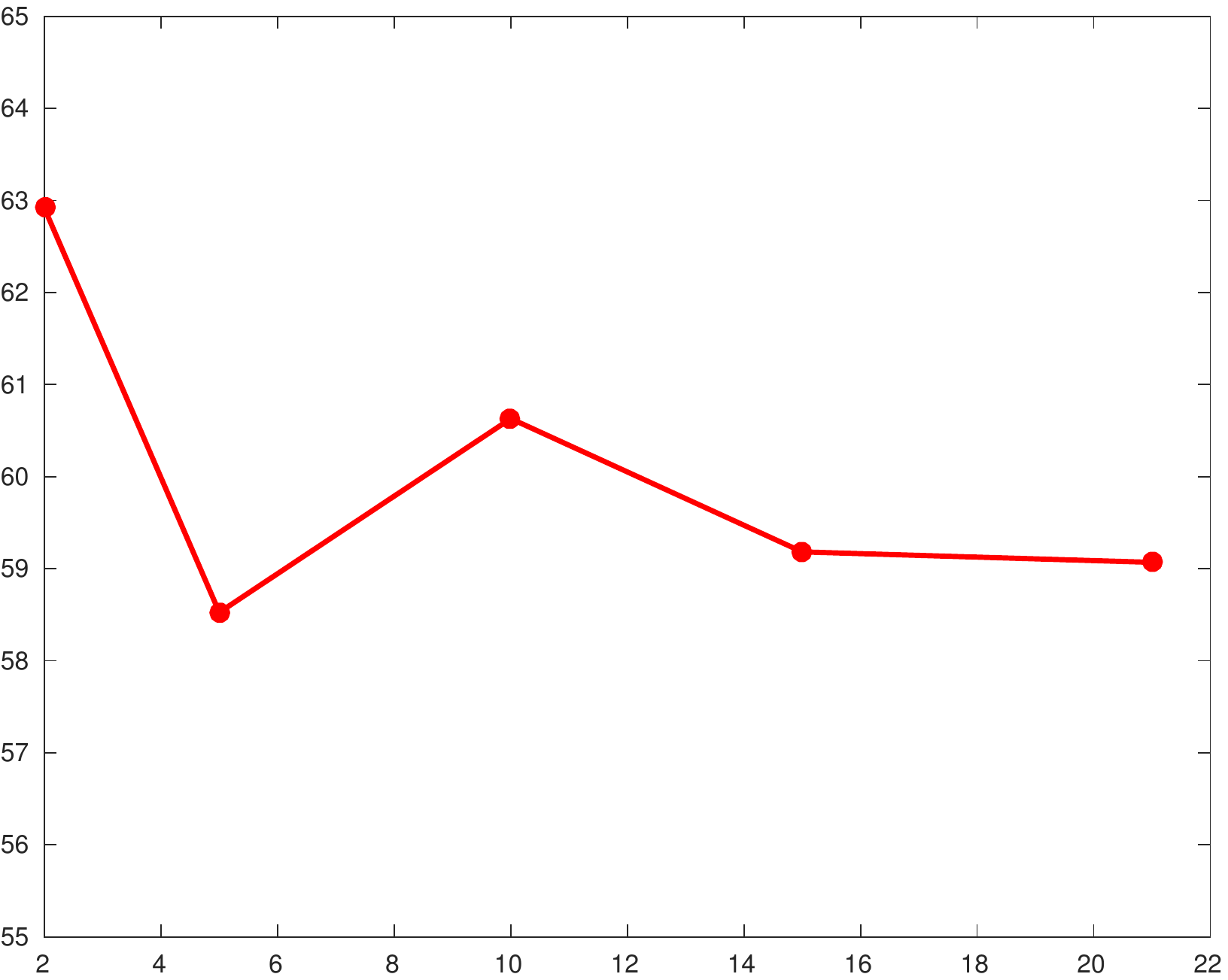}
\end{subfigure}%
\begin{subfigure}{\SUBWIDTH}
\includegraphics[width=0.95\linewidth]{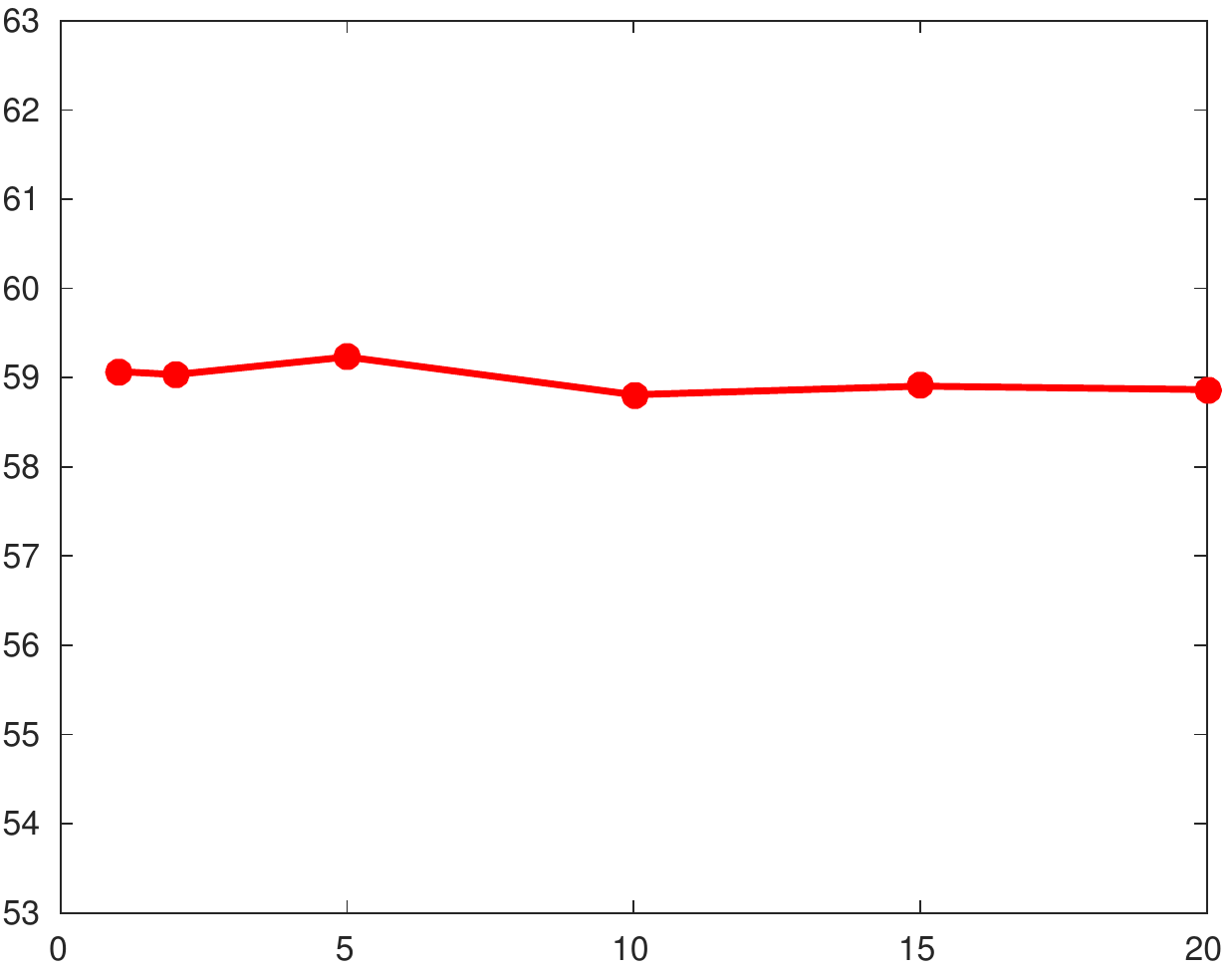}
\end{subfigure}%

\begin{subfigure}{\SUBWIDTH}
\includegraphics[width=0.95\linewidth]{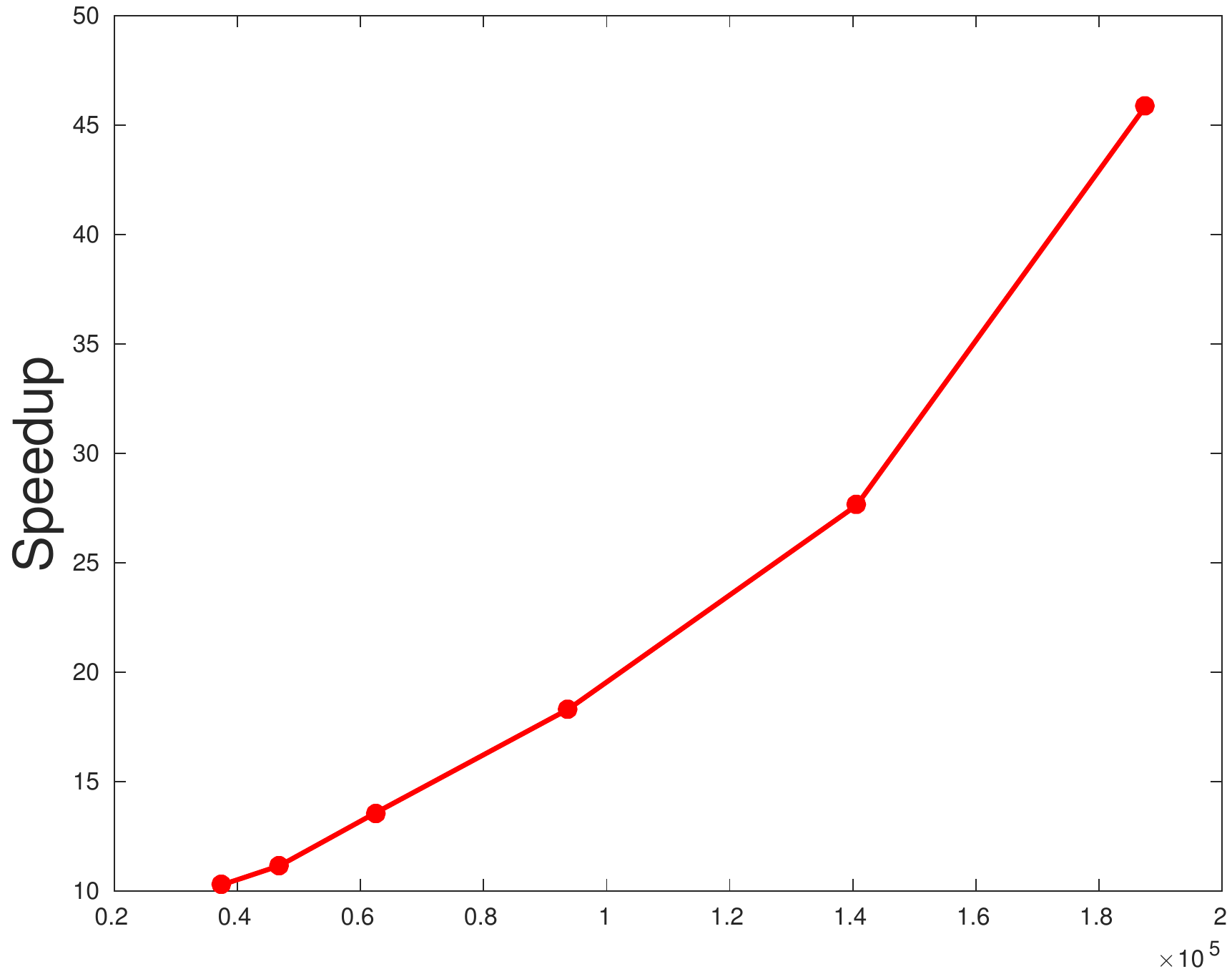}
\caption{Number of pixels}
\end{subfigure}%
\begin{subfigure}{\SUBWIDTH}
\includegraphics[width=0.95\linewidth]{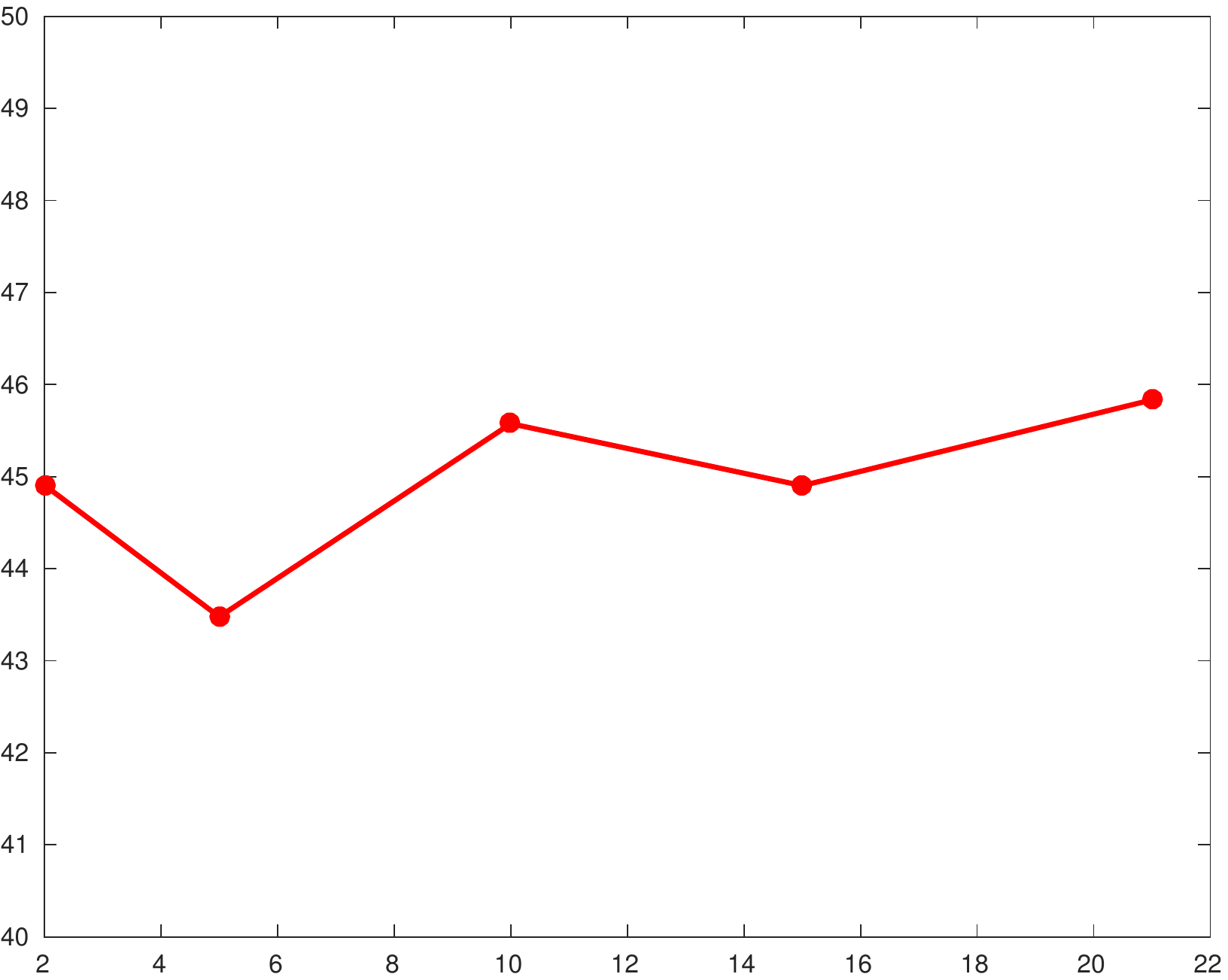}
\caption{Number of labels}
\end{subfigure}%
\begin{subfigure}{\SUBWIDTH}
\includegraphics[width=0.95\linewidth]{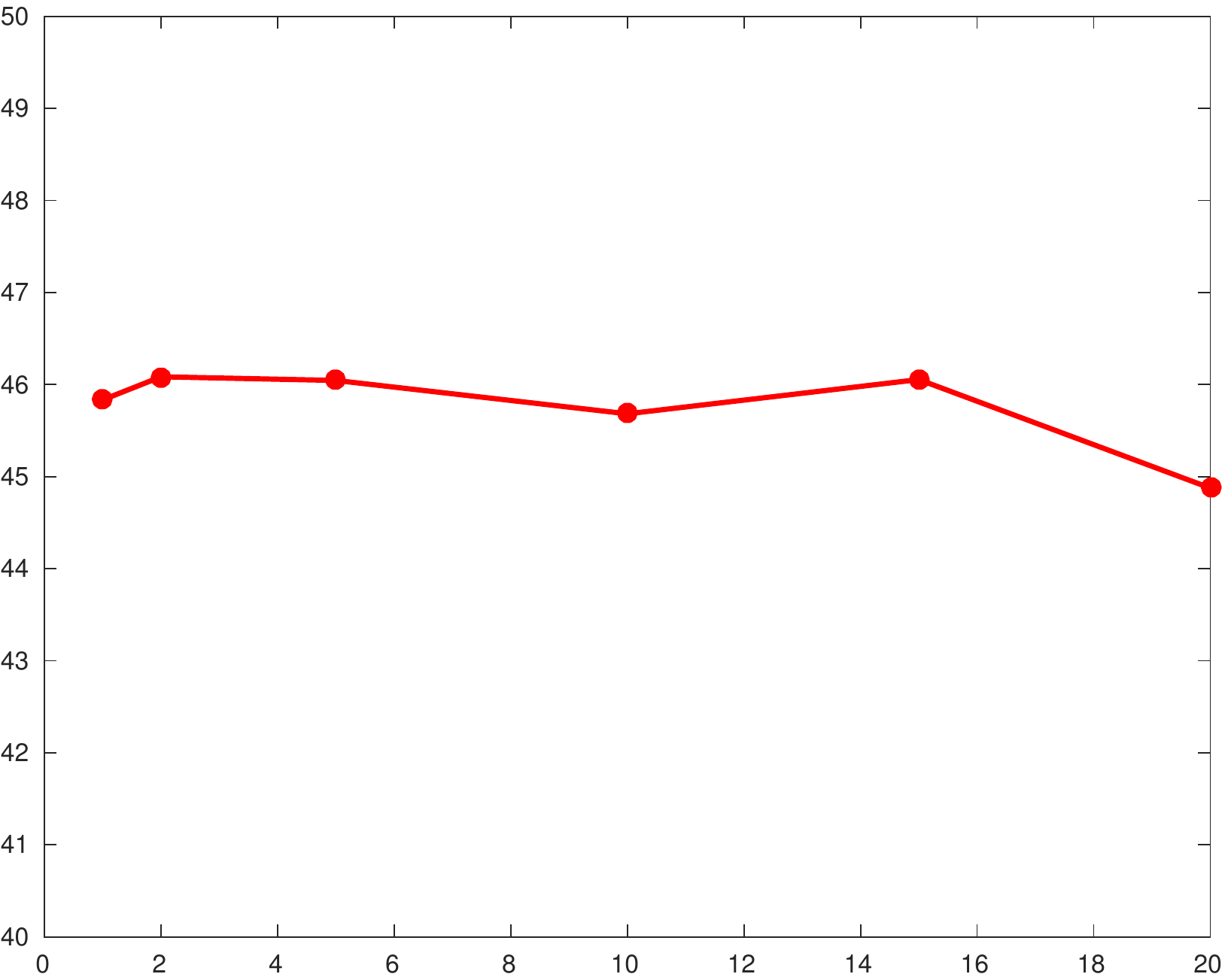}
\caption{Filter standard deviation}
\end{subfigure}
\vspace{-0.2cm}
\caption{\em Speedup of our modified filtering
algorithm over the divide-and-conquer strategy of~\cite{desmaison2016efficient}
on a Pascal image, \textbf{top:} spatial kernel (d = 2), \textbf{bottom: }
bilateral kernel (d = 5). Note that our speedup grows with the number of pixels
and is approximately constant with respect to the number of labels and
filter standard deviation.}
\label{fig:pasph}
\end{center}
\vspace{-0.5cm}
\end{figure*}

\begin{figure*}
\def\SUBWIDTH{0.3\linewidth}
\begin{center}
\begin{subfigure}{\SUBWIDTH}
\includegraphics[width=0.95\linewidth]{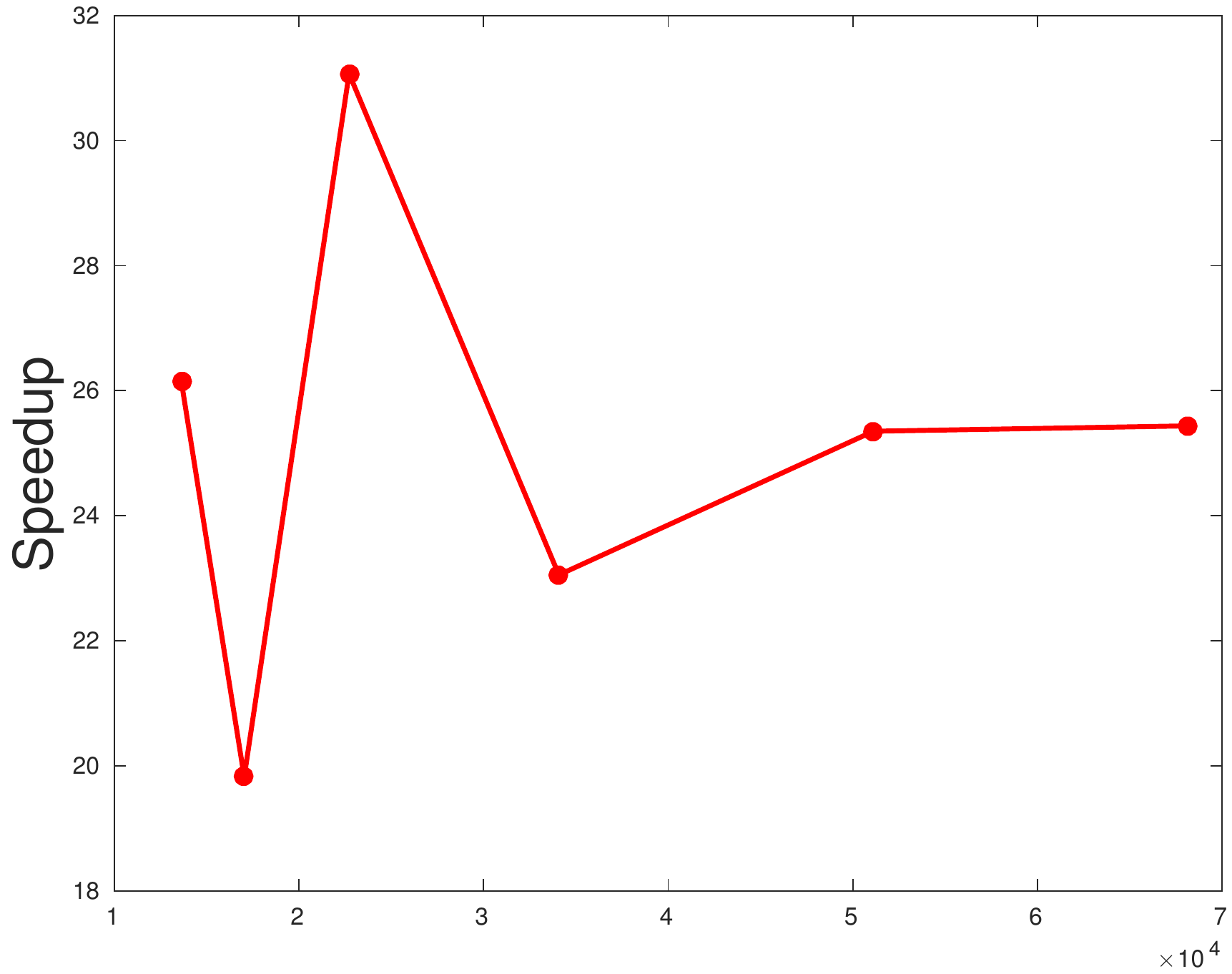}
\end{subfigure}%
\begin{subfigure}{\SUBWIDTH}
\includegraphics[width=0.95\linewidth]{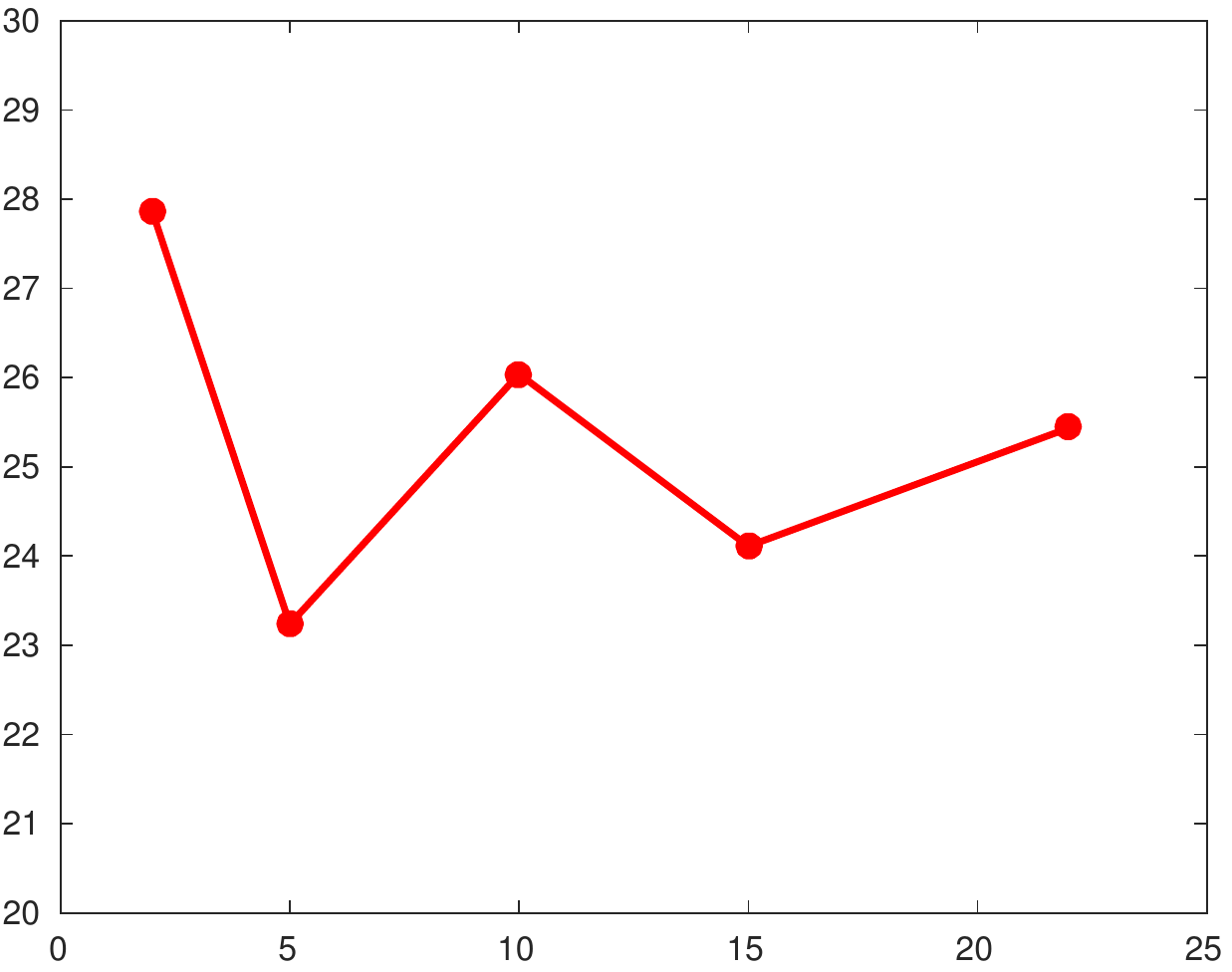}
\end{subfigure}%
\begin{subfigure}{\SUBWIDTH}
\includegraphics[width=0.95\linewidth]{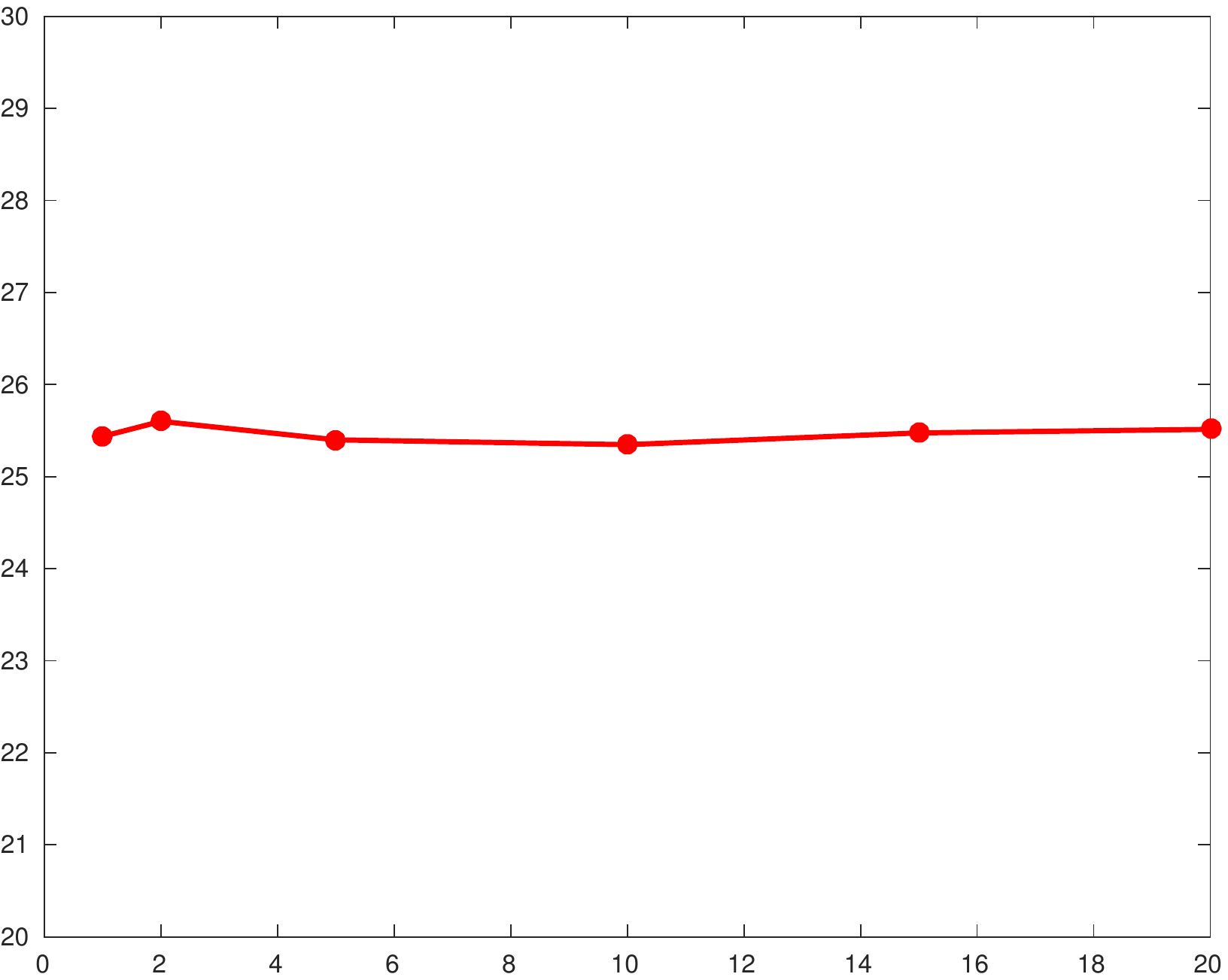}
\end{subfigure}%

\begin{subfigure}{\SUBWIDTH}
\includegraphics[width=0.95\linewidth]{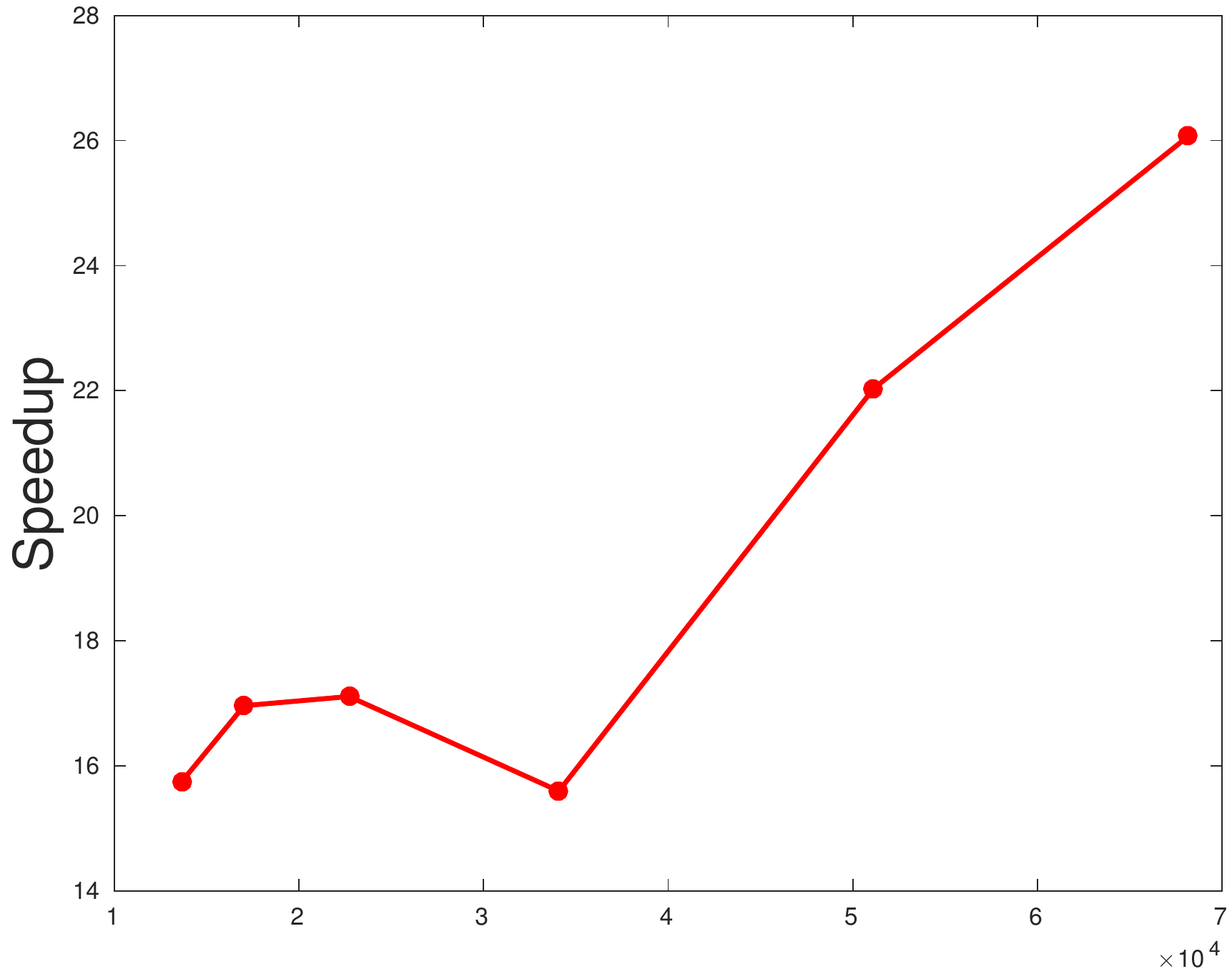}
\caption{Number of pixels}
\end{subfigure}%
\begin{subfigure}{\SUBWIDTH}
\includegraphics[width=0.95\linewidth]{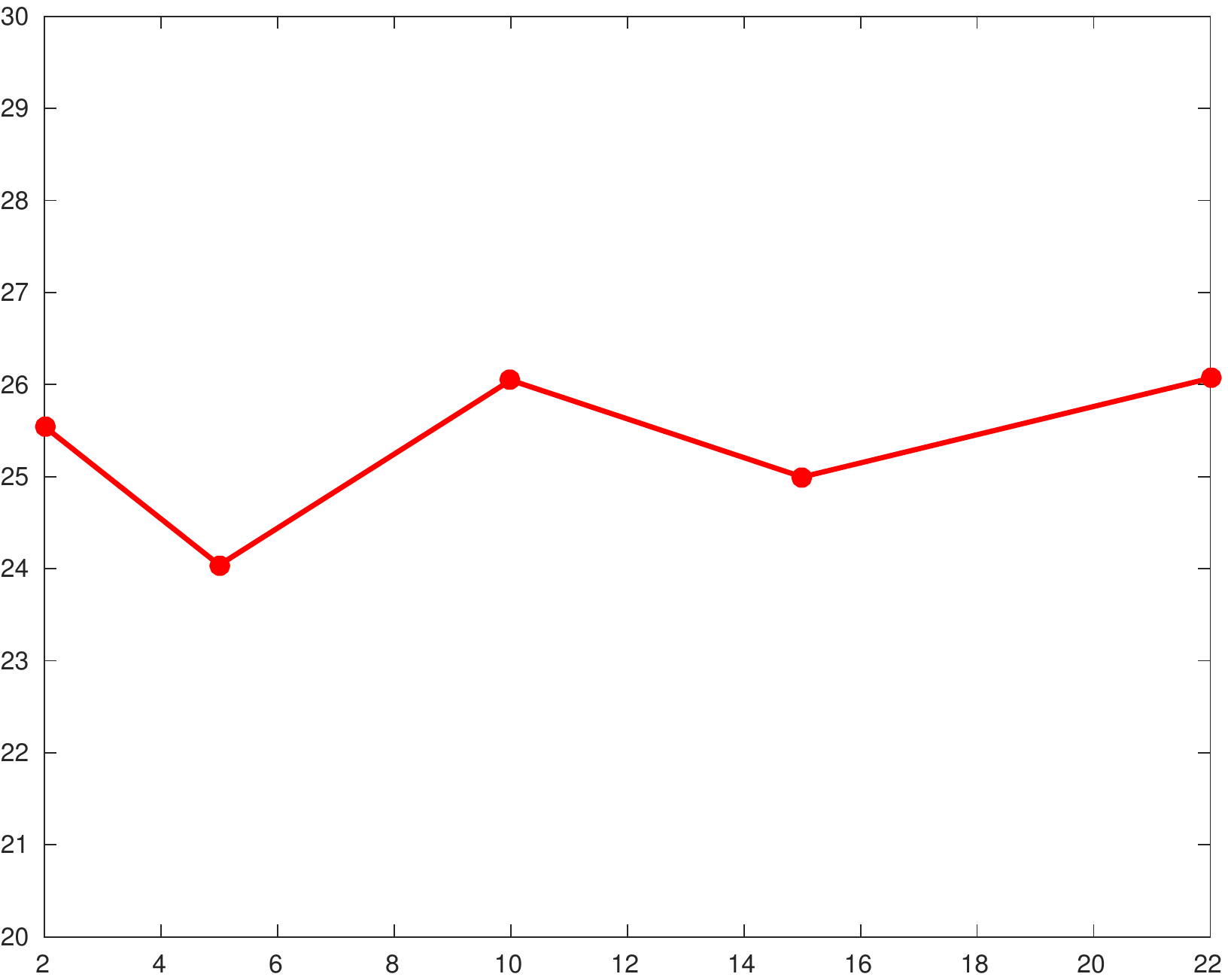}
\caption{Number of labels}
\end{subfigure}%
\begin{subfigure}{\SUBWIDTH}
\includegraphics[width=0.95\linewidth]{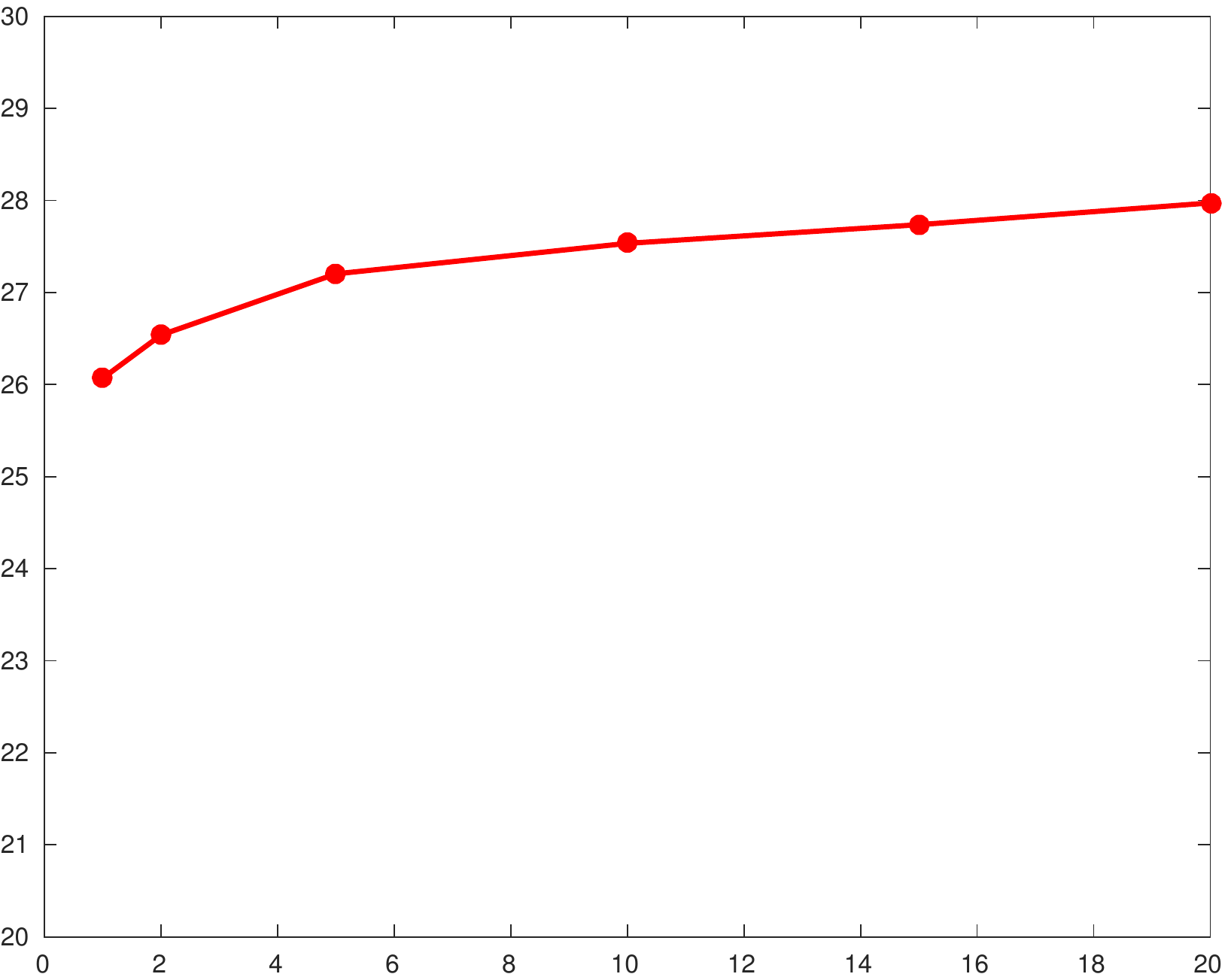}
\caption{Filter standard deviation}
\end{subfigure}
\vspace{-0.2cm}
\caption{\em Speedup of our modified filtering
algorithm over the divide-and-conquer strategy of~\cite{desmaison2016efficient}
on a MSRC image, \textbf{top:} spatial kernel (d = 2), \textbf{bottom: }
bilateral kernel (d = 5). Note that our speedup grows with the number of pixels
and is approximately constant with respect to the number of labels and
filter standard deviation.}
\label{fig:msrcph}
\end{center}
\vspace{-0.2cm}
\end{figure*}